\documentclass{article} 
\usepackage{iclr2026_conference,times}

\usepackage{hyperref}
\usepackage{url}

\title{Dissecting Discrete Soft Actor-Critic:\\ Limitations and Principled Alternatives}

\author{
Reza Asad \thanks{Correspondence to: \texttt{rasad@sfu.ca}} \\
Simon Fraser University \\
\And
Reza Babanezhad \\
Samsung AI \\
\And
Sharan Vaswani \\
Simon Fraser University \\
}

\usepackage[utf8]{inputenc} 
\usepackage[T1]{fontenc}    
\usepackage{url}            
\usepackage{amsfonts}       
\usepackage{xcolor}         

\usepackage{lmodern}
\usepackage[english]{babel}
\usepackage{latexsym}
\usepackage{amsmath}
\usepackage{mathrsfs}
\usepackage{amssymb}
\usepackage{mathtools}
\usepackage{cases}
\usepackage[inline,shortlabels]{enumitem} 
\usepackage{bm}
\usepackage{datetime}
\usepackage{accents}
\usepackage{tikz}
\usepackage{listings}
\usepackage{mdframed}
\usepackage{pgfplots}
\usepackage{pgfplotstable}
\usepackage[linesnumbered,lined,boxed,commentsnumbered,ruled,longend]{algorithm2e}
\usepackage{algorithmic}

\usepackage{dsfont}
\usepackage{color}
\usepackage{colortbl}
\usepackage{pifont}
\usepackage[bf,font=small,singlelinecheck=off]{caption}
\usepackage{microtype} 
\usepackage{float}
\usepackage{xfrac} 
\usepackage{xspace}
\usepackage{booktabs}
\usepackage{blkarray}

\allowdisplaybreaks
\usepackage{hyperref}
\hypersetup{
    unicode=false,          
    pdftoolbar=true,        
    pdfmenubar=true,        
    pdffitwindow=false,     
    pdfstartview={FitH},    
    pdftitle={XXXXX},    
    pdfauthor={XXX},     
    pdfsubject={Bandits, Reinforcement Learning},   
    pdfcreator={Creator},   
    pdfproducer={Producer}, 
    pdfkeywords={bandits} {reinforcement learning} {policy gradient}, 
    pdfnewwindow=true,      
    colorlinks=true,       
    linkcolor=blue,          
    citecolor=blue,        
    filecolor=magenta,      
    urlcolor=cyan           
}
\usepackage{amsthm}
\usepackage{times}
\usepackage{natbib}
\usepackage{nicefrac}
\usepackage{wrapfig}
\usepackage{pgfplots}
\usepackage[capitalize]{cleveref}
\usepackage[nottoc,numbib]{tocbibind}

\usepackage[normalem]{ulem}

\usepackage{thmtools, thm-restate}
\declaretheorem{theorem}
\declaretheorem{proposition}
\declaretheorem{corollary}

\newtheorem{lemma}{Lemma}
\usepackage{comment}

\usepackage{pifont}  

\newcommand{\xmark}{\textcolor{red}{\ding{55}}} 

\newcommand{\SPMAFKL}{\texttt{SPMA-FKL}}
\newcommand{\SPMARKL}{\texttt{SPMA-RKL}}
\newcommand{\NPGFKL}{\texttt{NPG-FKL}}
\newcommand{\NPGRKL}{\texttt{NPG-RKL}}
\newcommand{\SAC}{\texttt{SAC}}
\newcommand{\DSAC}{\texttt{DSAC}}

\newcommand{\DQN}{\texttt{DQN}}
\newcommand{\SOFTDQN}{\texttt{Soft-DQN}}
\newcommand{\MDQN}{\texttt{M-DQN}}
\newcommand{\IQN}{\texttt{IQN}}
\newcommand{\RAINBOW}{\texttt{Rainbow}}

\newcommand{\LMC}{\texttt{Adam LMCDQN}}
\newcommand{\NoisyNetDQN}{\texttt{NoisyNet DQN}}
\newcommand{\QRDQN}{\texttt{QRDQN}}
\newcommand{\CFiveOne}{\texttt{C51}}
\newcommand{\GreedyAC}{\texttt{GreedyAC}}

\newcommand{\NPG}{\texttt{NPG}}
\newcommand{\SPMA}{\texttt{SPMA}}
\newcommand{\MDPO}{\texttt{MDPO}}
\newcommand{\PPO}{\texttt{PPO}}

\def\pitt{{\pi_{t+1}}}
\def\pith{{\pi_{t+1/2}}}

\newcommand{\pit}{{\pi_{t}}}

\newcommand{\taut}{{\tau_{t}}}
\newcommand{\alphat}{{\alpha_{t}}}

\newcommand{\inner}[2]{\langle #1,\, #2 \rangle}
\newcommand{\normsq}[1]{\left\|#1 \right\|_{2}^{2}}
\newcommand{\norm}[1]{\left\|#1 \right\|}

\newcommand{\thtt}{\theta_{t+1}}

\def\floor#1{\lfloor #1 \rfloor}
\def\1{\bm{1}}

\newcommand{\etat}{{\eta_t}}

\DeclareMathAlphabet{\mathsfit}{\encodingdefault}{\sfdefault}{m}{sl}
\SetMathAlphabet{\mathsfit}{bold}{\encodingdefault}{\sfdefault}{bx}{n}

\def\cA{{\mathcal{A}}}

\def\cH{{\mathbb{H}}}

\def\cL{{\mathcal{L}}}

\def\cP{{\mathcal{P}}}

\def\cR{{\mathcal{R}}}

\def\cS{{\mathcal{S}}}

\def\cZ{{\mathcal{Z}}}

\newcommand{\E}{\mathbb{E}}

\newcommand{\R}{\mathbb{R}}

\DeclareMathOperator*{\argmax}{arg\,max}
\DeclareMathOperator*{\argmin}{arg\,min}

\newcommand{\bepsilon}{\boldsymbol{\epsilon}}
\def\FMAPG/{{FMA-PG}}
\def\DFMAPG/{{DFMA-PG}}
\def\SFMAPG/{{SFMA-PG}}
\def\cT{{\mathcal{T}}}

\newcommand{\norminf}[1]{\left\|#1\right\|_{\infty}}
\newcommand{\normone}[1]{\left\|#1 \right\|_1}

\def\cA{{\mathcal{A}}}

\def\cM{{\mathcal{M}}}

\def\cT{{\mathcal{T}}}
\def\cH{{\mathcal{H}}}

\def\cD{{\mathcal{D}}}

\def\ppi{{\pi}}

\def\ppith{{\pi_{t + \nicefrac{1}{2}}}}
\def\ppitt{\pi_{t+1}}

\def\qpi{{q^{\pi}}}

\pgfplotsset{compat=1.18}

\iclrfinalcopy
\begin{document}

\maketitle

\begin{abstract}
While Soft Actor-Critic ($\SAC$) is highly effective in continuous control, its discrete counterpart ($\DSAC$) performs poorly on challenging discrete-action domains such as Atari. Consequently, starting from $\DSAC$, we revisit the design of actor-critic methods in this setting. First, we determine that the coupling between the actor and critic entropy is the primary reason behind the poor performance of $\DSAC$. We demonstrate that by merely decoupling these components, $\DSAC$'s performance significantly improves. Motivated by this insight, we introduce a flexible off-policy actor-critic framework that subsumes $\DSAC$ as a special case and yields novel objectives. Our framework allows using an $m$-step Bellman operator for the critic update, and instantiates the actor objective by combining standard policy optimization methods with entropy regularization. Theoretically, we prove that the proposed methods can guarantee convergence to the optimal regularized value function in the tabular setting, generalizing the results in prior work. 
Empirically, we evaluate the proposed objectives on standard Atari games. Our ablations indicate that, unlike $\DSAC$, these objectives, including novel ones, perform robustly even without entropy regularization or explicit exploration mechanisms.
\end{abstract}

\section{Introduction}
\label{sec:introduction}

Soft Actor-Critic ($\SAC$)~\citep{haarnoja2018soft} is a widely used off-policy reinforcement learning algorithm for continuous control, combining entropy-regularized critic learning with policy optimization. It has achieved strong empirical performance across simulated control benchmarks and real-world robotic systems~\citep{haarnoja2018softapp}, and has inspired numerous extensions that further improve its effectiveness~\citep{zhu2024q,chen2021randomized,hiraoka2021dropout,bhatt2019crossq,kuznetsov2020controlling,ishfaq2025langevin}.
Motivated by this success, prior work has sought to adapt $\SAC$ to off-policy discrete-action settings~\citep{christodoulou2019soft}, which are central to domains such as Atari. However, despite being a natural extension, the resulting discrete variant of $\SAC$ ($\DSAC$) exhibits weak empirical performance on standard benchmarks.

To address this limitation, follow-up work has proposed alternative discrete-action extensions of $\SAC$~\citep{xu2021target,zhou2022revisiting}. However, these methods typically introduce multiple interacting components that complicate implementation and can limit practical performance. For example, $\texttt{SD-SAC}$~\citep{zhou2022revisiting} relies on ad-hoc actor regularization and double-averaged $Q$-clipping to achieve good performance on Atari. Other lines of work propose alternative off-policy actor-critic designs for the discrete-action setting, including off-policy variants of PPO~\citep{queeney2021generalized,meng2023off,gan2024transductive} and the $\GreedyAC$ algorithm~\citep{neumann2018greedy}. While these latter approaches offer complementary perspectives, they have not been evaluated on large-scale benchmarks such as Atari.

Motivated by the strong performance of continuous $\SAC$ and the limited success of its discrete counterpart, $\DSAC$, we seek to address the following central question:
\begin{center}
\textit{Can we isolate the bottleneck limiting $\DSAC$’s performance and identify design choices for building principled and efficient off-policy discrete-action actor–critic methods?}
\end{center}

To achieve our objective, we first investigate the poor Atari performance of $\DSAC$~\citep{christodoulou2019soft}. Prior work attributes this weakness to the use of a fixed target entropy~\citep{xu2021target}, unstable coupling between policy and critic updates, and $Q$-value underestimation bias~\citep{zhou2022revisiting}. Consequently, existing remedies adapt the entropy target~\citep{xu2021target} or add an entropy penalty along with clipped double-average $Q$-learning~\citep{zhou2022revisiting}. While these modifications can stabilize training, they introduce extra hyper-parameters and algorithmic complexity. For example, the Atari experiments in~\citet{zhou2022revisiting} replace the automatic entropy-tuning mechanism of the original $\DSAC$ with a fixed entropy coefficient, introducing an additional tuning burden. Other related work includes~\citet{neumanninvestigating} that modifies the $\DSAC$ objective and proposes surrogates with explicit KL regularization, but does not evaluate them on the standard Atari benchmark. Importantly, none of these methods admit theoretical guarantees, even in the tabular setting.

\vspace{0.5em}
\textbf{Contribution 1}: Our extensive ablation study shows that critic entropy substantially impacts the empirical performance of $\DSAC$. In particular, disabling the entropy regularization in the critic update (i.e., using the standard Bellman operator for policy evaluation) while keeping all other components (entropy-regularized soft actor update, automatic entropy tuning) fixed yields a stable variant of $\DSAC$. This variant does not require double $Q$-learning or additional hyperparameters, and significantly outperforms the default $\DSAC$ across Atari games (see~\cref{fig:dsac-vs-dqn-20-games-m-1}). Furthermore, in~\cref{sec:experiments} (Conclusion 1), we show that this method also outperforms prior $\DSAC$ variants~\citep{xu2021target, zhou2022revisiting} on Atari, including hard-exploration games.

\begin{figure}[!ht]
\centering
\begin{minipage}{0.54\textwidth}
    \includegraphics[width=\linewidth]{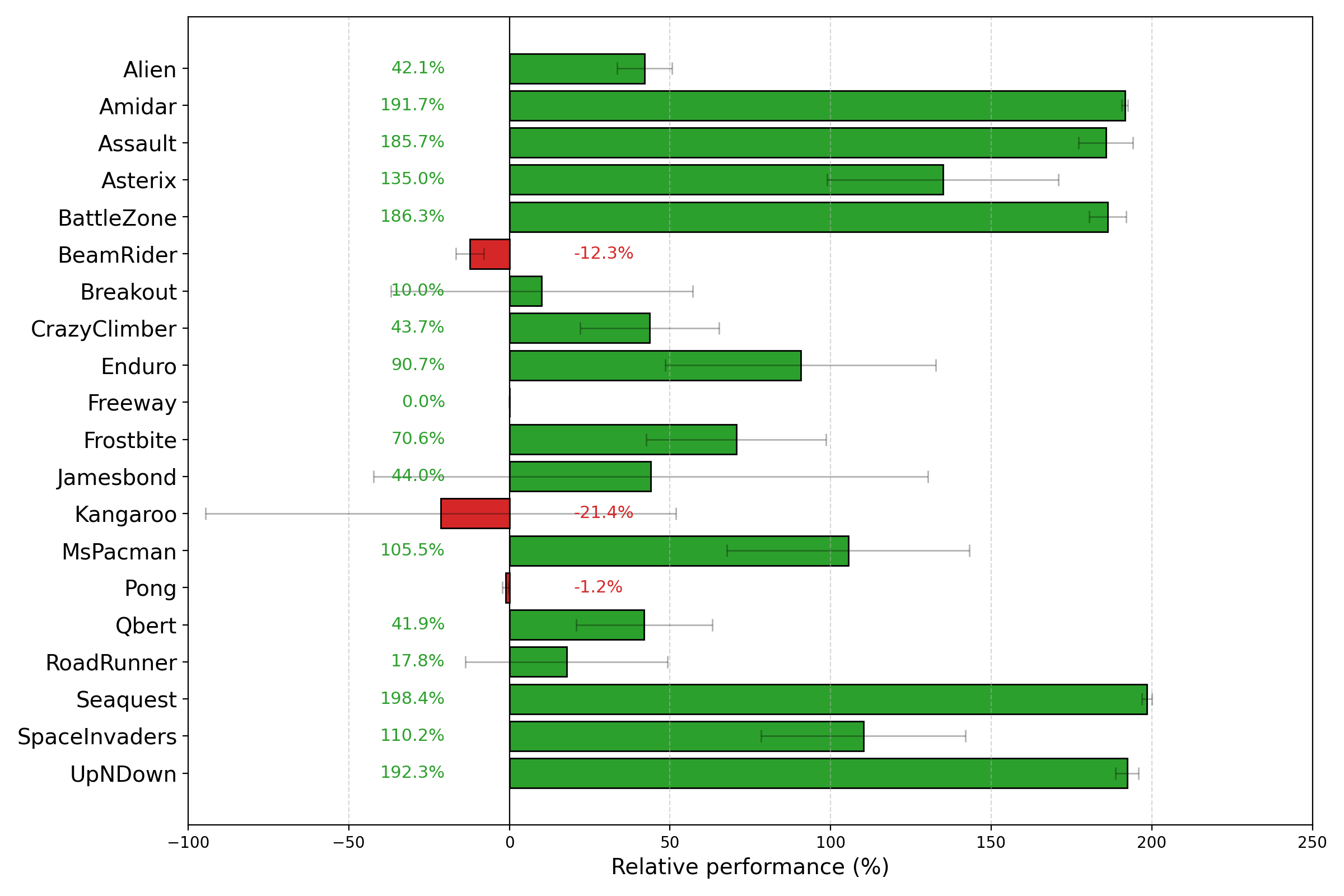}
\end{minipage}%
\hfill
\begin{minipage}{0.42\textwidth}
\caption{Relative performance of $\DSAC$ \textit{without critic entropy} (our contribution) compared to the default $\DSAC$ across 20 Atari games, averaged over 5 seeds with 95\% confidence intervals. Removing the critic entropy leads to substantial performance gains.}
\label{fig:dsac-vs-dqn-20-games-m-1}
\end{minipage}
\end{figure}

In order to systematically design and study algorithms related to $\DSAC$, it is necessary to develop a more general approach. Prior works~\citep{vieillard2020leverage,xiao2022convergence} have proposed such actor-critic frameworks in the discrete-action setting. While \citet{vieillard2020leverage} show that the discrete extension of $\SAC$ falls within their framework, our proposed $\DSAC$ variant can not be captured by it. In particular, in~\citet{vieillard2020leverage}, the actor and critic entropy are closely coupled, and the framework cannot support using entropy regularization asymmetrically (e.g. entropy regularization for actor, but not for the critic). On the other hand,~\citet{xiao2022convergence} propose a policy gradient framework and analyze the actor objective in $\DSAC$. However, they assume that $Q$ is estimated via a black-box procedure and do not instantiate the critic update. 

\vspace{0.5em}
\textbf{Contribution 2}: In~\cref{sec:algorithm}, we develop a general off-policy actor-critic framework for the discrete-action setting. Importantly, our framework decouples critic and actor entropy, enabling greater flexibility. The modular structure of our actor update gives rise to a family of objectives, including two novel ones, and subsumes $\DSAC$ as a special case. In particular, in the policy evaluation step for the critic, we use a look-ahead target formed by either the soft \textit{critic entropy}-regularized or the standard hard Bellman operator. The policy optimization step for the actor consists of two stages: (i) computing an intermediate policy and (ii) projecting it onto the class of realizable policies while simultaneously maximizing a proximal entropy regularization term (referred to as the \textit{actor entropy}). For (i), we consider two methods ---$\NPG$~\citep{kakade2001natural} due to its widespread use, and $\SPMA$~\citep{asad2024fast} because it can avoid normalization across actions and has demonstrated good empirical performance in the on-policy setting. For (ii), following~\citep{chan2022greedification}, we study both forward and reverse KL divergences for the projection step to understand the effect of this choice within the proposed framework.

\vspace{0.5em}
\textbf{Contribution 3}: Unlike prior work that extends $\SAC$ to discrete-action settings or studies its weak empirical performance~\citep{christodoulou2019soft, xu2021target, zhou2022revisiting}, our framework has a strong theoretical foundation. In~\cref{sec:theory-tabular}, we analyze the proposed methods in the tabular setting under asymmetric actor and critic entropy and establish the first convergence guarantees for such methods. Moreover, when the actor and critic entropy is coupled, we recover the convergence rates in~\citet{vieillard2020leverage}. 
More concretely, in~\cref{thm:meta-entropy}, we reduce the problem of analyzing the sub-optimality in the entropy-regularized value function to (i) bounding the policy evaluation error for the critic and (ii) bounding the regret for an online convex optimization problem related to the policy update step for the actor. 
Our modular framework can be used to analyze different combinations of the actor and critic. For example, in~\cref{corr:soft-npg-pe-no-entropy}, we leverage our framework to provide a theoretical guarantee for the $\NPG$ update with actor entropy and $m$ steps of the hard Bellman operator (i.e., no critic entropy) for policy evaluation. This corresponds to one of the variants that we later evaluate empirically in the function approximation setting.

\vspace{0.5em}
\textbf{Contribution 4}: In~\cref{sec:experiments}, we empirically evaluate the objectives instantiated using our actor-critic framework. Our results reveal three key findings. First, similar to $\DSAC$, the objectives derived from our framework benefit from the hard Bellman operator for policy evaluation, resulting in improved performance. Second, unlike $\DSAC$, the performance of these objectives, including two novel ones, remains robust even without entropy regularization or explicit exploration. 
Third, using the forward KL divergence when projecting the intermediate policy onto the class of realizable policies generally leads to higher expected reward. 

\section{Preliminaries}
\label{sec:preliminary}
We consider an infinite-horizon discounted Markov decision process (MDP), defined as $\cM = \langle \cS, \cA, \cP, r, \rho, \gamma \rangle$, where $\cS$ and $\cA$ denote the state and action spaces, $\mathcal{P}: \cS \times \cA \rightarrow \Delta_{\cS}$ is the transition probability function, $r: \cS \times \cA \rightarrow [0, 1]$ is the reward function, $\rho \in \Delta_{\cS}$ is the initial state distribution, and $\gamma \in [0, 1)$ is the discount factor. Throughout this paper, we assume that the state and action spaces are finite but potentially large.

For a fixed $s \in \cS$, the policy $\pi$ induces a distribution $\ppi(\cdot \vert s)$ over the actions. The action-value function $\qpi: \cS \times \cA \rightarrow \mathbb{R}$ of policy $\pi$ is defined as $\qpi(s, a) := \E[\sum_{t=0}^{\infty} \gamma^t r(s_t, a_t)| s_0= s , a_0 = a]$, with $s_t \sim p(\cdot \vert s_{t-1}, a_{t-1})$ and $a_t \sim \ppi(\cdot \vert s_t)$. Given an initial state $s \sim \rho$, the corresponding value function is defined as $v^{\pi}(s) := \E_{a \sim \ppi(. \mid s)}[\qpi(s, a)]$. The advantage function $\mathfrak{a}^\pi: \cS \times \cA \rightarrow \mathbb{R}$ induced by $\pi$ is represented by $\mathfrak{a}^\pi(s, a) := \qpi(s, a) - v^{\pi}(s)$. We define $J(\pi) := v^\pi(\rho) = \E_{s \sim \rho} [v^\pi(s)]$ as the expected discounted cumulative reward. 

\vspace{0.5em}

Given the Shannon entropy function $\cH(\pi) = - \sum_{a} \pi(a) \, \ln(\pi(a))$, if $\tau \geq 0$ is the entropy coefficient, then, the soft (entropy-regularized) counterparts of the above functions~\citep{liu2024elementary,vieillard2020leverage} are defined as: $v_\tau^\pi(s)  := v^\pi(s) + \tau \, \sum_{t = 0}^\infty \gamma^t \left[ \cH(\pi(\cdot|s_t)) \mid s_0 = s \right]$, $J_\tau(\pi) := \E_{s \sim\rho} [v_\tau^\pi(s)]$, $q_\tau^\pi(s,a) := \E_{s' \sim \cP(\cdot|s,a)} \left[ r(s,a) + \gamma v_\tau^\pi(s') \right]$ and $\mathfrak{a}_\tau^\pi(s,a) := q_\tau^\pi(s,a) - v_\tau^\pi(s) - \tau \, \ln(\pi(a|s))$. Note that the soft value and action-value functions, $v_\tau^\pi(s)$ and $q_\tau^\pi(s, a)$, are bounded and lie within the interval $\left[0, H_\tau \right]$, where $H_\tau := \frac{1 + \tau \ln(A)}{1 - \gamma}$~\citep{liu2024elementary}. 

\vspace{0.5ex}

Furthermore, for a fixed $s \in \cS$ and an arbitrary pair of policies $\pi_1$ and $\pi_2$, the \textit{soft} Bellman operator $T_\tau^\pi$ is defined such that: $(T_\tau^{\pi_{1}} v_\tau^{\pi_2})(s) = \E_{a \sim \pi_1(\cdot|s)} \left[q_\tau^{\pi_2}(s,a) - \tau \, \ln(\pi_1(a|s)) \right] = (T^{\pi_{1}} v_\tau^{\pi_2})(s) + \tau \cH(\pi_1(\cdot|s))$ and $(T_\tau^{\pi_{1}} q_\tau^{\pi_2})(s,a) = (T^{\pi_{1}} q_\tau^{\pi_2})(s,a) + \tau \E_{s' \sim \cP(\cdot|s,a)}\, \cH(\pi_1(\cdot|s'))$. If $\tau = 0$, we refer to the corresponding operator as the \textit{hard} Bellman operator. The objective is to $\max_{\pi \in \Pi} J_\tau(\pi)$, where $\Pi$ is the set of feasible policies. We denote the optimal entropy-regularized policy by $\pi_\tau^* := \argmax_{\pi} J_\tau(\pi)$ and its corresponding value function as $v_\tau^*$. 

\section{General Off-policy Actor-critic Framework}
\label{sec:algorithm}
Building on the empirical findings from~\cref{sec:introduction}, we present a general off-policy actor-critic framework that subsumes $\DSAC$ and yields new actor objectives, including two novel ones. In~\cref{subsec:policy-evaluation}, we focus on the policy evaluation step and instantiate the resulting critic objective. In~\cref{subsec:policy-optimization}, we focus on two alternative policy updates and instantiate the corresponding actor objectives. 


\vspace{-1ex}
\subsection{Policy Evaluation}
\label{subsec:policy-evaluation}
At iteration $t \in [K]$ of the actor-critic algorithm, we evaluate the current policy $\pi_t$ using the $m$-step entropy-regularized Bellman operator. The coefficient, $\zeta \geq 0$, controls the strength of entropy regularization (referred to as the \textit{critic entropy}). The corresponding estimate of the entropy-regularized $q$ function at iteration $t$ is denoted by $q_\zeta^t$ and computed recursively as: 
\begin{align}
q_\zeta^0 &= q_\zeta^{\pi_0} \quad \text{;} \quad \forall t \geq 1, \quad q_\zeta^t = \mathbb{P}_{[0, H_\tau]}[(T^{\pi_t}_\zeta)^{m} q_\zeta^{t-1}] \,,
\label{eq:pe-general}
\end{align} 
where $\mathbb{P}_{[0, H_\tau]}$ projects each entry onto the $[0, H_\tau]$ interval. As $m \to \infty$, the algorithm exactly evaluates the policy $\pi_t$ and $q_\zeta^t$ converges to the fixed point $q_\zeta^{\pi_t}$. In this special case, setting $\zeta = 0$ recovers the standard $q$ function. 

\vspace{0.5em}
\paragraph{Handling function approximation:}\cref{eq:pe-general} requires updating the state-action value function for each state and action. In settings where the state or action space is large, this is not computationally feasible and we aim to approximate this update. To this end, we focus on the off-policy setting and define the critic objective using the standard squared loss~\citep{haarnoja2018soft}. Specifically, we use the \textit{replay buffer} $\cD_t$ consisting of (state, action, next state, reward) pairs obtained by the policies in the previous iterations. We define $\phi$ as the parameters of a model (typically a neural network) parameterizing the critic and $q_\phi$ as the corresponding function. For defining the critic objective, the details of the model are irrelevant and are implicit in the $q_\phi$ notation. We use~\cref{eq:pe-general} to construct a \textit{one-step look-ahead target} (corresponding to $m = 1$\footnote{We focus on $m = 1$ for simplicity. The objective for $m > 1$ can be defined analogously.}) and define the critic objective $\cL_t(\phi)$ as:
\begin{gather}
\E_{\substack{(s,a,s',r(s,a)) \sim \cD_t \\ a' \sim \pi_t(\cdot|s')}} 
\normsq{ q_\phi(s,a) - \mathbb{P}_{[0, H_\tau]}
\bigl[
  r(s,a)
  + \gamma \,
  [q_\zeta^{t-1}(s',a')
  - \zeta \, \ln(\pi_t(a'|s'))]
\bigr] }
\label{eq:pe-fa}
\end{gather}
With a slight abuse of notation, by $(s,a,s',r(s,a)) \sim \cD_t$, we mean that the tuple is sampled from a discrete distribution over $\cD_t$. We set $q_\zeta^{t} = q_{\phi_t}$ where $\phi_t \approx \argmin \cL_t(\phi)$. Similar to~\citet{haarnoja2018soft}, in practice, we do not clip the look-ahead target and optimize over the $(s,a,s',r(s,a))$ tuples in a randomly-sampled batch from $\cD_t$.

\subsection{Policy Optimization}
\label{subsec:policy-optimization}
At iteration $t \in [K]$ in the actor-critic algorithm, the policy optimization step uses the $q$ function estimates from~\cref{subsec:policy-evaluation} to update the policy. For a state $s \in \cS$, we first compute an intermediate policy $\ppith$ using two representative methods --- (i) the natural policy gradient ($\NPG$) or policy mirror descent (PMD)~\citep{kakade2001natural,xiao2022convergence} update and (ii) the recently proposed $\SPMA$ update~\citep{asad2024fast}. We note that the framework is not limited to these choices. With a suitable step-size $\etat$ and appropriate normalization, 
\begin{align}
\pith(a|s) &\propto \pit(a|s) \, \exp( \eta_t \, q_\zeta^t(s,a) )
\quad \text{\tiny(\textbf{NPG})}
\label{eq:npg}  \\  
\pith(a|s) &\propto \pit(a|s) \, \left[1 + \eta_t \left(q_\zeta^t(s,a) - v_\zeta^t(s) \right) \right]
\hfill \text{\tiny(\textbf{SPMA})}
\label{eq:spma}
\end{align}
In the special case, when $\zeta = 0$, we note that the $\SPMA$ update can use a sufficiently small step-size to avoid an explicit normalization across the actions~\citep{asad2024fast}. On the other hand, the $\NPG$ update always requires an explicit normalization across actions to ensure that $\pith$ is a valid probability distribution. We now use a proximal update to incorporate entropy-regularization in the policy optimization step. Given $\pith$ and the entropy regularization parameter $\tau \geq 0$, the updated policy $\pi_{t+1}$ can be computed in two alternative ways that aim to find the ``closest'' policy $\pi$ to $\pith$ while encouraging the resulting policy to have sufficiently high entropy (referred to as the \textit{actor entropy}\footnote{The actor entropy coefficient $\tau$ is independent from the critic entropy coefficient $\zeta$. Intuitively, actor entropy regularizes the distribution over actions at the current state, whereas the critic entropy regularizes the distribution in the next-state.}). Specifically, if $\taut := \tau \, \etat \geq 0$ is the entropy regularization parameter at iteration $t$, we use either the forward KL (FKL) or reverse KL (RKL) divergence to measure the proximity between policies. For each state $s \in \cS$, the updated policy $\pitt(\cdot \mid s)$ is obtained as the solution to one of the following optimization problems:
\begin{subequations} 
\begin{align} \argmin_{\pi(\cdot|s) \in \Delta} \text{KL}(\pith(\cdot|s) \| \pi(\cdot | s)) - \taut \cH(\pi(\cdot | s)) \quad \text{\scriptsize(\textbf{FKL})} \label{eq:policy-update-fkl} \\ \argmin_{\pi(\cdot|s) \in \Delta} \text{KL}(\pi(\cdot | s) \| \pith(\cdot|s)) - \taut \cH(\pi(\cdot | s)) \quad \text{\scriptsize(\textbf{RKL})} \label{eq:policy-update-rkl} \end{align} 
\end{subequations}
Note that in the special case when $\tau = 0$, $\pitt = \pith$. Furthermore, the objective in~\cref{eq:policy-update-rkl} is convex in $\pi$ and the resulting update can be obtained in closed form where $\pitt(a|s) \propto [\pith(a|s)]^{\frac{1}{1 + \taut}}$ for all $(s,a)$. Combining~\cref{eq:policy-update-fkl,eq:policy-update-rkl} with~\cref{eq:npg,eq:spma} gives rise to four possible ways of updating the policy. We instantiate the corresponding actor objectives in the function approximation setting below. 

\vspace{0.5em}
\paragraph{Handling function approximation:}\cref{eq:policy-update-fkl,eq:policy-update-rkl} require updating the state-action value function for each state and action. In order to scale to large state-action spaces, we use function approximation in the policy space. Specifically, given the parameters $\theta$ of a model parameterizing the actor and $\pi(\theta)$ as the corresponding policy, we define $\Pi_\theta = \{\pi \vert \, \exists \, \theta \text{ s.t } \pi = \pi(\theta) \}$ as the set of realizable policies. Similar to~\cref{subsec:policy-evaluation}, the choice of the model is implicit in the $\pi(\theta)$ notation.

Following~\citet{haarnoja2018soft}, we modify~\cref{eq:policy-update-fkl,eq:policy-update-rkl} to optimize (i) only over the states in the replay buffer $\cD_t$ and (ii) over the restricted policy class $\Pi_\theta$, yielding $\pitt$ for the FKL and RKL variants:
$
\argmin_{\pi \in \Pi_\theta}
\!\!\sum_{s \sim \cD_t}\!\!
\bigg[
  \text{KL}(\pith(\cdot|s) \| \pi(\cdot | s))
  - \taut \cH(\pi(\cdot | s))
\bigg], \\ \text{and} \argmin_{\pi \in \Pi_\theta}
\!\!\sum_{s \sim \cD_t}\!\!
\bigg[
  \text{KL}(\pi(\cdot | s) \| \pith(\cdot|s))
  - \taut \cH(\pi(\cdot | s))
\bigg]
$.

Following ~\citet{vaswani2021general,lavington2023target, tomar2020mirror, xiong2024dual}, we convert the above projection problem into an unconstrained optimization over $\theta$, and form $\ell_t(\theta)$, the corresponding actor objective. We define $\pitt = \pi(\thtt)$ where $\thtt \approx \argmax_{\theta} \ell_t(\theta)$. For each variant, we append the postfix $(\tau, \zeta)$ to denote its dependence on the actor and critic entropy. The actor objective $\ell_t(\theta)$ can take one of the four forms, summarized in~\cref{tab:objectives}.

\begin{table*}[t]
    \centering
    \vspace{-1em} 
    
    \resizebox{0.98\textwidth}{!}{%
        
        \renewcommand{\arraystretch}{2.2} 
        \setlength{\tabcolsep}{4pt}       
        
        \begin{tabular}{l r}
        \toprule
        $ \displaystyle
        \E_{s \sim \cD_t}
        \Bigg[
          \E_{a \sim \pi_{\theta}(\cdot|s)}
          \Big[
            q_\zeta^t(s, a)
            - \tau \, \ln(\pi_{\theta}(a|s))
          \Big]
        - \frac{1}{\etat} \,
        \text{KL}\!\big(
          \pi_{\theta}(\cdot|s) \,\|\, \pi_t(\cdot|s)
        \big)\Bigg]
        $ 
        & 
        \scriptsize\textbf{$\NPGRKL(\tau,\zeta)$} \\

        $ \displaystyle
        \E_{s \sim \cD_t}
        \Bigg[
          \mathbb{E}_{a \sim \pi_{\theta}(\cdot|s)}
          \Big[
            \ln\Big(1 + \etat \big(q_\zeta^t(s,a) - v_\zeta^t(s)\big)\Big)
             - \taut \, \ln \pi_{\theta}(a|s)
          \Big]
          - \mathrm{KL}\!\big(
            \pi_{\theta}(\cdot|s) || \pi_t(\cdot|s)
          \big)
        \Bigg]
        $
        & 
        \scriptsize\textbf{$\SPMARKL(\tau, \zeta)$} \\

        $ \displaystyle
        \E_{s \sim \cD_t}
        \Bigg[
          \E_{a \sim \pi_t(\cdot | s)}
          \bigg[
            \frac{\exp(\etat \, q_\zeta^t(s, a))}{\sum_{a'} \pit(a') \, \exp(\etat \, q_\zeta^t(s,a'))}
            \,
            \ln\!\left(\frac{\pi_{\theta}(a|s)}{\pit(a|s)}\right)
          \bigg]
          + \taut \cH(\pi_{\theta}(\cdot|s))
        \Bigg]
        $
        & 
        \scriptsize\textbf{$\NPGFKL(\tau, \zeta)$} \\

        $ \displaystyle
        \E_{s \sim \cD_t}
        \Bigg[
          \E_{a \sim \pi_t(\cdot | s)}
          \bigg[
            \frac{1+\etat (q_\zeta^t(s,a)-v_\zeta^t(s))}
                 {\sum_{a'}\!\pit(a')(1+\etat (q_\zeta^t(s,a')-v_\zeta^t(s)))}
            \,
            \ln\!\left(\frac{\pi_{\theta}(a|s)}{\pit(a|s)}\right)
          \bigg]
          + \taut \cH(\pi_{\theta}(\cdot|s))
        \Bigg]
        $
        & 
        \scriptsize\textbf{$\SPMAFKL(\tau, \zeta)$} \\
        \bottomrule
        \end{tabular}%
    }
    \caption{Actor objectives in the proposed framework. $(\tau,\zeta)$ denotes actor/critic entropy.} 
    \label{tab:objectives}
\end{table*}

\vspace{0.5em}
\textbf{Comparing the objectives}: Note that the $\NPGRKL$ and $\SPMARKL$ objectives differ in the first term, which is linear in the $q$ function for $\NPGRKL$, while it is logarithmic in the advantage for $\SPMARKL$. Similarly, the FKL variants also differ in the first term. Crucially, in the special case of zero critic entropy and $\eta \le 1-\gamma$, \(\SPMAFKL(\tau,0)\) does not require an explicit normalization over the actions~\citep{asad2024fast}, making it easier to implement in practice. Finally, we note that for the RKL variants, the expectation is over the actions sampled from $\pi_\theta$. The objective can be optimized by calculating the full expectation, using importance sampling, or applying the reparameterization trick. On the other hand, the FKL variants involve an expectation over the actions sampled from $\pi_t$, simplifying the implementation. In~\cref{alg:actor-critic} in~\cref{app:pseudocode}, we present the complete actor-critic pseudo-code in the function approximation setting.   

\vspace{0.5em}
\textbf{Comparison to existing methods:} We note that $\NPGRKL(\tau, \tau)$ recovers the off-policy variant of $\MDPO$ studied in~\citet{tomar2020mirror}. In the limit that $\etat \to \infty$, $\NPGRKL(\tau, \tau)$ recovers the original $\SAC$ objective in~\citet{haarnoja2018soft}. This is intuitive since $\SAC$ can be viewed as soft policy iteration. Importantly, $\NPGRKL(\tau, 0)$ and $\etat \to \infty$ recovers the $\DSAC$ variant that demonstrated good empirical performance in~\cref{fig:dsac-vs-dqn-20-games-m-1}. In the special case of the tabular setting and $\zeta = \tau$ (i.e. the actor and critic entropy is coupled and fixed), the $\NPGRKL$ variant is the same as that proposed in~\citet{vieillard2020leverage}. On the other hand, the $\SPMARKL(\tau, \zeta)$ objective is novel, and has not been studied in previously. Moreover, although $\NPGFKL(\tau,\tau)$ coincides with the objective of~\citet{mei2019principled} in the tabular setting, it gives rise to a distinct and novel objective under function approximation. 

Finally we note that in the on-policy setting where states are sampled from $d^{\pi_t}$, the distribution induced by policy $\pi_t$ (instead of $\cD_t$), setting $\tau = \zeta = 0$ and $m = \infty$ for the critic (corresponding to exact policy evaluation) we can recover the framework in~\citet{vaswani2021general}, and its instantiations~\citep{tomar2020mirror,asad2024fast}. Next, we consider the tabular setting and prove theoretical guarantees for the RKL variants. 

\section{Theoretical Guarantees}
\label{sec:theory-tabular}
We consider the tabular setting and analyze the actor-critic algorithm when using decoupled non-zero critic and actor entropy. For the critic, we consider estimating the $q$ functions using~\cref{eq:pe-general}. For the actor, we consider using the RKL variant in~\cref{eq:policy-update-rkl} in conjunction with the $\NPG$ and $\SPMA$ updates in~\cref{eq:npg,eq:spma}, and denote the corresponding variants as \textit{soft $\NPG$} and \textit{soft $\SPMA$} respectively. 

In~\cref{thm:meta-entropy} below, we first reduce bounding the sub-optimality in the entropy-regularized value function to a per-state online convex optimization problem. It is important to note that this reduction is independent of the specific actor and critic updates.

\begin{restatable}[Generic Reduction with Actor Entropy]{theorem}{reduction}
If $\pi^*_\tau$ is the optimal entropy-regularized policy whose value function is $v^*_\tau$, for a $q_\zeta^{t}$ obtained via the policy evaluation scheme at iteration $t$, define $\epsilon_t := q_{\zeta}^t - q_\tau^{\pi_t}$. For any sequence of policies $\{\pi_0, \pi_1, \ldots, \pi_{K-1}\}$, if $\bar{\pi}_K$ is the corresponding uniform mixture policy, then,  
\begin{align*}
\norminf{v_\tau^{*} - v_\tau^{\bar \pi_K}} &\leq  \frac{\norminf{\text{Regret}(K)}}{K \, (1 - \gamma)} +  \frac{2 \, \sum_{t \in [K]} \norminf{\bepsilon_t}}{K \, (1 - \gamma)}   \,,\\
\text{where,} \, (\text{Regret}(K))(s)
:= \sum_{t = 0}^{K-1} \Big[
  &\langle \pi^*_\tau(\cdot|s) - \pi_t(\cdot|s),  q_{\zeta}^t (s,\cdot) \rangle + \tau \big( \cH(\pi^*_\tau(\cdot|s)) -  \cH(\pi_t(\cdot|s)) \big)
\Big],
\end{align*}
is the regret incurred on an online optimization problem for each state $s \in \cS$.
     
\label{thm:meta-entropy}
\end{restatable}
The above result shows that the sub-optimality of a mixture policy (one that randomly chooses a policy from $\{\pi_0, \pi_1, \ldots, \pi_{K-1} \}$) obtained by using a generic policy optimization method can be bounded in terms of the regret incurred by the method, and the sum of the \textit{policy evaluation errors} incurred in estimating $q_\tau^{\pi_t}$. Since $\epsilon_t$ depends on $\pi_t$, it depends on the specific policy optimization method, and we now bound it for both soft $\NPG$ and soft $\SPMA$. 
\begin{corollary}[Policy Evaluation Error]
Using the policy evaluation update in~\cref{eq:pe-general} and the soft $\NPG$ or soft $\SPMA$ policy update with $\etat = \frac{1}{c + \tau \, (t+1)}$ for a constant $c \geq 0$, if  $\delta(\tau, \zeta) := \frac{|\tau - \zeta| \, \ln(A)}{1 - \gamma}$, the error $\epsilon_t$ can be bounded for all $t \in [K]$ as: 
\begin{align*}
\epsilon_t := \norminf{\bepsilon_{t}}
= O\Bigg(
\frac{\gamma^m}{(1 - \gamma)^4}
\Big(\frac{1}{t} + \frac{1}{K}
\Big)
+ \frac{\delta(\tau, \zeta)}{1-\gamma}
\Bigg).
\end{align*}
\label{cor:npg-spma-pe-no-entropy}
\end{corollary}
Note that as $m$ (the number of steps of the Bellman operator) increases, the policy is evaluated more accurately, and $\epsilon_t$ decreases. As $m \to \infty$, $\epsilon_t \to O(\delta(\tau, \zeta))$. Moreover, as $t$ increases, $\epsilon_t$ decreases and as $t \to K \to \infty$, $\epsilon_t \to O(\delta(\tau, \zeta))$, which quantifies the mismatch between the actor and critic entropy and is equal to zero when $\zeta = \tau$. 

Next, we show that both soft $\NPG$ and soft $\SPMA$ can control the regret for the online optimization problem defined in~\cref{thm:meta-entropy}. 
\begin{corollary}[Regret Bounds with Actor Entropy]
Suppose $\pi_0(\cdot|s)$ is the uniform distribution over actions for  state $s$. For a sequence $\{q_{\zeta}^t\}_{t=0}^{K-1}$ with $\norminf{q_{\zeta}^t} \leq H_\tau$, the regret for soft $\NPG$ and soft $\SPMA$ as defined in~\cref{thm:meta-entropy}, with $\etat = \frac{1}{c + \tau \,(t+1)}$ for constant $c \geq 0$ can be bounded as:
\begin{align*}
\max_{s} \Bigg| \text{Regret}(K) \Bigg| 
=O\!\left(\frac{\ln(K)}{(1-\gamma)^2}\right).
\end{align*}
\label{cor:npg-spma-regret-guarantee}    
\end{corollary}
The actor entropy in~\cref{eq:policy-update-rkl} makes the online functions strongly-convex and hence both methods incur only a logarithmic regret~\citep{orabona2019modern}. Combining the results in~\cref{cor:npg-spma-pe-no-entropy,cor:npg-spma-regret-guarantee} with~\cref{thm:meta-entropy}, we obtain our main theorem for both soft $\NPG$ and soft $\SPMA$. 

\begin{theorem}[Sub-optimality of Soft $\NPG$/Soft $\SPMA$]
Let $\pi^*_\tau$ be the optimal entropy-regularized policy with value function $v^*_\tau$. Consider the soft $\NPG$ or soft $\SPMA$ updates with step size $\etat = \frac{1}{c + \tau \, (t+1)}$ for constant $c$ defined in Theorems~\ref{thm:soft-npg-pe-entropy} and~\ref{thm:soft-spma-pe-entropy} and $\pi_0(\cdot | s)$ as the uniform policy over actions for all $s \in \cS$. Using policy evaluation step in~\cref{eq:pe-general}, with $\delta(\tau, \zeta) := \frac{|\tau - \zeta| \, \ln(A)}{1-\gamma}$, the resulting uniform mixture policy $\bar{\pi}_K$ satisfies the following sub-optimality bound:   
\begin{align*}
\norminf{ v_\tau^{\bar{\pi}_{K}} - v_\tau^*}
\!= \! \tilde{O}\Bigg(\!
\frac{\ln(K)}{K\!(1\!-\!\gamma)^3} + \frac{\gamma^m}{K\!(1\!-\!\gamma)^5} + \frac{\delta(\tau,\!\zeta)}{(1\!-\!\gamma)^2}
\!\Bigg).
\end{align*}
\label{thm:soft-npg-pe-entropy-main}
\end{theorem}

Hence, the sub-optimality can be bounded as $\tilde{O}\left(\frac{1}{K} + \frac{\gamma^m}{K} + \delta(\tau, \zeta) \right)$ up to logarithmic factors. Note that when the actor and critic entropy is coupled ($\zeta = \tau$), both soft $\NPG$ and soft $\SPMA$ have an $O(1/K)$ convergence for the resulting mixture policy. For the practical variant that uses $m = 1$ and $\zeta = 0$, the above result can achieve an $O(1/K)$ convergence to an $O(\tau)$ neighbourhood of $v^*_\tau$. All proofs are deferred to~\cref{app:theoretical-results}, where we also explore the case without actor entropy i.e. $\tau = 0$ and prove an $O(1/\sqrt{K})$ rate.  

\textbf{Comparison to existing results:}~\citet{vieillard2020leverage} consider the special case of soft $\NPG$ with coupled actor and critic entropy ($\zeta = \tau$) and $m = 1$, and also establish a convergence rate of $O(1/K)$. On the other hand,~\citet[Theorem 4.4]{xiao2022convergence} analyze soft $\NPG$ with coupled actor and critic entropy ($\zeta = \tau$) and $m = \infty$, and prove an $O(1/K)$ convergence rate. In comparison, our framework can match these convergence rates while being more flexible --- we can support $m \in [1, \infty)$ (left as an open question in~\citet{vieillard2020leverage}), allow for decoupling of actor and critic entropy and support policy optimization methods beyond $\NPG$ (e.g., $\SPMA$) that have sublinear regret and Lipschitz policy updates (see the proof of~\cref{thm:error-analysis-general} in~\cref{app:theoretical-results}). In the next section, we empirically study the derived objectives (soft $\NPG$ and soft $\SPMA$ with decoupled actor and critic entropy, $m = 1$) in the function approximation setting to understand the effect of the proposed design choices.

\section{Empirical Evaluation}
\label{sec:experiments}
We evaluate $\DSAC$ and our proposed framework on various Atari 2600 games~\citep{bellemare2013arcade}. We first present an ablation study of $\DSAC$, and subsequently investigate the impact of actor and critic entropy and the direction of the KL divergence on the empirical performance of $\DSAC$ and the derived objectives. To summarize performance across the 20 Atari games used in the original $\DSAC$ and subsequent analyses~\citep{zhou2022revisiting}, we follow~\citet{agarwal2021deep} and report the IQM of final human-normalized scores over all game-seed pairs, with 95\% stratified bootstrap confidence intervals. 
We further evaluate our hypothesis about $\DSAC$ on 20 diverse \texttt{stable-retro} games~\citep{stable-retro}, which have substantially larger action spaces than Atari in~\cref{app:dsac-ablation-stable-retro}. For implementation, we follow~\citet{tomar2020mirror, asad2024fast} and use \texttt{stable-baselines3}~\citep{stable-baselines3} with five random seeds, reporting per-game average expected returns with 95\% confidence intervals across seeds. Experimental details and hyper-parameters for all algorithms are provided in~\cref{app:stable-baselines-details}.


\paragraph{Conclusion 1: Using $\zeta = 0$ Improves DSAC Performance:} To isolate the bottleneck contributing to $\DSAC$’s poor empirical performance, we fix $\DQN$ (a standard baseline for discrete-action environments) as a reference and conduct a targeted ablation of $\DSAC$ using a single $Q$-network. Specifically, we examine the role of critic entropy regularization by comparing $\zeta = 0$ and $\zeta = \tau$, while retaining actor entropy. To set the actor entropy coefficient $\tau$, we use either a fixed grid-searched value or use the adaptive scheme in~\citet{christodoulou2019soft}. Our results in~\cref{fig:dsac-ablation-m-1} indicate that both $\zeta = 0$ and $\zeta = \tau$ can result in good performance with a \textit{fixed, well-tuned per environment entropy coefficient} $\tau$ (red and purple vs. blue). Under this configuration, we also observe that on the discrete classic control tasks Acrobot, MountainCar, and Pendulum reported in~\citet{neumann2018greedy}, $\DSAC$ performs on par with $\GreedyAC$, which employs an alternative policy update (see Figures~\ref{fig:dsac-greedyac-32} and~\ref{fig:dsac-greedyac-256} in~\cref{app:greedy-ac-comparison}). These observations are consistent with 
entropy-regularized value-based methods~\citep{vieillard2020munchausen}. In particular, $\SOFTDQN$ and $\MDQN$\footnote{In contrast to $\DSAC$, both $\SOFTDQN$ and $\MDQN$ also require tuning an $\epsilon$-greedy parameter and an additional parameter related to entropy regularization~\citep{vieillard2020munchausen}.} use a fixed, tuned value of $\tau$.

\begin{figure*}[!ht]
\centering
\begin{minipage}{0.55\textwidth}
    \includegraphics[width=\linewidth]{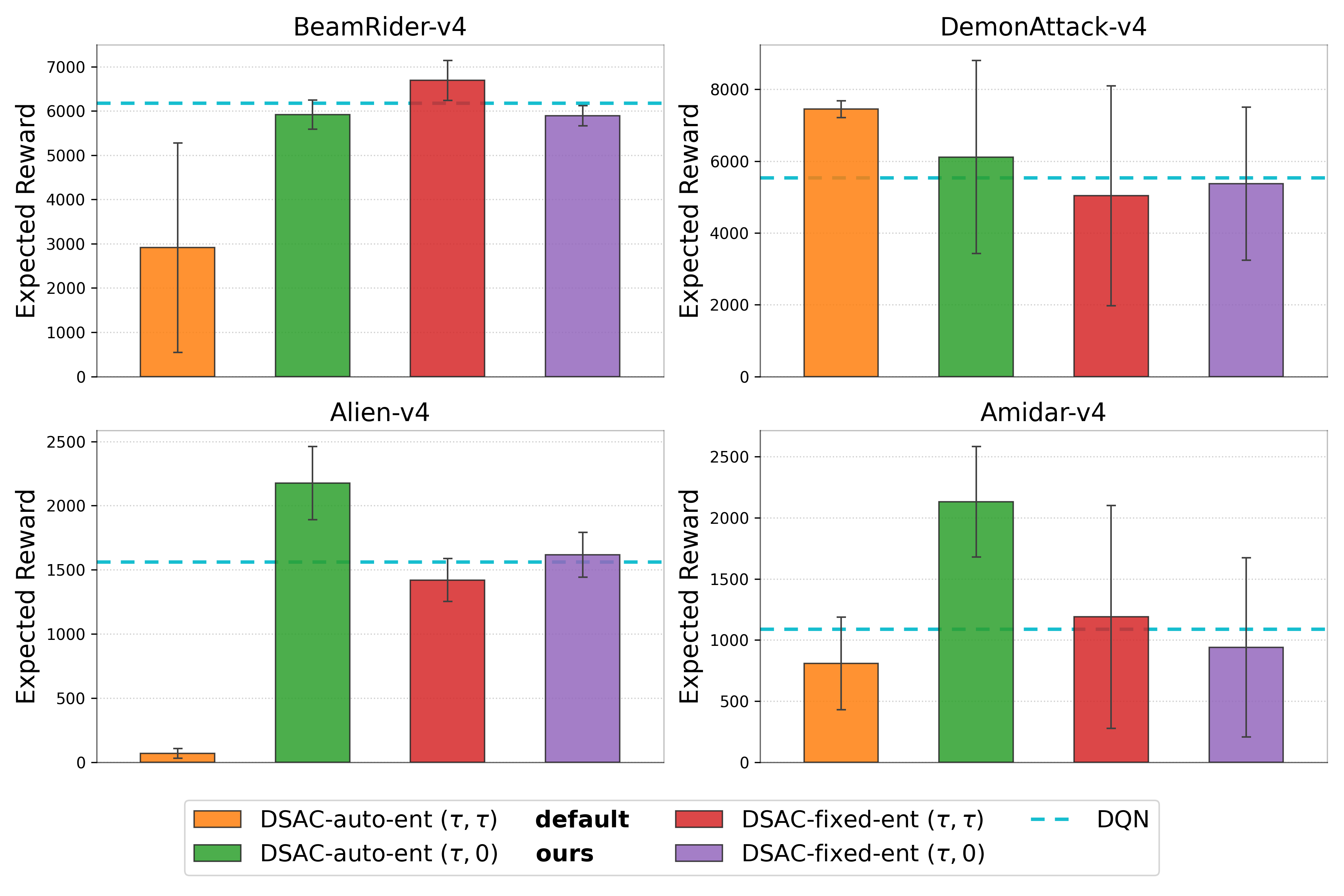}
\end{minipage}%
\hfill
\begin{minipage}{0.4\textwidth}
\caption{\textbf{$\DSAC$ ablation: sensitivity to critic entropy.} Expected reward after $10$M steps.
Training $\DSAC$ without critic entropy ($\zeta=0$) and using the default adaptive entropy coefficient loss (\texttt{auto-ent}) from~\citet{christodoulou2019soft} to set the actor entropy coefficient $\tau$ achieves consistently strong performance relative to $\DQN$ (green vs. blue). In contrast, the default $\DSAC$ ($\zeta=\tau$ with adaptive $\tau$, in orange) performs poorly and improves only with a carefully tuned fixed $\tau$ (red).}    \label{fig:dsac-ablation-m-1}
\end{minipage}
\end{figure*}


On the other hand, the adaptive scheme of~\citet{christodoulou2019soft} sets $\taut \approx \argmin \mathcal{E}_t(\alpha) = \E_{s \sim \cD_t} \E_{a \sim \pi_t(.|s)} [-\alpha \, \ln(\pi_t(a|s)) - \alpha \bar{\cH}]$, thereby encouraging the policy entropy to match the fixed \textit{target entropy} $\bar{\cH}$. When using this adaptive scheme, $\zeta = 0$ yields performance competitive with $\DQN$ (green vs. blue), same as the 20-game results in~\cref{fig:dsac-vs-dqn-20-games-m-1}, while setting $\zeta = \tau$ (the default setting in $\DSAC$) results in much worse performance (orange vs. blue). Hence, to retain strong empirical performance while avoiding manual tuning, we recommend using the adaptive scheme for actor entropy coefficient $\tau$ and setting the critic entropy coefficient $\zeta = 0$. 

For the $\DSAC(\tau,0)$ variant with the adaptive scheme of~\citet{christodoulou2019soft}, we test two prior hypotheses for $\DSAC$'s poor performance. First, we revisit~\citet{zhou2022revisiting}, who attribute $\DSAC$'s instability partly to coupled actor-critic learning and partly to underestimation from taking the minimum of two $Q$ functions. Their \texttt{SD-SAC} uses a tuned fixed entropy coefficient, an entropy-penalty regularizer limiting policy entropy changes across iterates, and clipped double-average $Q$-learning. However,~\cref{app:double-q-results,tab:normalized_score_double_q_sd_sac} shows that taking the minimum of two $Q$ functions does not hurt $\DSAC(\tau,0)$: both single-$Q$ and double-$Q$ variants significantly outperform \texttt{SD-SAC}.

Second, we revisit the claim of~\citet{xu2021target} that a fixed target entropy $\bar{\cH}$ harms $\DSAC$ with adaptive $\tau$, and show that our configuration substantially outperforms their algorithm (\cref{app:adaptive-target-ent-results,tab:normalized_score_adaptive_target_ent}). We further evaluate our conclusion on hard-exploration Atari games, where $\DSAC(\tau,0)$ outperforms baselines including $\texttt{SD-SAC}$ and $\texttt{TES-SAC}$ (\cref{app-hard-exploration,tab:normalized_score_hard_exploration_summary}), and on \texttt{stable-retro} games with substantially larger action spaces. 
Our results demonstrate the benefit of disabling critic entropy while using an adaptive actor entropy persists beyond Atari (\cref{app:dsac-ablation-stable-retro}).


\paragraph{Conclusion 2: All Objectives from our Proposed Framework Benefit from Setting $\zeta = 0$.}
From \cref{fig:all-objectives-with-without-critic-ent-20-games-m-1} in \cref{app:soft-vs-hard-bellman-results}, we observe that when the actor entropy is retained and an adaptive entropy coefficient $\tau$ (as in $\DSAC$) is used, all derived objectives with $\zeta = 0$ achieve improved performance relative to their counterparts that use $\zeta = \tau$. 
\textit{Based on our observations, for the proposed objectives, we recommend using $\zeta = 0$ and adaptive $\tau$ for actor entropy.} 

\paragraph{Conclusion 3: Entropy Regularization is Only Sometimes Important:} 
We study the effect of disabling entropy regularization for both actor and critic (i.e., $\tau=0$ and $\zeta=0$) when the actor objective $\ell_t(\theta)$ is optimized with $n$ gradient steps. For $n=1$, all objectives, including $\DSAC$, suffer from a sharp drop in policy entropy and degraded performance (see~\cref{fig:all-objectives-no-ent-4-games-ent-metric-m-1-10,fig:all-objectives-no-ent-20-games-m-1}), indicating insufficient exploration in the absence of entropy terms. However, unlike $\DSAC$, increasing $n$ alleviates this effect for objectives derived from our proposed framework: setting $n=10$ preserves higher entropy and yields more stable learning, resulting in improved performance (see~\cref{tab:iqm_summary_no_ent}). Per-game final expected rewards are shown in~\cref{fig:spma-vs-npg-vs-baselines-no-ent-m-10-12-games} of~\cref{app:ent-reg-kl-ablation}. 

The gap between $\DSAC$ and other objectives can be explained by the lack of regularization in $\DSAC$. In particular, recall that $\DSAC$ is a limiting case of $\NPGRKL$ with $\etat \to \infty$, and hence does not have the reverse KL regularization term (see the $\NPGRKL$ objective in~\cref{tab:objectives}). Consequently, for $n>1$, $\NPGRKL(0,0)$ benefits from additional optimization steps under KL regularization, whereas $\DSAC(0,0)$ lacks both the KL and entropy terms.

\begin{table}[t]
\centering
\setlength{\tabcolsep}{4pt}

\begin{minipage}[t]{0.57\columnwidth}
\vspace{0pt}
\centering
\resizebox{0.9\linewidth}{!}{%
\begin{tabular}{lccc}
\toprule
Method & IQM & CI Low & CI High \\
\midrule
NPG-FKL (0,0), $n=10$  & \textbf{1.123} & 1.061 & 1.182 \\
NPG-FKL (0,0), $n=1$   & -0.003 & -0.007 & 0.001 \\
\cmidrule(lr){1-4}
NPG-RKL (0,0), $n=10$  & 1.085 & 0.990 & 1.172 \\
NPG-RKL (0,0), $n=1$   & -0.006 & -0.010 & -0.002 \\
\cmidrule(lr){1-4}
SPMA-FKL (0,0), $n=10$ & 0.907 & 0.841 & 0.983 \\
SPMA-FKL (0,0), $n=1$  & -0.001 & -0.005 & 0.003 \\
\cmidrule(lr){1-4}
SPMA-RKL (0,0), $n=10$ & 0.630 & 0.554 & 0.709 \\
SPMA-RKL (0,0), $n=1$  & -0.007 & -0.011 & -0.004 \\
\cmidrule(lr){1-4}
DSAC (0,0), $n=10$     & -0.010 & -0.013 & -0.007 \\
DSAC (0,0), $n=1$      & -0.006 & -0.009 & -0.002 \\
\bottomrule
\end{tabular}%
}
\end{minipage}\hfill
\begin{minipage}[t]{0.4\columnwidth}
\vspace{0pt}
\captionsetup{type=table}
\caption{IQM summary of final human-normalized scores with 95\% stratified bootstrap confidence intervals in the zero-entropy setting $(\tau=0 \,\,\text{and}\,\, \zeta=0)$: unlike $\DSAC$, objectives from our framework recover strong performance in the absence of entropy regularization when the number of actor gradient updates is increased. Our novel $\NPGFKL$ achieves the strongest performance}
\label{tab:iqm_summary_no_ent}
\end{minipage}
\end{table}

\paragraph{Conclusion 4: Objectives from Our Proposed Framework are Comparable to Baselines:}
To evaluate the empirical effectiveness of the derived objectives, we compare against established baselines $\PPO$~\citep{schulman2017proximal}, $\DQN$~\citep{mnih2013playing}, and the recent value-based method $\LMC$~\citep{ishfaq2023provable}.
We select $\LMC$ as a baseline for three reasons.
First, it matches or outperforms strong value-based baselines, including $\NoisyNetDQN$~\citep{fortunato2017noisy}, $\CFiveOne$~\citep{bellemare2017distributional}, $\QRDQN$~\citep{dabney2017distributional}, and $\IQN$~\citep{dabney2018implicit}, across a broad set of Atari games.
Second, $\LMC$ avoids $\epsilon$-greedy exploration, which is known to be brittle to tune~\citep{hessel2018rainbow}. Third, the method is implemented as a modification of Adam~\citep{kingma2014adam}, making it straightforward to integrate into the same \texttt{stable-baselines3} codebase used for all methods in our framework. This choice allows for a controlled comparison that isolates algorithmic differences rather than implementation artifacts.

The IQM results across all games in~\cref{tab:normalized_score_summary} show that our novel $\NPGFKL(\tau, 0)$ achieves the strongest overall performance. Per-game final expected rewards are provided in~\cref{app:20-games-final-performance}. Full training curves are shown in Figures~\ref{fig:ppo-long}--\ref{fig:lmc-long} in Appendix~\ref{app:full-horizon-curves}.
We emphasize that our goal is to evaluate the objectives derived from our principled framework, rather than to construct a fully compositional agent such as $\RAINBOW$~\citep{hessel2018rainbow}.
Standard enhancements, including prioritized replay, multi-step returns, distributional critics, dueling architectures, NoisyNets, and double $Q$-learning might further improve the performance. However, such modifications are orthogonal to our contributions, and we leave them to future work.



\begin{table}[t]
\centering
\setlength{\tabcolsep}{4pt}

\begin{minipage}[t]{0.57\columnwidth}
\vspace{0pt}
\centering
\resizebox{0.9\linewidth}{!}{%
\begin{tabular}{lccc}
\toprule
Method & IQM & CI Low & CI High \\
\midrule
NPG-FKL $(\tau,0)$  & \textbf{1.2193} & 1.1603 & 1.2827 \\
NPG-RKL $(\tau,0)$  & 1.0092 & 0.8784 & 1.1430 \\
SPMA-FKL $(\tau,0)$ & 1.0088 & 0.8970 & 1.1301 \\
SPMA-RKL $(\tau,0)$ & 1.0529 & 0.9659 & 1.1561 \\
DSAC $(\tau,0)$     & 0.9226 & 0.8228 & 1.0288 \\
\cmidrule(lr){1-4}
Adam LMCDQN (Baseline) & 1.0363 & 0.9423 & 1.1265 \\
DQN (Baseline)         & 0.7827 & 0.7563 & 0.8091 \\
PPO (Baseline)         & 0.6835 & 0.6258 & 0.7311 \\
\bottomrule
\end{tabular}%
}
\end{minipage}\hfill
\begin{minipage}[t]{0.4\columnwidth}
\vspace{0pt}
\captionsetup{type=table}
\caption{IQM summary of final human-normalized scores with 95\% stratified bootstrap confidence intervals: comparing our objectives with actor entropy enabled ($\tau \ne 0$), critic entropy disabled ($\zeta=0$), and $n=1$ actor update against baselines. Our novel $\NPGFKL(\tau,0)$ performs best.}
\label{tab:normalized_score_summary}
\end{minipage}
\end{table}

\paragraph{Conclusion 5: Empirical Results Tend to Favour the Forward KL Direction:} \citet{chan2022greedification} have studied the impact of the KL direction in the discrete-action setting when the intermediate policy is Boltzmann i.e. $\ppith \propto \exp(q_\tau^t(s,a)/\tau)$, and have reported no significant differences in the performance. We revisit this question in the context of our actor objectives, evaluating $\SPMA$ and $\NPG$ under two regimes: (i) $\zeta=0, \tau=0$ and (ii) $\zeta=0$, $\tau \neq 0$ (i.e., only actor entropy). For (i), we use $n = 10$ (as $n = 1$ performs poorly) and observe that using the forward KL yields a clear advantage for $\SPMA$, while the difference is small for $\NPG$  (~\cref{tab:iqm_summary_no_ent}). For (ii), forward and reverse KL perform similarly for $\SPMA$, whereas $\NPG$ benefits more from the forward KL direction (~\cref{tab:normalized_score_summary}).

\section{Conclusion and Future Work}
\label{sec:conclusion-and-future-work}
We revisited off-policy actor-critic design in discrete-action settings, where $\DSAC$ performs poorly. By decoupling actor and critic entropy, we identified a $\DSAC$ variant that substantially improves over the default. Motivated by this, we introduced a flexible off-policy actor-critic framework that subsumes $\DSAC$ and yields novel actor objectives, including $\SPMARKL$ and $\NPGFKL$, that remain robust even without entropy regularization or explicit exploration.

Future work includes:  
(i) analyzing these actor objectives with function approximation to guide algorithm design;  
(ii) developing adaptive schemes for the actor step size~$\etat$;  
(iii) further studying the objectives without entropy regularization; 
(iv) incorporating enhancements from compositional algorithms such as $\RAINBOW$; and  
(v) extending the framework to continuous-action settings.

\bibliography{ref}
\bibliographystyle{iclr2026_conference}

\clearpage
\appendix
\appendix
\newcommand{\appendixTitle}{%
\vbox{
    \centering
	\hrule height 4pt
	\vskip 0.2in
	{\LARGE \bf Supplementary Material}
	\vskip 0.2in
	\hrule height 1pt 
}}
\appendixTitle
\section*{Organization of the Appendix}\label{appendix:org}
\begin{itemize}
   \item[\ref{app:actor-objective-instantiations}] \hyperref[app:actor-objective-instantiations]{Actor Objective Instantiations}

   \item[\ref{app:pseudocode}] \hyperref[app:pseudocode]{General Off-Policy Actor-Critic Pseudocode}

   \item[\ref{app:theoretical-results}] \hyperref[app:theoretical-results]{Theoretical Results}
   \begin{itemize}
        \item[\ref{app:meta-entropy-proof}] \hyperref[app:meta-entropy-proof]{Proof of~\cref{thm:meta-entropy}}

        \item[\ref{app:cor-pe-proof}] \hyperref[app:cor-pe-proof]{Proof of~\cref{cor:npg-spma-pe-no-entropy}}
        \begin{itemize}
             \item[\ref{app:generic-pe-proof}] \hyperref[app:generic-pe-proof]{Generic Policy Evaluation}

             \item[\ref{app:pe-soft-npg-cor-proofs}] \hyperref[app:pe-soft-npg-cor-proofs]{(soft) $\NPG$ Corollaries}

             \item[\ref{app:pe-soft-spma-cor-proofs}] \hyperref[app:pe-soft-spma-cor-proofs]{(soft) $\SPMA$ Corollaries}
             
        \end{itemize}

        \item[\ref{app:cor-po-proof}] \hyperref[app:cor-po-proof]{Proof of~\cref{cor:npg-spma-regret-guarantee}}
        \begin{itemize}
             \item[\ref{app:generic-po-proof}] \hyperref[app:generic-po-proof]{Generic Regret Bound}

             \item[\ref{app:po-soft-npg-cor-proofs}] \hyperref[app:po-soft-npg-cor-proofs]{(soft) $\NPG$ Corollaries}

             \item[\ref{app:po-soft-spma-cor-proofs}] \hyperref[app:po-soft-spma-cor-proofs]{(soft) $\SPMA$ Corollaries}
             
        \end{itemize}

        \item[\ref{app:thm-main-result-proof}] \hyperref[app:thm-main-result-proof]{Proof of~\cref{thm:soft-npg-pe-entropy-main}}
        \begin{itemize}
             \item[\ref{app:soft-npg-main-thm-proof}] \hyperref[app:soft-npg-main-thm-proof]{soft $\NPG$}

             \item[\ref{app:soft-spma-main-thm-proof}] \hyperref[app:soft-spma-main-thm-proof]{soft $\SPMA$}
             
        \end{itemize}

   \end{itemize}

   \item[\ref{app:helper-lemmas}] \hyperref[app:helper-lemmas]{Helper Lemmas}

   \item[\ref{app:experimental-details-and-additional-results}] \hyperref[app:experimental-details-and-additional-results]{Experimental Details and Additional Results}

\end{itemize}
\section{Actor Objective Instantiations}
\label{app:actor-objective-instantiations}
In this section, we instantiate our forward and reverse KL-based objectives in Equations~\ref{eq:policy-update-fkl} and~\ref{eq:policy-update-rkl} using soft $\NPG$, and soft $\SPMA$ in the function approximation setting. 
\subsection{\texorpdfstring{$\NPGFKL (\tau, \zeta)$}{NPGFKL (τ, ζ)}} 
\begin{align*}
\pitt &= \argmin_{\pi_{\theta} \in \Pi} \E_{s \sim \cD_t} \, \left[\text{KL}(\pith(\cdot | s) || \pi_{\theta}(\cdot | s)) - \tau_t \cH(\pi_{\theta}(\cdot|s)) \right]\\
& = \scalebox{0.95}{$\argmin_{\pi_{\theta} \in \Pi} \E_{s \sim \cD_t} \, \left[ \E_{a \sim \pi_t(\cdot | s)} \left[ - \, \frac{\exp(\etat \, q_\zeta^t(s, a))}{\cZ_t(s)} \, \ln\left(\frac{\pi_{\theta}(a|s)}{\frac{\pit(a|s) \, \exp(\etat q^t_\zeta(s,a))}{\cZ_t(s)}}\right) \right] - \tau_t \cH(\pi_{\theta}(\cdot|s)) \right]$} \tag{Using the NPG update in~\cref{eq:npg}} \\
& = \argmin_{\pi_{\theta} \in \Pi} \E_{s \sim \cD_t} \, \left[ \E_{a \sim \pi_t(\cdot | s)} \left[ - \, \frac{\exp(\etat \, q_\zeta^t(s, a))}{\cZ_t(s)} \, \ln\left(\frac{\pi_{\theta}(a|s)}{\pi_t(a|s)}\right) \right] - \tau_t \cH(\pi_{\theta}(\cdot|s)) \right] \tag{dropping the constant w.r.t $\pi$}\\
& = \scalebox{0.99}{$\argmax_{\pi_{\theta} \in \Pi} \E_{s \sim \cD_t} \, \left[ \E_{a \sim \pi_t(\cdot | s)} \left[ \frac{\exp(\etat \, q_\zeta^t(s, a))}{\sum_{a'} \pit(a') \, \exp(\etat \, q_\zeta^t(s,a))} \, \ln\left(\frac{\pi_{\theta}(a|s)}{\pi_t(a|s)}\right) \right] + \tau_t \cH(\pi_{\theta}(\cdot|s)) \right]$}\\ 
\end{align*}

\subsection{\texorpdfstring{$\SPMAFKL (\tau, \zeta)$}{SPMAFKL (τ, ζ)}} 
\begin{align*}
\pitt &= \argmin_{\pi_{\theta} \in \Pi} \E_{s \sim \cD_t} \, \left[\text{KL}(\pith(\cdot | s) || \pi_{\theta}(\cdot | s)) - \tau_t \cH(\pi_{\theta}(\cdot|s)) \right]\\
& = \scalebox{0.99}{$\argmin_{\pi_{\theta} \in \Pi} \E_{s \sim \cD_t} \, \Bigg[ \E_{a \sim \pi_t(\cdot | s)} \left[- \, \frac{\left(1 + \etat (q_\zeta^t(s, a) - v_\zeta(s)^t) \right)}{\cZ_t(s)} \, \ln\left(\frac{\pi_{\theta}(a|s)}{\frac{\pit(a|s) \, (1 + \etat \, (q^t_\zeta(s, a) - v^t_\zeta(s)))}{\cZ_t(s)}}\right) \right]$}\\
&\hspace{4.9cm} - \tau_t \cH(\pi_{\theta}(\cdot|s))  \Bigg] \tag{Using the SPMA update in~\cref{eq:spma}} \\
& = \argmin_{\pi_{\theta} \in \Pi} \E_{s \sim \cD_t} \, \Bigg[ \E_{a \sim \pi_t(\cdot | s)} \left[ - \, \frac{\left(1 + \etat (q_\zeta^t(s, a) - v_\zeta^t(s)) \right)}{\cZ_t(s)} \, \ln\left(\frac{\pi_{\theta}(a|s)}{\pi_t(a|s)}\right) \right]\\
&\hspace{4.5 cm} - \tau_t \cH(\pi_{\theta}(\cdot|s)) \Bigg] \tag{dropping the constant w.r.t $\pi$}\\
& = \argmax_{\pi_{\theta} \in \Pi} \E_{s \sim \cD_t} \, \Bigg[ \E_{a \sim \pi_t(\cdot | s)} \left[ \frac{\left(1 + \etat (q_\zeta^t(s, a) - v_\zeta^t(s)) \right)}{\sum_{a'} \, \pit(a'|s) \, \left(1 + \etat (q_\zeta^t(s, a') - v_\zeta^t(s)) \right)} \, \ln\left(\frac{\pi(a|s)}{\pi_t(a|s)}\right) \right]\\
&\hspace{4.5cm} + \tau_t \cH(\pi_{\theta}(\cdot|s)) \Bigg] \end{align*}

\subsection{\texorpdfstring{$\NPGRKL (\tau, \zeta)$}{NPGRKL (τ, ζ)}}
\begin{align*}
\pitt &= \argmin_{\pi_{\theta} \in \Pi} \E_{s \sim \cD_t} \, \left[\text{KL}(\pi_{\theta}(\cdot | s) || \pith(\cdot | s)) - \tau_t \cH(\pi_{\theta}(\cdot|s)) \right]\\
&= \argmin_{\pi_{\theta} \in \Pi} \E_{s \sim \cD_t} \, \E_{a \sim \pi_{\theta}(.|s)} \left[(1+\taut) \, \ln(\pi_{\theta}(a|s)) - \ln(\pith(a|s)) \right]\\
&= \argmin_{\pi_{\theta} \in \Pi} \E_{s \sim \cD_t} \E_{a \sim \pi_{\theta}(.|s)} \left[(1+\taut) \, \ln(\pi_{\theta}(a|s)) - \ln\bigg(\pit(a|s) \exp(\etat \, q_\zeta^t(s, a))\bigg) \right] \tag{Using the $\NPG$ update in~\cref{eq:npg} and since $\cZ_t$ can be marginalized out} \\
& = \argmin_{\pi_{\theta} \in \Pi} \E_{s \sim \cD_t} \E_{a \sim \pi_{\theta}(.|s)} \left[(1+\taut) \, \ln(\pi_{\theta}(a|s)) - \ln(\pit(a|s)) - \etat \, q_\zeta^t(s, a)) \right]\\
& = \argmin_{\pi_{\theta} \in \Pi} \E_{s \sim \cD_t} \E_{a \sim \pi_{\theta}(.|s)} \left[(1+\tau \etat) \, \ln(\pi_{\theta}(a|s))  - \etat \, \left(q_\zeta^t(s, a) + \frac{1}{\etat} \ln(\pit(a|s)) \right) \right] \tag{Since $\taut = \etat \, \tau$} \\
& = \argmax_{\pi_{\theta} \in \Pi} \E_{s \sim \cD_t} \E_{a \sim \pi_{\theta}(.|s)} \left[q_\zeta^t(s, a) - \tau \, \ln(\pi_{\theta}(a|s)) - \frac{1}{\etat} \ln\left(\frac{\pi_{\theta}(a|s)}{\pit(a|s)}\right) \right]\\
&= \argmax_{\pi_{\theta} \in \Pi} \E_{s \sim \cD_t} \left[\E_{a \sim \pi_{\theta}(.|s)} \left[q_\zeta^t(s, a) - \tau \, \ln(\pi_{\theta}(a|s)) \right] - \frac{1}{\etat} \, \text{KL}(\pi_{\theta}(\cdot|s) \, || \, \pi_t(\cdot|s))\right]
\end{align*}     

Given an estimate of the entropy-regularized $q$ function, $\DSAC$ is a special of $\NPGRKL$ by setting $\etat = \infty$ resulting in the following surrogate loss:
\begin{align*}
\pitt &= \argmax_{\pi_{\theta} \in \Pi} \E_{s \sim \cD_t} \E_{a \sim \pi_{\theta}(.|s)} \left[\left(q_\zeta^t(s, a)  \right) - \tau \, \ln(\pi_{\theta}(a|s)) \right]    
\end{align*}


\subsection{\texorpdfstring{$\SPMARKL (\tau, \zeta)$}{SPMARKL (τ, ζ)}} 
\begin{align*}
\pitt &= \argmin_{\pi_{\theta} \in \Pi} \E_{s \sim \cD_t} \, \left[\text{KL}(\pi_{\theta}(\cdot | s) || \pith(\cdot | s)) - \tau_t \cH(\pi_{\theta}(\cdot | s)) \right]\\    
&= \argmin_{\pi_{\theta} \in \Pi} \E_{s \sim \cD_t} \, \E_{a \sim \pi_{\theta}(.|s)} \left[(1+\taut) \, \ln(\pi(a|s)) - \ln(\pith(a|s)) \right]\\
&= \scalebox{0.91}{$
\argmin_{\pi_{\theta} \in \Pi} \E_{s \sim \cD_t} \E_{a \sim \pi_{\theta}(.|s)} \left[(1+\taut) \, \ln(\pi_{\theta}(a|s)) - \ln\bigg(\pit(a|s) \, [1 + \etat \left(q_\zeta^{t}(s,a') - v_\zeta^t(s) \right)]\bigg) \right]
\tag{Using the SPMA update in~\cref{eq:spma}}
$} \\
& = \argmax_{\pi_{\theta} \in \Pi} \mathbb{E}_{s \sim \cD_t} \left[ \mathbb{E}_{a \sim \pi_{\theta}(\cdot|s)} \ln\left(1 + \etat (q_\zeta^t(s,a) - v_\zeta^t(s))\right) - \mathrm{KL}(\pi_{\theta} || \pi_t) - \taut \, \ln \pi_{\theta}(a|s) \right] 
\end{align*}

\section{General Off-Policy Actor-Critic Pseudocode}
\label{app:pseudocode}

\begin{algorithm}[H]
    \caption{General Off-Policy Actor-Critic Framework}
    \label{alg:actor-critic}
    \begin{algorithmic}[1]
    \STATE \textbf{Input:} $\theta_0$ ($\pi_0$'s parameters), $\phi_0$ ($q_\zeta^0$'s parameters), $\pi_{\theta}$ (function approximation for actor), $q_{\phi}$ (function approximation for critic), $\cL_t$ (critic loss), $\ell_t$ (actor loss), $K$ (total iterations), $N$ (number of environment steps), $n$ (number of policy optimization steps), $\alpha$ (inner-loop step-size) 
    \FOR{$t = 0$ \textbf{to} $K-1$}
        \STATE Interact with the environment for $N$ steps to collect data using $\pit$: $D_t \leftarrow D_{t} \cup \{s_i, a_i, r(s_i, a_i), s_{i+1}\}_{i=1}^N$
        \STATE $\phi_t = \argmin \cL_t(\phi) \, ; \, q_\zeta^t = q_{\phi_t}$
        \STATE Initialize inner-loop: $\omega_0 = \theta_t$
        \FOR{$j = 0$ \textbf{to} $n-1$}
        \STATE $\omega_{j+1} = \omega_{j} - \alpha \nabla_{\omega} \ell_t(\omega_j)$
        \ENDFOR
        \STATE $\theta_{t+1} = \omega_n$ 
        \STATE $\ppitt(\cdot|s) = \pi_{\theta_{t+1}}(\cdot|s)$
    \ENDFOR
    \STATE \textbf{Return:} $\theta_K$
    \end{algorithmic}
\end{algorithm}

\section{Theoretical Results}
\label{app:theoretical-results}

\subsection{Proof of~\texorpdfstring{\cref{thm:meta-entropy}}{}}
\label{app:meta-entropy-proof}
\reduction*

\begin{proof}
\begin{align*}
v_\tau^{\pi_\tau^*} - v_\tau^{\pi} & = \cT_{\tau}^{\pi_\tau^*} v_\tau^{\pi_\tau^*}- v_\tau^{\pi} \tag{Since $v_\tau^{\pi_\tau^*}$ is a fixed point of $\cT_{\tau}^{\pi_\tau^*}$} \\
& = [\cT_{\tau}^{\pi_\tau^*} v_\tau^{\pi} - v_\tau^{\pi}] + [\cT_{\tau}^{\pi_\tau^*} v_\tau^{\pi^*}- \cT_{\tau}^{\pi_\tau^*} v_\tau^{\pi}] \tag{Add/subtract $\cT_{\tau}^{\pi_\tau^*} v_\tau^{\pi}$} \\
& = [\cT_{\tau}^{\pi_\tau^*} v_\tau^{\pi} - v_\tau^{\pi}] + \gamma \cP_{\pi_\tau^*} (v_\tau^{\pi_\tau^*} - v_\tau^{\pi}) \tag{Using the definition of $\cT_{\tau}^{\pi_\tau^*}$} \\
\implies v_\tau^{\pi^*_\tau} - v_\tau^{\pi} & = \left(I - \gamma \cP_{\pi^*_\tau}\right)^{-1} \left[\cT_{\tau}^{\pi^*_\tau} v_\tau^\pi - \cT_{\tau}^{\pi} \, v_\tau^\pi \right]
\intertext{Summing up from $t = 0$ to $t = K - 1$ and dividing by $K$,}
& v_\tau^{*} - \frac{\sum_{t = 0}^{K - 1} v_\tau^{\pi_t}}{K} = \frac{1}{K} \, (I - \gamma \cP_{\pi^*_\tau})^{-1} \, \sum_{t = 0}^{K - 1}\left[\cT_\tau^{\pi^*_\tau} \, v_\tau^{\pi_t} - v_\tau^{\pi_t} \right] \tag{By definition of $v_\tau^*$} \\ 
\implies v_\tau^{*}- v_\tau^{\bar{\pi}_K} & = \frac{1}{K} (I - \gamma \cP_{\pi^*_\tau})^{-1} \, \sum_{t = 0}^{K - 1}\left[\cT_\tau^{\pi^*_\tau} \, v_\tau^{\pi_t} - v_\tau^{\pi_t} \right]  \tag{Since $v_\tau^{\bar{\pi}_K} = \frac{\sum_{t = 0}^{K - 1} v_\tau^{\pi_t}}{K}$} \\
& =  \frac{1}{K} \, (I - \gamma \cP_{\pi^*_\tau})^{-1} \, \sum_{t = 0}^{K - 1}\left[\cT_\tau^{\pi^*_\tau} \, v_\tau^{\pi_t} - \cT_\tau^{\pi_t} \, v_\tau^{\pi_t} \right] \tag{Since $v_\tau^\pi = \cT_\tau^\pi v_\tau^\pi$} \\
\implies \norminf{v_\tau^{*}- v_\tau^{\bar{\pi}_K}} &= \frac{1}{K} \norminf{(I - \gamma \cP_{\pi^*_\tau})^{-1} \, \sum_{t = 0}^{K - 1}\left[\cT_\tau^{\pi^*_\tau} \, v_\tau^{\pi_t} - \cT_\tau^{\pi_t} \, v_\tau^{\pi_t} \right]} \\
& \leq \frac{1}{K} \, \norminf{(I - \gamma \cP_{\pi^*_\tau})^{-1}} \, \norminf{\sum_{t = 0}^{K - 1}\left[\cT_\tau^{\pi^*_\tau} \, v_\tau^{\pi_t} - \cT_\tau^{\pi_t} \, v_\tau^{\pi_t} \right]} \tag{By definition of matrix norm}  \\
& \leq \frac{1}{K \, (1-\gamma)} \, \norminf{\sum_{t = 0}^{K - 1}\left[\cT_\tau^{\pi^*_\tau} \, v_\tau^{\pi_t} - \cT_\tau^{\pi_t} \, v_\tau^{\pi_t} \right]} \tag{Since $\norminf{(I - \gamma \cP_\pi)^{-1}} = \norminf{\sum_{t = 0}^{\infty} [\gamma \, \cP_\pi]^{t}} \leq \sum_{t = 0}^{\infty} \gamma^t = \frac{1}{1 - \gamma}$}
\end{align*}
Let us calculate $\left[\cT_{\tau}^{\pi^*_\tau} v_\tau^{\pi_t} - \cT_{\tau}^{\pi_t} \, v_\tau^{\pi_t} \right](s)$. 
\begin{align*}
\left[\cT_{\tau}^{\pi^*_\tau} v_\tau^{\pi_t} - \cT_{\tau}^{\pi_t} \, v_\tau^{\pi_t} \right](s) &= (\cT_{\tau}^{\pi^*_\tau} v_\tau^{\pi_t})(s) - (\cT_{\tau}^{\pi_t} v_\tau^{\pi_t})(s)\\
&= \E_{a \sim \pi^*_\tau} \left[q_\tau^{\pi_t}(s,a) - \tau \, \ln(\pi^*_\tau(a|s)) \right] \\
&\quad - \E_{a \sim \pi_t} \left[q_\tau^{\pi_t}(s,a) - \tau \, \ln(\pi_t(a|s)) \right] \tag{By definition of $\cT_{\tau}^{\pi_{1}} \, v_\tau^{\pi_{2}}$} \\
& = \langle \pi^*_\tau(\cdot|s) - \pi_t(\cdot|s), q_\tau^{\pi_t}(s,\cdot) \rangle \nonumber \\
&\quad - \tau \left[ \langle \pi^*_\tau(\cdot|s), \ln(\pi^*_\tau(\cdot|s)) \rangle 
        - \langle \pi_t(\cdot|s), \ln(\pi_t(\cdot|s)) \rangle \right]\\
& = \langle \pi^*_\tau(\cdot|s) - \pi_t(\cdot|s),  q_\tau^{\pi_t}(s,\cdot) \rangle + \tau \, \left[ \cH(\pi^*_\tau(\cdot|s)) -  \cH(\pi_t(\cdot|s)) \right] \tag{By definition of $\cH(\pi(\cdot|s))$} \\
& = \langle \pi^*_\tau(\cdot|s) - \pi_t(\cdot|s), q_\zeta^t(s,\cdot) \rangle 
    + \tau \left[ \cH(\pi^*_\tau(\cdot|s)) - \cH(\pi_t(\cdot|s)) \right] \nonumber \\
&\quad + \langle \pi^*_\tau(\cdot|s) - \pi_t(\cdot|s), 
    q_\tau^{\pi_t}(s, \cdot) - q_\zeta^t(s,\cdot) \rangle
\tag{Add/Subtract $\langle \pi^*_\tau(\cdot|s) - \pi_t(\cdot|s), q_\zeta^t(s,\cdot) \rangle$}\\
& \leq \langle \pi^*_\tau(\cdot|s) - \pi_t(\cdot|s),  q_\zeta^t(s,\cdot) \rangle + \tau \, \left[ \cH(\pi^*_\tau(\cdot|s)) -  \cH(\pi_t(\cdot|s)) \right]\\ 
&\quad + \normone{\pi^{*}_{\tau}(\cdot|s) - \pi_t(\cdot|s)} \, \norminf{q_\tau^{\pi_t}(s,\cdot) - q_\zeta^t(s,\cdot)} \tag{By Holder's inequality} \\
&\leq \langle \pi^*_\tau(\cdot|s) - \pi_t(\cdot|s),  q_\zeta^{t}(s,\cdot) \rangle + \tau \, \left[ \cH(\pi^*_\tau(\cdot|s)) -  \cH(\pi_t(\cdot|s)) \right]\\
&\quad + 2 \norminf{q_\tau^{\pi_t}(s,\cdot) - q_\zeta^t(s,\cdot)} \tag{Since $\normone{\pi^{*}_{\tau}(\cdot|s) - \pi_t(\cdot|s)} \leq 2$}
\end{align*}
Define $\text{Regret}(K,u,s) = \sum_{t = 0}^{K-1} \left[\langle u(\cdot|s) - \pi_t(\cdot|s),  q_\zeta^{t}(s,\cdot) \rangle + \tau \, \left[ \cH(u(\cdot|s)) -  \cH(\pit(\cdot|s)) \right] \right]$ as the regret incurred for state $s$ when the comparator is policy $u$. Hence, 
\begin{align*}
\sum_{t = 0}^{K - 1} \left[\cT_{\tau}^{\pi^*_\tau} v_\tau^{\pi_t} - \cT_{\tau}^{\pi_t} \, v_\tau^{\pi_t} \right](s) & \le \text{Regret}(K,\pi^*_\tau,s) + 2\, \sum_{t=0}^K\norminf{q_\tau^{\pi_t}(s,\cdot) - q_\zeta^t(s,\cdot)} \\
\norminf{\sum_{t = 0}^{K - 1} \left[\cT_{\tau}^{\pi^*_\tau} v_\tau^{\pi_t} - \cT_{\tau}^{\pi_t} \, v_\tau^{\pi_t} \right]} &\le \max_{s} \text{Regret}(K,\pi^*_\tau,s) + 2\, \sum_{t=0}^K \norminf{\bepsilon_t} \tag{By definition of $\bepsilon_t$}
\end{align*}
Using the definition of $\text{Regret}(K) = \big[ \text{Regret}(K,\pi^*_\tau,s_i) \big]_{i=1}^S \in \R^{S}$, 
\begin{align*}
\norminf{v_\tau^{*} - v_\tau^{\bar\pi_K}} &\leq  \frac{\norminf{\text{Regret}(K)}}{K \, (1 - \gamma)} +  \frac{2 \, \sum_{t \in [K]} \norminf{\bepsilon_t}}{K \, (1 - \gamma)}     
\end{align*}
\end{proof}

In the special case $(0, \zeta)$, i.e., no entropy regularization from the actor side, the problem reduces to online linear optimization, as in Politex~\citep{abbasi2019politex}, yielding the following corollary. 

\begin{corollary}[Generic Reduction without Actor Entropy]
If $\pi^*$ is the optimal policy whose value function is equal to $v^*$, for an estimate of  $q_\tau^{\pi_t}$ at iteration $t$ s.t. $\bepsilon_t = q_{\zeta}^t - q_\tau^{\pi_t}$ and for any sequence of policies $\{\pi_0, \pi_1, \ldots, \pi_{K-1}\}$, if $\bar{\pi}_K$ is the corresponding mixture policy, then,  
\begin{align*}
\norminf{v^{*} - v^{\bar \pi_K}} &\leq \frac{\norminf{\text{Regret}(K)}}{K \, (1 - \gamma)} +  \frac{2 \, \sum_{t \in [K]} \norminf{\bepsilon_t}}{K \, (1 - \gamma)}   \,,
\end{align*}
where $(\text{Regret}(K))(s) :=  \sum_{t = 0}^{K-1} \left[\langle \pi^*(\cdot|s) - \pi_t(\cdot|s),  q_\zeta^t (s,\cdot) \rangle \right]$ is the regret incurred on an online linear optimization problem for each state $s \in \cS$. 
\label{cor:meta-no-entropy}
\end{corollary}

\subsection{Proof of~\texorpdfstring{\cref{cor:npg-spma-pe-no-entropy}}{}}
\label{app:cor-pe-proof}

\subsubsection{Generic Policy Evaluation}
\label{app:generic-pe-proof}
\begin{theorem}[Generic policy evaluation]
Using the policy evaluation scheme in~\cref{eq:pe-general}, if $\delta(\tau, \zeta):= \frac{|\tau - \zeta| \, \ln(A)}{1 - \gamma}$, and $TV_i := \norm{\pi_i(\cdot|s) - \pi_{i - 1}(\cdot|s)}_{1}$,  then, for all $t \in [K]$, 
\begin{align*}
\epsilon_t 
&:= \norminf{\bepsilon_{t}} 
   = \norminf{q_\tau^{\pi_t} - q_\zeta^t} \\
&\leq \frac{H_\tau \, \gamma^m}{1 - \gamma} 
   \sum_{i = 1}^{t}  (\gamma^m)^{t - i} 
   \max_{s} \Big[ 
       TV_i 
       + \tau (1 - \gamma) 
         \,\big|\cH(\pi_i(\cdot|s)) - \cH(\pi_{i-1}(\cdot|s))\big| 
   \Big] \\
&\quad + \sum_{i = 0}^{t} (\gamma^m)^{t - i} \, \delta(\tau, \zeta) 
\end{align*}
Furthermore, if for all $i \in [t]$, $TV_i \leq \frac{1}{2}$, then, for any constant $C \in (0,1/2)$, 
\begin{align*} \epsilon_t & \leq \frac{H_\tau \, \gamma^m}{(1 - \gamma)} \, \sum_{i = 1}^{t} (\gamma^m)^{t - i} \, \max_{s} \left[ TV_i + \tau (1 - \gamma) \, \left[ TV_i \, \ln\left(\frac{A}{C} \right) + \left(\frac{\ln(A)}{2} + \sqrt{2} \right) \, \sqrt{C} \right] \right] \\ & + \sum_{i = 0}^{t} (\gamma^m)^{t - i} \, \delta(\tau, \zeta) \end{align*}
\label{thm:error-analysis-general}
\end{theorem}
\begin{proof}
\begin{align*}
\epsilon_t & = \norminf{q_\tau^{\pi_t} - q_\zeta^t} = \norminf{q_\tau^{\pi_t} - \cP_{[0, H_\tau]}[(T_\zeta^{\pi_t})^{m} q_\zeta^{t-1}]} \tag{Using the update} \\
& = \norminf{\cP_{[0, H_\tau]}[q_\tau^{\pi_t}] - \cP_{[0, H_\tau]}[(T_\zeta^{\pi_t})^{m} q_\zeta^{t-1}]} \tag{Since $q_\tau^{\pi_t} \in [0, H_\tau]$} \\
& = \norminf{\cP_{[0, H_\tau]}[(T_\tau^{\pi_t})^{m} q_\tau^{\pi_t}] - \cP_{[0, H_\tau]}[(T_\zeta^{\pi_t})^{m} q_\zeta^{t-1}]} \tag{Since $q_\tau^{\pi_t} = T_\tau^{\pi_t} q_\tau^{\pi_t}$} \\
& \leq \norminf{(T_\tau^{\pi_t})^{m} q_\tau^{\pi_t} - (T_\zeta^{\pi_t})^{m} q_\zeta^{t-1}} \tag{Since projections are non-expansive} \\ 
& \leq \norminf{(T_\tau^{\pi_t})^{m} q_\tau^{\pi_t} - (T_\tau^{\pi_t})^{m} q_\zeta^{t-1}} + \norminf{(T_\tau^{\pi_t})^{m} q_\zeta^{t-1} - (T_\zeta^{\pi_t})^{m} q_\zeta^{t-1}} \tag{Add/Subtract $(T_\tau^{\pi_t})^{m} q_\zeta^{t-1}$ and using triangle inequality} \\
& \leq \norminf{(T_\tau^{\pi_t})^{m} q_\tau^{\pi_t} - (T_\tau^{\pi_t})^{m} q_\zeta^{t-1}} + \underbrace{\frac{|\zeta - \tau|}{1 - \gamma} \, \ln(A)}_{:= \delta(\tau, \zeta)} \tag{Using~\cref{lemma:bellman-difference}}\\
& = \norminf{(T_\tau^{\pi_t})^{m} q_\tau^{\pi_t} - (T_\tau^{\pi_t})^{m} q_\zeta^{t-1}}  + \delta(\tau, \zeta) \\
& \leq \gamma^{m} \, \norminf{q_\tau^{\pi_t} - q_\zeta^{t-1}} + \delta(\tau, \zeta) \tag{Since $T_\tau^\pi$ is a $\gamma$ contraction} \\
& \leq \gamma^{m} \, \norminf{q_\tau^{\pi_t} - q_\tau^{\pi_{t-1}}} + \gamma^{m} \, \norminf{q_\tau^{\pi_{t-1}} - q_\zeta^{t-1}} + \delta(\tau, \zeta) \tag{Add/subtract $q_\tau^{\pi_{t-1}}$ and using triangle inequality} \\
& = \gamma^{m} \, \norminf{q_\tau^{\pi_t} - q_\tau^{\pi_{t-1}}} + \gamma^{m} \, \epsilon_{t-1} + \delta(\tau, \zeta)
\end{align*}
The first term is the difference in the $q_\tau$ functions between two consecutive policies, and we bound it next. For all $(s,a)$, 
\begin{align*}
q_\tau^{\pi_t}(s,a) - q_\tau^{\pi_{t-1}}(s,a) &= \gamma \, \sum_{s'} \cP(s'|s,a) v_\tau^{\pi_t}(s') - \gamma \, \sum_{s'} \cP(s'|s,a) v_\tau^{\pi_{t-1}}(s')\\
&= \gamma \, \E_{s' \sim \cP(\cdot|s,a)} [v_\tau^{\pi_{t}}(s') - v_\tau^{\pi_{t-1}}(s')] \\
\implies \norminf{q_\tau^{\pi_t} - q_\tau^{\pi_{t-1}}} & \leq \gamma \norminf{v_\tau^{\pi_{t}} - v_\tau^{\pi_{t-1}}}
\end{align*}
Let us now bound the difference in the $v_\tau$ functions between two consecutive policies. 
\begin{align*}
\norminf{v_\tau^{\pi_{t}} - v_\tau^{\pi_{t-1}}} &= \norminf{T_\tau^{\pi_t} v_\tau^{\pi_t} - T_\tau^{\pi_t} v_\tau^{\pi_{t-1}} + T_\tau^{\pi_t} v_\tau^{\pi_{t-1}} - T_\tau^{\pi_{t-1}} v_\tau^{\pi_{t-1}}} \tag{Since $T_\tau^{\pi} v_\tau^{\pi} = v^\pi_\tau$ and add/subtract $T_\tau^{\pi_t} v_\tau^{\pi_{t-1}}$} \\
& \leq \norminf{T_\tau^{\pi_t} v_\tau^{\pi_t} - T_\tau^{\pi_t} v_\tau^{\pi_{t-1}}} + \norminf{T_\tau^{\pi_t} v_\tau^{\pi_{t-1}} - T_\tau^{\pi_{t-1}} v_\tau^{\pi_{t-1}}} \tag{Triangle inequality} \\
& \leq \gamma \, \norminf{v_\tau^{\pi_t} - v_\tau^{\pi_{t-1}}} + \norminf{T_\tau^{\pi_t} v_\tau^{\pi_{t-1}} - T_\tau^{\pi_{t-1}} v_\tau^{\pi_{t-1}}} \tag{Since $T_\tau^\pi$ is a $\gamma$-contraction} \\
\implies \norminf{v_\tau^{\pi_{t}} - v_\tau^{\pi_{t-1}}} & \leq \frac{1}{1 - \gamma} \, \norminf{T_\tau^{\pi_t} v_\tau^{\pi_{t-1}} - T_\tau^{\pi_{t-1}} v_\tau^{\pi_{t-1}}}
\end{align*}
In order to bound $\norminf{T_\tau^{\pi_t} v_\tau^{\pi_{t-1}} - T_\tau^{\pi_{t-1}} v_\tau^{\pi_{t-1}}}$, consider a fixed state $s$.  By definition of $T_\tau^{\pi}$, 
\begin{align*}
T_\tau^{\pi_t} v_\tau^{\pi_{t-1}}(s) - T_\tau^{\pi_{t-1}} v_\tau^{\pi_{t-1}}(s) &= \langle \pi_t(\cdot|s) - \pi_{t - 1}(\cdot|s), r(s, \cdot) \rangle\\
&\quad + \gamma \sum_{a} [\pi_t(a|s) - \pi_{t-1}(a|s)] \, \E_{s' \sim \cP(\cdot|s,a)} v_\tau^{\pi_{t-1}}(s') \\ &\quad + \tau [\cH(\pi_t(\cdot|s)) - \cH(\pi_{t-1}(\cdot|s))] \\
& \leq \norm{\pi_t(\cdot|s) - \pi_{t - 1}}_1 \, \norm{r(s, \cdot)}_\infty \\
&\quad + \gamma \, \norm{\pi_t(\cdot|s) - \pi_{t - 1}}_1 \, \norminf{v_\tau^{\pi_{t-1}}}\\ 
&\quad+ \tau [\cH(\pi_t(\cdot|s)) - \cH(\pi_{t-1}(\cdot|s))] \tag{By Holder's inequality} \\
& \leq \left(1 + \gamma \, H_\tau \right) \, \norm{\pi_t(\cdot|s) - \pi_{t - 1}(\cdot|s)}_{1} \\
&\quad + \tau \vert \cH(\pi_t(\cdot|s)) - \cH(\pi_{t-1}(\cdot|s)) \vert \tag{Since rewards are in $[0,1]$ and $v_\tau^{\pi_{t-1}}(s) \leq H_\tau$} \\
& \leq H_\tau \, \norm{\pi_t(\cdot|s) - \pi_{t - 1}(\cdot|s)}_{1} + \tau \vert\cH(\pi_t(\cdot|s)) - \cH(\pi_{t-1}(\cdot|s)) \vert \tag{Since $1 \leq 1 + \tau \ln(A) = (1-\gamma) \, H_\tau$} \\
\implies \norminf{T_\tau^{\pi_t} v_\tau^{\pi_{t-1}} - T_\tau^{\pi_{t-1}} v_\tau^{\pi_{t-1}}} & \leq H_\tau \, \left( \max_{s} \left[ TV_i + \tau (1 - \gamma) \, \vert \cH(\pi_t(\cdot|s)) - \cH(\pi_{t-1}(\cdot|s)) \vert \right] \right) \tag{Since $1 \le (1-\gamma) \, H_\tau$} \\
\end{align*}
Combining the above inequalities, 
\begin{align*} 
\norminf{v_\tau^{\pi_{t}} - v_\tau^{\pi_{t-1}}} & \leq \frac{H_\tau \,}{(1 - \gamma)} \, \max_{s} \left[ TV_i + \tau (1 - \gamma) \, \vert \cH(\pi_t(\cdot|s)) - \cH(\pi_{t-1}(\cdot|s)) \vert \right]\\  
\implies \norminf{q_\tau^{\pi_t} - q_\tau^{\pi_{t-1}}} & \leq  \frac{H_\tau}{(1 - \gamma)} \,\, \max_{s} \left[ TV_i + \tau (1 - \gamma) \,  \vert \cH(\pi_t(\cdot|s)) - \cH(\pi_{t-1}(\cdot|s)) \vert\right] \\
\end{align*}
\begin{align*}    
\implies \epsilon_t &\leq \underbrace{\frac{H_\tau \, \gamma^{m}}{(1 - \gamma)} \,  \, \max_{s} \left[ TV_i + \tau (1 - \gamma) \,\vert \cH(\pi_t(\cdot|s)) - \cH(\pi_{t-1}(\cdot|s)) \vert \right]}_{:= B_t} + \gamma^{m} \, \epsilon_{t-1} + \delta(\tau, \zeta)\\
\implies \epsilon_t & \leq B_t + \gamma^{m} \, \epsilon_{t-1} + \delta(\tau, \zeta)
\end{align*}
Bounding $\epsilon_0$, 
\begin{align*}
\epsilon_0 & = q_\tau^{\pi_0} - q_\zeta^{\pi_0} = \E_{s' \sim \cP(\cdot|s,a)} \left[ r(s,a) + \gamma v_\tau^\pi(s') \right] - \E_{s' \sim \cP(\cdot|s,a)} \left[ r(s,a) + \gamma v_\zeta^\pi(s') \right] \tag{By definition} \\
& = \gamma \E_{s' \sim \cP(\cdot|s,a)} [v_\tau^{\pi_0}(s') - v_\zeta^{\pi_0}(s')] \\
& = \gamma \E_{s' \sim \cP(\cdot|s,a)} \Bigg[ [v^{\pi_0}(s') + \tau \, \sum_{t = 0}^\infty \gamma^t \left[ \cH(\pi_0(\cdot|s_t)) | s_0 = s' \right]\\ 
&\hspace{2.1cm} - [v^{\pi_0}(s') + \zeta \, \sum_{t = 0}^\infty \gamma^t \left[ \cH(\pi_0(\cdot|s_t)) | s_0 = s' \right]] \Bigg] \tag{By definition} \\
& \leq \frac{|\tau - \zeta| \, \ln(A)}{1 - \gamma} = \delta(\tau, \zeta) \tag{Since $\cH(\pi(\cdot|s)) \leq \ln(A)$}
 \end{align*}
For a fixed $t \in [K]$, recursing from $i = t-1$ to $i = 1$ and using that
\begin{align*}
\epsilon_{t}  &\leq (\gamma^m)^{t} \epsilon_0 + \sum_{i = 1}^{t} (\gamma^m)^{t - i} \, (B_i + \delta(\tau, \zeta)) \\
& \leq \frac{H_\tau \, \gamma^m}{(1 - \gamma)} \, \sum_{i = 1}^{t}  (\gamma^m)^{t - i} \, \max_{s} \left[ TV_i + \tau (1 - \gamma) \,\vert\cH(\pi_i(\cdot|s)) - \cH(\pi_{i-1}(\cdot|s))\vert \right]\\
&\quad + \sum_{i = 0}^{t} (\gamma^m)^{t - i} \, \delta(\tau, \zeta) 
\end{align*}
Furthermore, if $TV_i \leq \frac{1}{2}$ for all $i \in [t]$, using~\cref{lemma:entropy-difference}, we can further upper-bound $\vert\cH(\pi_i(\cdot|s)) - \cH(\pi_{i-1}(\cdot|s))\vert$ to get that for any constant $C \in (0,1/2)$, 
\begin{align*}
\vert \cH(\pi_i(\cdot|s)) - \cH(\pi_{i-1}(\cdot|s)) \vert & \leq TV_i \,   \ln\left(\frac{A}{C} \right) +  \left(\frac{\ln(A-1)}{2} + \sqrt{2} \right) \, \sqrt{C}      
\end{align*}
Combining the above relations, in this case, we get that, 
\begin{align*}
\epsilon_{t} &\leq \frac{H_\tau \, \gamma^m}{(1 - \gamma)} \, \sum_{i = 1}^{t}  (\gamma^m)^{t - i} \, \max_{s} \left[ TV_i + \tau (1 - \gamma) \, \left[ TV_i \,   \ln\left(\frac{A}{C} \right) + \left(\frac{\ln(A)}{2} + \sqrt{2} \right) \, \sqrt{C}  \right] \right] \\ 
&\quad + \sum_{i = 0}^{t} (\gamma^m)^{t - i} \, \delta(\tau, \zeta) 
\end{align*}
\end{proof}

\begin{corollary}
Using the policy evaluation scheme in~\cref{eq:pe-general} with $\zeta = \tau$, for all $t \in [K]$, if for all $i \in [t]$, $TV_i := \norm{\pi_i(\cdot|s) - \pi_{i - 1}(\cdot|s)}_{1} \leq \frac{1}{2}$, then, for any constant $C \in (0,1/2)$, 
\begin{align*}
\epsilon_t & \leq \frac{H_\tau \, \gamma^m}{(1 - \gamma)} \, \sum_{i = 1}^{t}  (\gamma^m)^{t - i} \, \max_{s} \left[ TV_i + \tau (1 - \gamma) \, \left[ TV_i \,   \ln\left(\frac{A}{C} \right) + \left(\frac{\ln(A)}{2} + \sqrt{2} \right) \, \sqrt{C}  \right] \right] 
\end{align*}
\label{cor:error-analysis-entropy}    
\end{corollary}
\begin{proof}
Setting $\zeta = \tau$ in~\cref{thm:error-analysis-general}.     
\end{proof}

\begin{corollary}
Using the policy evaluation scheme in~\cref{eq:pe-general} with $\zeta = 0$, for all $t \in [K]$, if for all $i \in [t]$, $TV_i := \norm{\pi_i(\cdot|s) - \pi_{i - 1}(\cdot|s)}_{1} \leq \frac{1}{2}$, then, for any constant $C \in (0,1/2)$, 
\begin{align*}
\epsilon_t & \leq \frac{H_\tau \, \gamma^m}{(1 - \gamma)} \, \sum_{i = 1}^{t}  (\gamma^m)^{t - i} \, \max_{s} \left[ TV_i + \tau (1 - \gamma) \, \left[ TV_i \,   \ln\left(\frac{A}{C} \right) + \left(\frac{\ln(A)}{2} + \sqrt{2} \right) \, \sqrt{C} \right] \right]\\
&\quad + \frac{\tau \, \ln(A)}{(1 - \gamma)^2}
\end{align*}
\label{cor:error-analysis-no-entropy}    
\end{corollary}
\begin{proof}
Setting $\zeta = 0$ in~\cref{thm:error-analysis-general}.     
\end{proof}

The following proposition shows that the objective in~\cref{eq:policy-update-rkl} admits a closed-form solution. Substituting the intermediate policies $\NPG$ and $\SPMA$ into the closed-form expression in~\cref{app:eq-soft-update} yields their entropy-regularized (actor entropy) counterparts, referred to as soft $\NPG$ and soft $\SPMA$.
\begin{proposition}
If $\alphat := \frac{1}{1 + \taut}$, the closed-form solution for the proximal update in~\cref{eq:policy-update-rkl} for any $s,a$ is given as, 
\begin{align}
\pitt(a|s) &= \frac{[\pith(a|s)]^{\alphat}}{\sum_{a'} [\pith(a'|s)]^\alphat}
\label{app:eq-soft-update}    
\end{align}
\label{app:prop-soft-update}
\end{proposition}
\begin{proof}
    \begin{align*}
    \pitt(\cdot|s) &= \argmin_{\pi(\cdot|s) \in \Delta} \left[\text{KL}(\pi(\cdot|s) || \pith(\cdot|s)) - \taut \cH(\pi(\cdot|s)) \right]\\
    &= \argmin_{\pi(\cdot|s) \in \Delta} \E_{a \sim \pi(.|s)}\left[(1 + \tau_t) \, \ln(\pi(.|s))- \ln(\pith(\cdot|s))\right] \tag{Using the definition of $H(\pi(\cdot|s))$} \\
    &= \argmin_{\pi(\cdot|s) \in \Delta} \E_{a \sim \pi(.|s)}\left[\ln(\pi(.|s))- \alphat \, \ln(\pith(\cdot|s))\right] \tag{Since $\alphat = \frac{1}{1+\tau_t}$} \\
    &= \argmin_{\pi(\cdot|s) \in \Delta} \E_{a \sim \pi(.|s)}\left[\ln(\pi(.|s)) - \ln([\pith(\cdot|s)]^{\alphat})\right]\\
    &= \argmin_{\pi(\cdot|s) \in \Delta} \text{KL}(\pi(.|s) || \pith(\cdot|s)^{\alphat}) \tag{By definition of the KL divergence}
\end{align*}
Using the fact that KL projection onto the simplex results in normalization, we get that, 
\begin{align*}
    \pitt(a|s) &= \frac{[\pith(a|s)]^{\alphat}}{\sum_{a'} [\pith(a'|s)]^\alphat}
\end{align*}
\end{proof}

Note that when $\tau = 0$, we have $\taut = 0$ and $\alphat = 1$ for all $t$, recovering the standard unregularized updates for both $\NPG$ and $\SPMA$. For $\tau > 0$, the soft updates can be expressed as: for any $s, a$, with $\alphat = \frac{1}{1 + \taut}$ and $\taut = \etat \, \tau$, 
\begin{align}
\pitt(a|s) 
    &= \frac{[\pit(a|s)]^\alphat \, \exp(\etat \, \alphat \, q_\zeta^t(s,a))}{\cZ_t}, \\
\text{with} \quad 
\cZ_t &= \sum_{a'} [\pit(a'|s)]^\alphat \, \exp(\etat \, \alphat \, q_\zeta^t(s,a')), 
       &&\textbf{(Soft-NPG)} \label{eq:soft-NPG} \notag\\
\pitt(a|s) 
    &= \frac{[\pit(a|s)]^\alphat \, \left[1 + \eta_t \, \left(q_\zeta^{t}(s,a) - v_\zeta^t(s) \right) \right]^\alphat}{\cZ_t},\\ 
\text{with} \quad 
\cZ_t &= \sum_{a'} [\pit(a'|s)]^\alphat \, \left[1 + \eta_t \, \left(q_\zeta^{t}(s,a') - v_\zeta^t(s) \right) \right]^\alphat, 
       &&\textbf{(Soft-SPMA)} \label{eq:soft-SPMA} \notag
\end{align}

\subsubsection{\texorpdfstring{Policy Error Bound for (soft) $\NPG$}{Policy Error Bound for (soft) NPG}}
\label{app:pe-soft-npg-cor-proofs}

In the next corollary, we use~\cref{thm:error-analysis-general} with $\zeta = \tau$ to instantiate the policy error bound for soft $\NPG$ with entropy-regularized policy evaluation. 

\begin{corollary}[Policy evaluation with $\zeta = \tau$]
Using the policy evaluation update in~\cref{eq:pe-general} with $\zeta = \tau$, for soft $\NPG$ with $\etat = \frac{1}{c + \tau \, (t+1)}$ and a constant $c \geq \max \left\{\frac{8 \, (1 + \tau \, \ln(A))}{(1-\gamma)}, 32\,\tau\,\ln(A) \right\}$, for all $t \in [K]$, $\epsilon_t := \norminf{\bepsilon_{t}}$ can be bounded as: 
\begin{align*}
\epsilon_t & \leq \frac{8 (1+\tau \ln(A)) \gamma^m}{(1 - \gamma)^3} \, \Bigg[ \left(1 + \tau \ln\left(A \, K^4 \right) \right) \\ &\hspace{3.5cm}\left(\left(\ln(A \, K^4) + \frac{1+\tau \ln(A)}{\tau(1-\gamma)}\right)  \, \left(\frac{1}{t}+(\gamma^m)^{t/2}\right)
+ \frac{\sqrt{\tau}}{K}\right)\\
&\hspace{3.5cm} + \tau \, \left(\ln(A) + 1 \right) \, \frac{1}{K^2} \Bigg]
\end{align*}
\label{cor:npg-pe-entropy}    
\end{corollary}

\begin{proof}
For a fixed iteration $t \in [K]$ and state $s \in \cS$, let us first bound $\norm{\pi_{t+1}(\cdot|s) - \pi_{t}(\cdot|s)}_{1}$. 
\begin{align*}
\norm{\pi_{t+1}(\cdot|s) - \pi_{t}(\cdot|s)}_1 &\leq \norm{\pi_{t+1}(\cdot|s) - \pi_{t+1/2}(\cdot|s)}_1 + \norm{\pi_{t+1/2}(\cdot|s) - \pi_{t}(\cdot|s)}_1 \tag{Triangle inequality} \\
\end{align*}
We first bound $\norm{\pi_{t+1/2}(\cdot|s) - \pi_{t}(\cdot|s)}_1$. Using the mirror descent view of $\NPG$~\citep{xiao2022convergence}, the update can be written as:
\begin{align*}
& \pi_{t+1/2}(\cdot|s) = \argmin_{\pi \in \Delta} \left[-\etat \langle q_{\tau}^t(s, \cdot), \pi(\cdot|s) \rangle + \text{KL}(\pi(\cdot|s) || \pi_{t}(\cdot|s)) \right] \\
\implies & -\etat \langle q_{\tau}^t(s, \cdot), \pi_{t+1/2}(\cdot|s) \rangle + \text{KL}(\pi_{t+1/2}(\cdot|s) || \pi_{t}(\cdot|s)) \leq -\etat \langle q_{\tau}^t(s, \cdot), \pi_{t}(\cdot|s) \rangle\\
&\hspace{8.1cm} + \text{KL}(\pi_{t}(\cdot|s) || \pi_{t}(\cdot|s)) \\
\implies & \frac{1}{2} \norm{\pi_{t+1/2}(\cdot|s) - \pi_{t}(\cdot|s)}_1^2 \leq \text{KL}(\pi_{t+1/2}(\cdot|s) || \pi_{t}(\cdot|s)) \leq \etat \langle q_{\tau}^t(s, \cdot), \pi_{t+1/2}(\cdot|s) - \pi_{t}(\cdot|s) \rangle \tag{By Pinsker's inequality} \\
& \hspace{25ex} \leq \etat \norminf{q_{\tau}^t(s, \cdot)} \, \norm{\pi_{t+1/2}(\cdot|s) - \pi_{t}(\cdot|s)}_{1} \tag{By Holder's inequality} \\
\implies & \norm{\pi_{t+1/2}(\cdot|s) - \pi_{t}(\cdot|s)}_1 \leq 2 \etat \, H_\tau \tag{Since $\norminf{q_{\tau}^t(s, \cdot)} \leq H_\tau$}
\end{align*}
In order to bound $\norm{\pi_{t+1}(\cdot|s) - \pi_{t+1/2}(\cdot|s)}_1$, we use the~\cref{eq:policy-update-rkl} update. Specifically, 
\begin{align*}
\pi_{t+1}(\cdot|s) = \argmin_{\pi \in \Delta} \left[\text{KL}(\pi(\cdot|s) || \pi_{t+1/2}(\cdot|s)) - \taut \cH(\pi(\cdot|s)) \right]
\end{align*}
\vspace{-1em}
\begin{align*}
\implies \text{KL}(\pi_{t+1}(\cdot|s) || \pi_{t+1/2}(\cdot|s)) - \taut \cH(\pi_{t+1}(\cdot|s)) & \leq \text{KL}(\pi_{t+1/2}(\cdot|s) || \pi_{t+1/2}(\cdot|s))\\
&\quad- \tau_t \cH(\pi_{t+1/2}(\cdot|s)) \\
\frac{1}{2} \, \norm{\pi_{t+1}(\cdot|s) - \pi_{t+1/2}(\cdot|s)}_1^2 &\leq \text{KL}(\pi_{t+1} || \pi_{t+1/2}(\cdot|s)) \tag{By Pinsker's inequality}\\ & \leq \tau_t \cH(\pi_{t+1}(\cdot|s)) - \tau_t \cH(\pi_{t+1/2}(\cdot|s)) \\
& \le \taut \, \ln(A)\\
\implies \norm{\pi_{t+1}(\cdot|s) - \pi_{t+1/2}(\cdot|s)}_1 &\le \sqrt{2 \, \tau \, \etat \, \ln(A)} \tag{since $\taut = \tau \, \etat$}
\end{align*}
Combining the above relations, 
\begin{align}
\norm{\pi_{t+1}(\cdot|s) - \pi_{t}(\cdot|s)}_1 & \leq \sqrt{2 \, \tau \, \etat \, \ln(A)} + 2 \, \etat \, H_\tau  
\label{eq:npg-prob-diff}
\end{align}
Using~\cref{eq:npg-prob-diff} in~\cref{thm:error-analysis-general}, for the special case $\norm{\pi_{t+1}(\cdot|s) - \pi_{t}(\cdot|s)}_1 \le \frac{1}{2}$ for all $t \in [K]$, $\etat$ must satisfy the following conditions:
\begin{itemize}
    \item $2 \, \etat \, H_\tau  \leq \frac{1}{4} \implies \etat \leq \frac{(1-\gamma)}{8 \, (1 + \tau \, \ln(A))}$
    \item $\sqrt{2 \, \tau \, \etat \, \ln(A)} \leq \frac{1}{4} \implies \etat \leq \frac{1}{32 \, \tau \, \ln(A)}$ 
\end{itemize}
For $\etat = \frac{1}{c + \tau (t+1)} \leq \frac{1}{c}$, it is thus sufficient to ensure that, 
\begin{align*}
c \geq \max \left\{\frac{8 \, (1 + \tau \, \ln(A))}{(1-\gamma)}, 32 \, \tau \, \ln(A) \right\}    
\end{align*}
Given this constraint on $c$, we use~\cref{thm:error-analysis-general} with $\zeta = \tau$ and $TV_i := \norm{\pi_i(\cdot|s) - \pi_{i - 1}(\cdot|s)}_{1}$, to get that, for any constant $C \in (0,1/2)$, for all $t \in [K]$,
\begin{align*}
\epsilon_t & \leq \frac{H_\tau \, \gamma^m}{(1 - \gamma)} \, \sum_{i = 1}^{t}  (\gamma^m)^{t - i} \, \max_{s} \left[ TV_i + \tau (1 - \gamma) \, \left[ TV_i \,   \ln\left(\frac{A}{C} \right) + \left(\frac{\ln(A)}{2} + \sqrt{2} \right) \, \sqrt{C}  \right] \right]
\end{align*}
Observe that the choice of $c$ ensures $\norm{\pi_{t+1}(\cdot|s) - \pi_{t}(\cdot|s)}_1 \le \frac{1}{2}$. Using~\cref{lemma:prob-difference} we obtain,
\begin{align*}
& \norm{\pi_{t+1}(\cdot|s) - \pi_{t}(\cdot|s)}_1 \leq 4 \, \tau \, \etat \, \ln(\frac{A}{C}) + 2 \, \sqrt{\tau} \, C^{\frac{1}{4}} + 2 \, \etat \, H_\tau \\
\intertext{Combining the above with the upper bound on $\epsilon_t$, we get,}
\epsilon_t & \leq \frac{ H_\tau \, \gamma^m}{(1 - \gamma)} \, \sum_{i = 1}^t (\gamma^m)^{t-i} \,\Bigg[ \left(1 + \tau \, (1-\gamma) \, \ln(\frac{A}{C}) \right) \, \left(2 \, \eta_i \, (2\, \tau \, \ln(\frac{A}{C}) + H_{\tau}) + 2 \, \sqrt{\tau} \, C^{\frac{1}{4}} \right)\\ 
&\hspace{3.9cm}+ \tau (1 - \gamma) \left(\frac{\ln(A)}{2} + \sqrt{2} \right) \, \sqrt{C} \Bigg] \\
\intertext{Setting $C = \frac{1}{K^4}$,}
& \leq \frac{H_\tau \, \gamma^m}{(1 - \gamma)} \, \Bigg[2\, \left(1 + \tau \,  \ln\left(A \, K^4 \right) \right) \, \left(\left(2 \ln(A \, K^4) + \frac{H_{\tau}}{\tau}\right)  \, \sum_{i = 1}^t  \frac{(\gamma^m)^{t-i}}{(i+1)} + \frac{\sqrt{\tau}}{K(1-\gamma^m)}\right) \\ 
&\hspace{1.8cm}+ \frac{\tau}{1-\gamma^m} \, \Big(\frac{\ln(A)}{2} + \sqrt{2} \Big) \, \frac{1}{K^2} \Bigg] \tag{since $\gamma < 1$}\\
\intertext{Using~\cref{lemma:sequence-sum}, we can bound $\sum_{i = 1}^t  \frac{(\gamma^m)^{t-i}}{(i+1)} \leq \frac{2}{1 - \gamma^m}  \, (\frac{1}{t}+(\gamma^m)^{t/2})\leq \frac{2}{1 - \gamma}  \, (\frac{1}{t}+(\gamma^m)^{t/2})$.}
\implies \epsilon_t & \leq \frac{H_\tau \, \gamma^m}{(1 - \gamma)} \, \Bigg[4 \left(1 + \tau \, \ln\left(A \, K^4 \right) \right)\\
&\hspace{1.8cm}
\left(\left(2 \ln(A\,K^4) + \frac{H_{\tau}}{\tau}\right) \, \left(\frac{1}{t}+(\gamma^m)^{t/2}\right) + \frac{\sqrt{\tau}}{K(1-\gamma^m)}\right) \\
&\hspace{1.8cm} + \frac{\tau}{1-\gamma^m} \, \Big(\frac{\ln(A)}{2} + \sqrt{2} \Big) \, \frac{1}{K^2} \Bigg]\\
& \leq \frac{8 H_\tau \, \gamma^m}{(1 - \gamma)^2} \, \Bigg[\Big(1 + \tau \,  \ln\left(A \, K^4 \right) \Big) \, \Big(\big(\ln(A \, K^4) + \frac{H_{\tau}}{\tau}\big) \, (\frac{1}{t}+(\gamma^m)^{t/2}) + \frac{\sqrt{\tau}}{K}\Big)\\
&\hspace{2.1cm}
+ \tau \, \left(\ln(A) + 1 \right) \, \frac{1}{K^2} \Bigg]\tag{Since $\gamma < 1$}\\
\end{align*}
\end{proof}

In the next corollary, we use~\cref{thm:error-analysis-general} with $\zeta = 0$ to instantiate the policy error bound for (soft) NPG with entropy-regularized policy evaluation.
\begin{corollary}[Policy evaluation with $\zeta = 0$]
Using the policy evaluation update in~\cref{eq:pe-general} with $\zeta = 0$, $\epsilon_t := \norminf{\bepsilon_{t}}$ can be bounded as: 
\begin{itemize}
    \item \textbf{Soft $\NPG$}: If $\etat = \frac{1}{c + \tau \, (t+1)}$ for a constant $c \geq \max \left\{\frac{8 \, (1 + \tau \, \ln(A))}{(1-\gamma)}, 32 \, \tau \, \ln(A)\right\}$, then, for all $t \in [K]$,
    \begin{align*}
    \epsilon_t &\le \frac{8 (1+\tau \ln(A)) \gamma^m}{(1 - \gamma)^3} \, \bigg[ \Big(1 + \tau \ln\left(A \, K^4 \right) \Big)\\ &\hspace{3.4cm}\left(\left(\ln(A \, K^4) + \frac{1+\tau \ln(A)}{\tau(1-\gamma)}\right)  \, \left(\frac{1}{t}+(\gamma^m)^{t/2}\right)
    + \frac{\sqrt{\tau}}{K}\right)\\
    &\hspace{3.4cm} + \tau \, \left(\ln(A) + 1 \right) \, \frac{1}{K^2} \bigg] + \frac{\tau}{(1-\gamma)^2}\ln{(A)} 
    \end{align*} 
    \item \textbf{$\NPG$}: If $\etat = \eta = \frac{ (1-\gamma) \, \sqrt{\ln(A)}}{\sqrt{K}}$, then, for all $t \in [K]$, 
\begin{align*}
    \epsilon_t & \leq \frac{2 \, \sqrt{\ln(A)} \, \gamma^m}{(1 - \gamma)^3} \, \frac{1}{\sqrt{K}} 
\end{align*}
\end{itemize}
\label{cor:npg-pe-no-entropy}
\end{corollary}
\begin{proof}
Using~\cref{thm:error-analysis-general} for soft $\NPG$ and~\cref{cor:error-analysis-no-entropy} for $\NPG$, and following the same proof as~\cref{cor:npg-pe-entropy}.    
\end{proof}

\subsubsection{\texorpdfstring{Policy Error Bound for (soft) $\SPMA$}{Policy Error Bound for (soft) SPMA}}
\label{app:pe-soft-spma-cor-proofs}

In the next corollary, we use~\cref{thm:error-analysis-general} with $\zeta = \tau$ to instantiate the policy error bound for soft SPMA with entropy-regularized policy evaluation. 
\begin{corollary}[Policy evaluation with $\zeta = \tau$]
Using the policy evaluation update in~\cref{eq:pe-general} with $\zeta = \tau$, for soft $\SPMA$ with $\etat = \frac{1}{c + \tau \, (t+1)}$ and a constant $c \geq \max \left\{\frac{4(1 + \tau \, \ln(A))}{\, (1-\gamma)}, 32 \, \tau \, \ln(A)\right\}$, for all $t \in [K]$, $\epsilon_t$ can be bounded as: 
\begin{align*}
\epsilon_t := \norminf{\bepsilon_{t}} & \leq \frac{8 (1+\tau \ln(A)) \gamma^m}{(1 - \gamma)^3} \, \Bigg[ \left(1 + \tau \ln\left(A \, K^4 \right) \right) \\ &\hspace{3.5cm}\left(\left(\ln(A \, K^4) + \frac{1+\tau \ln(A)}{\tau(1-\gamma)}\right)  \, \left(\frac{1}{t}+(\gamma^m)^{t/2}\right)
+ \frac{\sqrt{\tau}}{K}\right)\\
&\hspace{3.5cm} + \tau \, \left(\ln(A) + 1 \right) \, \frac{1}{K^2} \Bigg]
\end{align*}
\label{cor:spma-pe-entropy}    
\end{corollary}
\begin{proof}
For a fixed iteration $t \in [K]$ and state $s \in \cS$, let us first bound $\norm{\pi_{t+1}(\cdot|s) - \pi_{t}(\cdot|s)}_{1}$. 
\begin{align*}
\norm{\pi_{t+1}(\cdot|s) - \pi_{t}(\cdot|s)}_1 &\leq \norm{\pi_{t+1}(\cdot|s) - \pi_{t+1/2}(\cdot|s)}_1 + \norm{\pi_{t+1/2}(\cdot|s) - \pi_{t}(\cdot|s)}_1 \tag{Triangle inequality} \\
& \leq \norm{\pi_{t+1}(\cdot|s) - \pi_{t+1/2}(\cdot|s)}_1 + \etat \sum_{a} \pi_{t}(a) \, [q^t_\tau(s,a) - v^t_\tau(s)] \tag{By the $\SPMA$ update in~\cref{eq:spma}} \\
\implies \norm{\pi_{t+1}(\cdot|s) - \pi_{t}(\cdot|s)}_1 & \leq \norm{\pi_{t+1}(\cdot|s) - \pi_{t+1/2}(\cdot|s)}_1 + \etat H_\tau \tag{Since $\vert q^t_\tau(s,a) - v^t_\tau(s) \vert \leq H_\tau$}
\end{align*}
In order to bound $\norm{\pi_{t+1}(\cdot|s) - \pi_{t+1/2}(\cdot|s)}_1$, we use the~\cref{eq:policy-update-rkl} update. Specifically, 
\begin{align*}
\pi_{t+1}(\cdot|s) = \argmin_{\pi \in \Delta} \left[\text{KL}(\pi || \pi_{t+1/2}(\cdot|s)) - \tau_t \cH(\pi(\cdot|s)) \right]
\end{align*}
\vspace{-1em}
\begin{align*}
\implies \text{KL}(\pi_{t+1}(\cdot|s) || \pi_{t+1/2}(\cdot|s)) - \tau_t \cH(\pi_{t+1}(\cdot|s)) & \leq \text{KL}(\pi_{t+1/2}(\cdot|s) || \pi_{t+1/2}(\cdot|s))\\ 
&\quad - \tau_t \cH(\pi_{t+1/2}(\cdot|s)) \\
\frac{1}{2} \, \norm{\pi_{t+1}(\cdot|s) - \pi_{t+1/2}(\cdot|s)}_1^2 &\leq \text{KL}(\pi_{t+1}(\cdot|s) || \pi_{t+1/2}(\cdot|s)) \tag{By Pinsker's inequality} \\ & \leq \tau_t \cH(\pi_{t+1}(\cdot|s)) - \tau_t \cH(\pi_{t+1/2}(\cdot|s)) \\
& \leq \tau_t \ln(A) \tag{Since $\cH(\pi) \in [0, \ln(A)]$} \\
\implies \norm{\pi_{t+1}(\cdot|s) - \pi_{t+1/2}(\cdot|s)}_1 & \leq \sqrt{2 \, \tau \, \etat \, \ln(A)} \tag{Since $\tau_t = \tau \, \etat$}
\end{align*}
Combining the above relations, 
\begin{align}
\norm{\pi_{t+1}(\cdot|s) - \pi_{t}(\cdot|s)}_1 & \leq \sqrt{2 \, \tau \, \etat \, \ln(A)} + \etat \, H_\tau    
\label{eq:spma-prob-diff}
\end{align}
In order to use~\cref{thm:error-analysis-general}, we need to ensure that $\norm{\pi_{t+1}(\cdot|s) - \pi_{t}(\cdot|s)}_1 \leq \frac{1}{2}$ for all $t \in [K]$. Using~\cref{eq:spma-prob-diff}, it is sufficient to ensure that $\etat$ satisfies the following relations:

\begin{itemize}
    \item $ \etat \, H_\tau  \leq \frac{1}{4} \implies \etat \leq \frac{(1-\gamma)}{4 \, (1 + \tau \, \ln(A))}$
    \item $\sqrt{2 \, \tau \, \etat \, \ln(A)} \leq \frac{1}{4} \implies \etat \leq \frac{1}{32 \, \tau \, \ln(A)}$
\end{itemize}

For $\etat = \frac{1}{c + \tau (t+1)} \leq \frac{1}{c}$, it is thus sufficient to ensure that, 
\begin{align*}
c \geq \max \left\{\frac{4 \, (1 + \tau \, \ln(A))}{(1-\gamma)}, 32 \, \tau \, \ln(A) \right\}    
\end{align*}
Given this constraint on $c$, we use~\cref{thm:error-analysis-general} with $\zeta = \tau$ and $TV_i := \norm{\pi_i(\cdot|s) - \pi_{i - 1}(\cdot|s)}_{1}$, to get that, for any constant $C \in (0,1/2)$, for all $t \in [K]$,
\begin{align*}
\epsilon_t & \leq \frac{H_\tau \, \gamma^m}{(1 - \gamma)} \, \sum_{i = 1}^t (\gamma^m)^{t-i} \, \max_{s} \left[ TV_i + \tau (1 - \gamma) \, \left( TV_i \,   \ln\left(\frac{A}{C} \right) +  \left(\frac{\ln(A)}{2} + \sqrt{2} \right) \, \sqrt{C} \right) \right]  
\end{align*}
Observe that the choice of $c$ ensures $\norm{\pi_{t+1}(\cdot|s) - \pi_{t}(\cdot|s)}_1 \le \frac{1}{2}$. Using~\cref{lemma:prob-difference} we obtain,
\begin{align*}
& \norm{\pi_{t+1}(\cdot|s) - \pi_{t}(\cdot|s)}_1 \leq 4 \, \tau \, \etat \, \ln(\frac{A}{C}) + 2 \, \sqrt{\tau} \, C^{\frac{1}{4}} + 2 \, \etat \, H_\tau \\
\intertext{Combining the above with the upper bound on $\epsilon_t$, we get,}
\epsilon_t & \leq \frac{ H_\tau \, \gamma^m}{(1 - \gamma)} \, \sum_{i = 1}^t (\gamma^m)^{t-i} \,\Bigg[ \left(1 + \tau \, (1-\gamma) \, \ln(\frac{A}{C}) \right) \, \left(2 \, \eta_i \, (2\, \tau \, \ln(\frac{A}{C}) + H_{\tau}) + 2 \, \sqrt{\tau} \, C^{\frac{1}{4}} \right)\\ 
&\hspace{3.9cm}+ \tau (1 - \gamma) \left(\frac{\ln(A)}{2} + \sqrt{2} \right) \, \sqrt{C} \Bigg] \\
\intertext{Setting $C = \frac{1}{K^4}$,}
& \leq \frac{H_\tau \, \gamma^m}{(1 - \gamma)} \, \Bigg[2\, \left(1 + \tau \,  \ln\left(A \, K^4 \right) \right) \, \left(\left(2 \ln(A \, K^4) + \frac{H_{\tau}}{\tau}\right)  \, \sum_{i = 1}^t  \frac{(\gamma^m)^{t-i}}{(i+1)} + \frac{\sqrt{\tau}}{K(1-\gamma^m)}\right) \\ 
&\hspace{1.8cm}+ \frac{\tau}{1-\gamma^m} \, \Big(\frac{\ln(A)}{2} + \sqrt{2} \Big) \, \frac{1}{K^2} \Bigg] \tag{since $\gamma < 1$}\\
\intertext{Using~\cref{lemma:sequence-sum}, we can bound $\sum_{i = 1}^t  \frac{(\gamma^m)^{t-i}}{(i+1)} \leq \frac{2}{1 - \gamma^m}  \, (\frac{1}{t}+(\gamma^m)^{t/2})\leq \frac{2}{1 - \gamma}  \, (\frac{1}{t}+(\gamma^m)^{t/2})$.}
\implies \epsilon_t & \leq \frac{H_\tau \, \gamma^m}{(1 - \gamma)} \, \Bigg[4 \left(1 + \tau \, \ln\left(A \, K^4 \right) \right)\\
&\hspace{1.8cm}
\left(\left(2 \ln(A\,K^4) + \frac{H_{\tau}}{\tau}\right) \, \left(\frac{1}{t}+(\gamma^m)^{t/2}\right) + \frac{\sqrt{\tau}}{K(1-\gamma^m)}\right) \\
&\hspace{1.8cm} + \frac{\tau}{1-\gamma^m} \, \Big(\frac{\ln(A)}{2} + \sqrt{2} \Big) \, \frac{1}{K^2} \Bigg]\\
& \leq \frac{8 H_\tau \, \gamma^m}{(1 - \gamma)^2} \, \Bigg[\Big(1 + \tau \,  \ln\left(A \, K^4 \right) \Big) \, \Big(\big(\ln(A \, K^4) + \frac{H_{\tau}}{\tau}\big) \, (\frac{1}{t}+(\gamma^m)^{t/2}) + \frac{\sqrt{\tau}}{K}\Big)\\
&\hspace{2.1cm}
+ \tau \, \left(\ln(A) + 1 \right) \, \frac{1}{K^2} \Bigg]\tag{Since $\gamma < 1$}\\
\end{align*}
\end{proof}

In the next corollary, we use~\cref{thm:error-analysis-general} with $\zeta = 0$ to instantiate the policy error bound for (soft) $\SPMA$ without entropy-regularized policy evaluation. 
\begin{corollary}[Policy evaluation with $\zeta = 0$]
Using the policy evaluation update in~\cref{eq:pe-general} with (soft) $\SPMA$, $\epsilon_t := \norminf{\bepsilon_{t}}$ can be bounded as: 
\begin{itemize}

    \item \textbf{Soft $\SPMA$}: if $\etat = \frac{1}{c + \tau \, (t+1)}$ for a constant $c \geq \max \left\{\frac{4(1 + \tau \, \ln(A))}{\, (1-\gamma)}, 32 \, \tau \, \ln(A)\right\}$, then, for all $t \in [K]$,
    \begin{align*}
    \epsilon_t &\le \frac{8 (1+\tau \ln(A)) \gamma^m}{(1 - \gamma)^3} \, \bigg[ \Big(1 + \tau \ln\left(A \, K^4 \right) \Big)\\ &\hspace{3.4cm}\left(\left(\ln(A \, K^4) + \frac{1+\tau \ln(A)}{\tau(1-\gamma)}\right)  \, \left(\frac{1}{t}+(\gamma^m)^{t/2}\right)
    + \frac{\sqrt{\tau}}{K}\right)\\
    &\hspace{3.4cm} + \tau \, \left(\ln(A) + 1 \right) \, \frac{1}{K^2} \bigg] + \frac{\tau}{(1-\gamma)^2}\ln{(A)} 
    \end{align*}    
    \item \textbf{$\SPMA$}: if $\etat = \eta = \min\left\{\frac{1-\gamma}{2}, \frac{\, (1-\gamma) \, \sqrt{\ln(A)}}{\sqrt{K}} \right\}$, then, for all $t \in [K]$, 
    \begin{align*}
    \epsilon_t := \norminf{\bepsilon_{t}} & \leq  \frac{ \sqrt{\ln(A)} \, \gamma^m}{(1 - \gamma)^3} \, \frac{1}{\sqrt{K}}
    \end{align*}
\end{itemize}
\label{cor:spma-pe-no-entropy}
\end{corollary}
\begin{proof}
Using~\cref{thm:error-analysis-general} for soft $\SPMA$ and~\cref{cor:error-analysis-no-entropy} for $\SPMA$, and following the same proof as~\cref{cor:spma-pe-entropy}. Note that the requirement for $\eta \le \frac{1-\gamma}{2}$ is to ensure $\ln(1 + \etat \, \Delta^t(s,\cdot))$ in the $\SPMA$ update is well-defined (see the proof in~\cref{cor:spma-regret-guarantee} for details)      
\end{proof}

\subsection{Proof of~\texorpdfstring{\cref{cor:npg-spma-regret-guarantee}}{}}
\label{app:cor-po-proof}

\subsubsection{Generic Regret Bound}
\label{app:generic-po-proof}

\begin{theorem}[Generic Regret Bound]
Consider a sequence of linear functions $f_t(\pi) := \langle \pi, d_t \rangle$ for a sequence of vectors $\{d_0, d_1, \ldots, d_{K-1} \}$ s.t. $\norminf{d_t} \leq D_t$. Consider the following update at iteration $t \in [K]$, if $\etat$ is a step-size sequence, $\taut = \etat \, \tau$, $\pi_0$ is the uniform distribution and  
\begin{align*}
    \pitt &= \argmin_{\pi \in \Delta_A} \Big\{
        \langle \pi, d_t \rangle + \text{KL}(\pi||\pit) + \cR_{t}(\pi)
    \Big\} \\
    \cR_{t}(\pi) &:= \taut \, \cR(\pi) \\
    \cR(\pi) &:= \ln(A) - \cH(\pi) \geq 0
\end{align*}
then, for any comparator $u \in \Delta_A$, 
\begin{align*}
\sum_{t = 0}^{K-1} \left[\frac{\inner{\pit - u}{d_t}}{\etat} + \tau \, [\cH(u) - \cH(\pit)] \right] & \leq \sum_{t = 0}^{K-1} \left[\frac{\text{KL} (u||\pit)}{\eta_t} - \frac{\text{KL}(u||\pitt)}{\eta_t} -  \tau \text{KL}(u||\pitt) \right]\\
&\quad + \sum_{t = 0}^{K-1} \frac{D_t^2 }{2 \, \etat} 
\end{align*}
\label{thm:generic-regret-bound}    
\end{theorem}
\begin{proof}
The following properties will be helpful in proving the theorem. For policies $\pi, \pi'$ and comparator $u$,
\begin{align*}
    & \cR_t(\pi) - \cR_t(\pi') = \taut  \, \langle \ln(\pi), \pi-\pi'\rangle - \taut \, \text{KL}(\pi'||\pi) \tag{Entropy property}\\
    & \langle u-\pi', \ln(\pi') - \ln(\pi) \rangle = \text{KL}(u||\pi) - \text{KL}(u||\pi') - \text{KL} (\pi'||\pi) \tag{3 point property}\\    
    & \langle \pi - \pitt, d_t+ \ln(\pitt) - \ln(\pit) + \tau_t \ln(\pitt) \rangle \geq 0 \tag{Optimality condition}
\end{align*}
\begin{align*}
    [f_t(\pit) - f_t(u)] + \cR_t(\pitt) - \cR_t(u) & = \inner{\pit - u}{d_t} + \inner{\tau_t \ln(\pitt)}{\pitt-u}\\ 
    &\quad - \tau_t \text{KL}(u||\pitt) \tag{Entropy property with $\pi = \pitt$, $\pi' = u$}\\
    & = \inner{\pitt - u}{d_t} + \inner{\pit - \pitt}{d_t} \\
    &\quad +\inner{\tau_t \ln(\pitt)}{\pitt-u} - \tau_t \text{KL}(u||\pitt)\\
    & = \underbrace{\inner{\pitt-u}{d_t+\tau_t \ln(\pitt)-\ln(\pit) + \ln(\pitt)}}_{\leq 0 \text{ by the optimality condition for $\pi = u$}} \\
    &\, \, \, \, + \inner{\pitt-u}{\ln(\pit) -\ln(\pitt) } + \inner{\pit - \pitt}{d_t}\\ 
    &\quad - \tau_t \text{KL}(u||\pitt) \tag{Dropping the negative term} \\
    & \leq  \inner{\pitt-u}{\ln(\pit) -\ln(\pitt) } + \inner{\pit - \pitt}{d_t}\\  
    &\quad - \tau_t \text{KL}(u||\pitt) \\
    & = \text{KL} (u||\pit) - \text{KL}(u||\pitt) - \text{KL} ( \pitt||\pit)\\
    &\quad + \inner{\pit - \pitt}{d_t}  - \tau_t \text{KL}(u||\pitt) \tag {3 point property with $u = u, \pi = \pit, \pi' = \pitt$}\\
    & \leq \text{KL} (u||\pit) - \text{KL}(u||\pitt) - \text{KL} ( \pitt||\pit)\\
    &\quad - \tau_t \text{KL}(u||\pitt) + \frac{1}{2}||\pit-\pitt||_1^2 + \frac{1}{2}||d_t||_\infty^2 \tag{Fenchel-Young inequality}\\
    & \leq \text{KL} (u||\pit) - \text{KL}(u||\pitt) - \text{KL} ( \pitt||\pit)\\
    &\quad - \tau_t \text{KL}(u||\pitt) + \text{KL}(\pitt||\pit) + \frac{1}{2}||d_t||_\infty^2 \tag{Pinsker's inequality}\\
    & =  \text{KL} (u||\pit) - \text{KL}(u||\pitt) - \tau_t \text{KL}(u||\pitt) + \frac{1}{2}||d_t||_\infty^2 \\
    & \leq \text{KL} (u||\pit) - \text{KL}(u||\pitt) - \tau_t \text{KL}(u||\pitt) + \frac{D_t^2 }{2} \tag{Since $\norminf{d_t} \leq D_t$}\\
    & =  \text{KL} (u||\pit) - \text{KL}(u||\pitt) - \eta_t \tau \text{KL}(u||\pitt) + \frac{D_t^2 }{2} \tag{Since $\taut = \etat \, \tau$}
\end{align*}
Rearranging and dividing both-sides by $\eta_t$ we get 
\begin{align*}
\frac{\inner{\pit - u}{d_t}}{\etat} + \cR(\pitt) - \cR(u) & \leq \frac{\text{KL} (u||\pit)}{\eta_t} - \frac{\text{KL}(u||\pitt)}{\eta_t} -  \tau \text{KL}(u||\pitt) + \frac{D_t^2 }{2 \, \etat} \\
\frac{\inner{\pit - u}{d_t}}{\etat} + \cR(\pit) - \cR(u) & \leq [\cR(\pit) - \cR(\pitt)] + \frac{\text{KL} (u||\pit)}{\eta_t} - \frac{\text{KL}(u||\pitt)}{\eta_t}\\
&\quad-  \tau \text{KL}(u||\pitt) + \frac{D_t^2 }{2 \, \etat} \\
\intertext{Summing from $t = 0$ to $K -1$,}
\sum_{t = 0}^{K-1} \left[\frac{\inner{\pit - u}{d_t}}{\etat} + \cR(\pit) - \cR(u) \right] & \leq [\cR(\pi_0) - \cR(\pi_K)] + \sum_{t = 0}^{K-1} \Bigg[\frac{\text{KL} (u||\pit)}{\eta_t} - \frac{\text{KL}(u||\pitt)}{\eta_t}\\
&\hspace{4.3cm} -  \tau \text{KL}(u||\pitt) + \frac{D_t^2 }{2 \, \etat}\Bigg] \\
& \leq \sum_{t = 0}^{K-1} \left[\frac{\text{KL} (u||\pit)}{\eta_t} - \frac{\text{KL}(u||\pitt)}{\eta_t} -  \tau \text{KL}(u||\pitt) \right]\\ 
&\quad + \sum_{t = 0}^{K-1} \frac{D_t^2 }{2 \, \etat} \tag{Since $\cR(\pi_0) = \ln(A) - \cH(\pi_0) = 0$ and $\cR(\pi_K) \geq 0$} 
\end{align*}
\end{proof}

\subsubsection{\texorpdfstring{Regret Bound for (soft) $\NPG$}{Regret Bound for (soft) NPG}}
\label{app:po-soft-npg-cor-proofs}
In this section, we instantiate the regret and policy evaluation bounds for (soft) $\NPG$
\begin{corollary}[Regret Bounds]
Suppose $\pi_0(\cdot|s)$ is the uniform distribution over actions for each state $s$. For any sequence $\{q_{\zeta}^t\}_{t = 0}^{K-1}$ satisfying $\norminf{q_{\zeta}^t} \leq H_\tau$, the regret for (soft) $\NPG$ can be bounded as:
\begin{itemize}
    \item \textbf{Soft $\NPG$}: Setting $\etat = \frac{1}{c + \tau \, (t+1)}$ for a constant $c \geq 0$ to be determined later, guarantees that,
    \begin{align*}
    \max_{s} \; & \left \vert 
    \sum_{t = 0}^{K-1} 
    \big[ \inner{\pi^*_\tau(\cdot|s) - \pit(\cdot|s)}{q_\zeta^t(s,\cdot)} 
    + \tau \, [\cH(\pi^*_\tau(\cdot|s)) - \cH(\pit(\cdot|s))] 
    \big] \right \vert \\[0.3em]
    & \leq \frac{H_\tau^2}{2 \, \tau} \, [1 + \ln(K)] + (c + \tau) \, \ln(A)
    \end{align*}

    \item \textbf{$\NPG$}: Setting $\etat = \eta = \frac{\sqrt{2} \, (1-\gamma) \, \sqrt{\ln(A)}}{\sqrt{K}}$ guarantees that,
    \begin{align*}
    \max_{s} \left|\sum_{t = 0}^{K-1} \left[\inner{\pi^*(\cdot|s) - \pit(\cdot|s)}{q_\zeta^t(s,\cdot)} \right]\right| \leq \frac{\sqrt{2 \ln(A)} \, \sqrt{K}}{1-\gamma}        
    \end{align*}
    
\end{itemize}
\label{cor:npg-regret-guarantee}    
\end{corollary}
\begin{proof}
First note that by using the mirror descent view of the (soft) $\NPG$ update~\citep{xiao2022convergence}, it can be equivalently written as: for all $s \in \cS$, 
\begin{align*}
\pitt(\cdot|s) = \argmin_{\pi \in \Delta} \left[-\etat \langle q_\zeta^t(s, \cdot), \pi(\cdot|s) \rangle + \text{KL}(\pi(\cdot|s) || \pit(\cdot|s)) - \taut \cH(\pi(\cdot|s)) \right]     
\end{align*}
By comparing to the update in~\cref{thm:generic-regret-bound}, we note that $d_t = -\etat \, q_\zeta^t(s,\cdot)$ and $\norminf{d_t} = \etat \, \norminf{q_\zeta^t} \leq \etat \, H_\tau$. If $\tau_t = \etat \tau$ and $\pi_0(\cdot|s)$ is a uniform distribution for each state $s$, we can instantiate~\cref{thm:generic-regret-bound} for each state $s$, and obtain the following regret bound for the comparator $u$.
\begin{align*}
\sum_{t = 0}^{K-1} & \left[\inner{u(\cdot|s) - \pit(\cdot|s)}{q_\zeta^t(s,\cdot)} + \tau \, [\cH(u(\cdot|s)) - \cH(\pit(\cdot|s))] \right] \\
& \leq \sum_{t = 0}^{K-1} \left[\frac{\text{KL} (u(\cdot|s)||\pit(\cdot|s))}{\eta_t} - \frac{\text{KL}(u(\cdot|s)||\pitt(\cdot|s))}{\eta_t} -  \tau \text{KL}(u(\cdot|s)||\pitt(\cdot|s)) \right]\\
&\quad + \frac{H_\tau^2}{2} \sum_{t = 0}^{K-1} \etat     
\end{align*}
Now we consider two cases corresponding to $\NPG$ and its soft variant. 

\textbf{Soft $\NPG$}: Using that $\tau \neq 0$, setting $u = \pi^*_\tau$ and bounding the RHS in the above inequality, 
\begin{align*}
\sum_{t = 0}^{K-1} & \left[\inner{\pi^*_\tau(\cdot|s) - \pit(\cdot|s)}{q_\zeta^t(s,\cdot)} + \tau \, [\cH(\pi^*_\tau(\cdot|s)) - \cH(\pit(\cdot|s))] \right] \\
& \leq \sum_{t = 1}^{K-1} \text{KL} (\pi^*_\tau(\cdot|s)||\pit(\cdot|s))
\left[ \frac{1}{\etat} - \frac{1}{\eta_{t-1}} - \tau \right] + \frac{1}{\eta_0} \, \text{KL} (\pi^*_\tau(\cdot|s)||\pi_0(\cdot|s)) + \frac{H_\tau^2}{2} \sum_{t = 0}^{K-1} \etat \\
& = \frac{H_\tau^2}{2} \sum_{t = 0}^{K-1} \frac{1}{c + \tau \, (t+1)} + (c + \tau) \, \text{KL} (\pi^*_\tau(\cdot|s)||\pi_0(\cdot|s)) \tag{Setting $\etat = \frac{1}{c + \tau \, (t+1)}$} \\
& \leq \frac{H_\tau^2}{2} \sum_{t = 0}^{K-1} \frac{1}{\tau \, (t+1)} + (c + \tau) \, \ln(A) \tag{Since $\pi_0(\cdot|s)$ is a uniform distribution for all $s$} \\
& \leq \frac{H_\tau^2}{2 \, \tau} \, [1 + \ln(K)] + (c + \tau) \, \ln(A) \tag{Since $\sum_{t = 1}^{K} 1/t \leq 1 + \ln(K)$}
\end{align*}
Since the above bound holds for all $s$, 
\begin{align*}
\max_{s} \; & \left \vert 
\sum_{t = 0}^{K-1} 
\big[ \inner{\pi^*_\tau(\cdot|s) - \pit(\cdot|s)}{q_\zeta^t(s,\cdot)} 
+ \tau \, [\cH(\pi^*_\tau(\cdot|s)) - \cH(\pit(\cdot|s))] 
\big] \right \vert \\[0.3em]
& \leq \frac{H_\tau^2}{2 \, \tau} \, [1 + \ln(K)] + (c + \tau) \, \ln(A)
\end{align*}

\textbf{$\NPG$}: Using $u = \pi^*$ and a constant step-size i.e. $\etat = \eta$ for all $t$, in which case the regret bound can be simplified as:
\begin{align*}
\sum_{t = 0}^{K-1} \left[\inner{\pi^*(\cdot|s) - \pit(\cdot|s)}{q_\zeta^t(s,\cdot)} \right] & \leq \frac{1}{\eta} \text{KL} (\pi^*(\cdot|s)||\pi_0(\cdot|s) + \frac{\eta K}{2 \, (1 - \gamma)^2} \\
& \le \frac{\ln(A)}{\eta} + \frac{\eta K}{2 \, (1 - \gamma)^2} \tag{Since $\pi_0$ is the uniform distribution} \\
& \leq \frac{\sqrt{2 \ln(A)} \, \sqrt{K}}{1-\gamma} \tag{Setting $\eta = \frac{\sqrt{2} \, (1-\gamma) \, \sqrt{\ln(A)}}{\sqrt{K}}$}
\end{align*}
Since the above bound holds for all $s$, 
\[
\max_{s} \left|\sum_{t = 0}^{K-1} \left[\inner{\pi^*(\cdot|s) - \pit(\cdot|s)}{q_\zeta^t(s,\cdot)} \right]\right| \leq \frac{\sqrt{2 \ln(A)} \, \sqrt{K}}{1-\gamma}.
\]
\end{proof}

\subsubsection{\texorpdfstring{Regret Bound for (soft) $\SPMA$}{Regret Bound for (soft) SPMA}}
\label{app:po-soft-spma-cor-proofs}
In this section, we instantiate the regret and policy evaluation bounds for (soft) $\SPMA$.
\begin{corollary}[Regret Bounds]
Suppose $\pi_0(\cdot|s)$ is the uniform distribution over actions for each state $s$, and let $\etat = \frac{1}{c + \tau \, (t+1)}$ for some constant $c \geq 0$ to be determined later. For any sequence $\{q_{\zeta}^t\}_{t = 0}^{K-1}$ satisfying $\norminf{q_{\zeta}^t} \leq H_\tau$, the regret for (soft) $\SPMA$ can be bounded as:
\begin{itemize}
    \item \textbf{Soft $\SPMA$}: Setting $\etat = \frac{1}{c + \tau \, (t+1)}$ for a constant $c \geq  2 \, \max\{H_\tau, \zeta \ln(A)\}$, guarantees that,
    \begin{align*}
    \max_{s} \; & \left \vert 
    \sum_{t = 0}^{K-1} 
    \big[ \inner{\pi^*_\tau(\cdot|s) - \pit(\cdot|s)}{q_\zeta^t(s,\cdot)} 
    + \tau \, [\cH(\pi^*_\tau(\cdot|s)) - \cH(\pit(\cdot|s))] 
    \big] \right \vert \\[0.3em]
    & \leq \frac{3 \, H_\tau^2}{\tau} \, [1 + \ln(K)] + (c + \tau) \, \ln(A)
    \end{align*}
    \item \textbf{$\SPMA$}: Setting $\etat = \eta = \min\left\{\frac{1-\gamma}{2}, \frac{\sqrt{2} \, (1-\gamma) \, \sqrt{\ln(A)}}{\sqrt{K}} \right\}$ guarantees that,
    \begin{align*}
    \max_{s} \left|\sum_{t = 0}^{K-1} \left[\inner{\pi^*(\cdot|s) - \pit(\cdot|s)}{q_\zeta^t(s,\cdot)} \right]\right|  \leq \frac{7\, \sqrt{\ln(A)} \, \sqrt{K}}{\sqrt{2} \, (1-\gamma)} + \frac{2}{1-\gamma} \, \ln(A).
    \end{align*}
    
\end{itemize}
\label{cor:spma-regret-guarantee}    
\end{corollary}
\begin{proof}
For a fixed state $s \in \cS$, first note that the (soft) $\SPMA$ update in~\cref{eq:policy-update-rkl} can be equivalently be written as follows: if $\Delta^t(s,a) := q_\zeta^{t}(s,a) - v_\zeta^t(s)$ for $d_t := -\ln(1 + \etat \, \Delta^t(s,\cdot))$, 
\begin{align*}
\pitt(\cdot|s) &=  \argmin_{\pi \in \Delta} \left[\text{KL}(\pi(\cdot|s) || \pit(\cdot|s) \, \left[1 + \eta_t \, \left(\Delta^t(s, \cdot)\right) \right]) - \taut \cH(\pi(\cdot|s)) \right] \\
& = \argmin_{\pi \in \Delta} \left[\langle d_t, \pi(\cdot|s) \rangle + \text{KL}(\pi(\cdot|s) || \pit(\cdot|s)) - \taut \cH(\pi(\cdot|s)) \right] \,,
\end{align*}
where we require that $1 + \etat \Delta^t(s,\cdot) \geq 0$. Note that since $\norminf{q_\zeta^t(s,\cdot)} \leq H_\tau$, $\norminf{\Delta^t(s,\cdot))} \leq H_\tau$ and we require that $\etat \leq \frac{1}{2 H_\tau}$. With this choice, $|\etat \, \Delta^t(s,a)| \leq \frac{1}{2}$. By comparing to the update in~\cref{thm:generic-regret-bound}, we note that $d_t = -\ln(1 + \etat \, \Delta^t(s,\cdot))$. Since $|\ln(1+x)| \leq 2|x|$ for all $x \geq -\frac{1}{2}$, 
\begin{align*}
\vert -\ln(1 + \etat \, \Delta^t(s,a)) \vert \leq 2 \, \etat \, \vert \Delta^t(s,a) \vert \leq 2 \, \etat \, H_\tau \implies \norminf{d_t} \leq 2 \, \etat \, H_\tau
\end{align*} 
Hence, $D_t = 2 \, \etat \, H_\tau$. If $\tau_t = \etat \tau$ and $\pi_0(\cdot|s)$ is a uniform distribution for each state $s$, we can instantiate~\cref{thm:generic-regret-bound} for each state $s$, and obtain the following regret bound for the comparator $u$, 
\begin{align*}
\sum_{t = 0}^{K-1} & \left[\frac{\inner{u(\cdot|s) - \pit(\cdot|s)}{\ln(1 + \etat \, \Delta^t(s,\cdot))}}{\etat} + \tau \, [\cH(u(\cdot|s)) - \cH(\pit(\cdot|s))] \right] \\
& \leq \sum_{t = 0}^{K-1} \left[\frac{KL (u(\cdot|s)||\pit(\cdot|s))}{\eta_t} - \frac{KL(u(\cdot|s)||\pitt(\cdot|s))}{\eta_t} -  \tau KL(u(\cdot|s)||\pitt(\cdot|s)) \right]\\
&\quad + 2 H_\tau^2  \, \sum_{t = 0}^{K-1} \etat
\end{align*}
In order to simplify the above expression, first note that, 
\begin{align*}
    \inner{\pit(\cdot|s)}{\ln(1+\eta_t \, \Delta^t(s,\cdot))} & \leq \ln( 1 -\eta_t \, \zeta \, \cH(\pit(\cdot|s)) ) \tag{Using Jensen's inequality and the fact $\sum_a \pi_t(a|s) \Delta^t(s, a) = -\zeta \, \cH(\pit(\cdot|s)$}\\ 
\end{align*}
If $\etat$ is chosen such that $\eta_t \, \zeta \, \cH(\pit(\cdot|s)) \le \tfrac{1}{2}$, and since $\cH(\pi) \in [0,\ln(A)]$, it suffices to ensure $\etat \le \frac{1}{2 \zeta \, \ln(A)},$ and use the fact $\ln(1+x) \leq x $ for $x > -1$  to guarantee:
\begin{align}
    \inner{\pit(\cdot|s)}{\ln(1+\eta_t \, \Delta^t(s,\cdot))} & \leq -\eta_t \, \zeta \, \cH(\pit(\cdot|s)) \label{eq:pit-ln-upperbound}
\end{align}
On the other hand, since choosing $\etat \leq \frac{1}{2 H_\tau}$ ensures $\ln(1+\eta_t \Delta^t(s,\cdot))$ is well-defined,
\begin{align*}
    \inner{u(\cdot|s)}{\ln(1+\eta_t \Delta^t(s,\cdot))} &\geq [\inner {u(\cdot|s)}{\eta_t \Delta^t(s,\cdot)} -\inner {u(\cdot|s)}{\eta_t^2 [\Delta^t(s,\cdot)]^2}] \tag{since $\ln(1+x) \geq x -x^2$ for $x> -1/2 $}\\
    &= \eta_t \left(\inner{u(\cdot|s)}{q_\zeta^t(s,\cdot)} - v_\zeta^t(s) \right) - \eta_t^2 \, \inner{u(\cdot|s)}{[\Delta^t(s,\cdot)]^2} \\
    &= \eta_t \left(\inner{u(\cdot|s)}{q_\zeta^t(s,\cdot)} -\inner{\pit}{q_\zeta^t(s,\cdot)} -  \zeta \cH(\pit(\cdot|s))\right)\\
    &\quad - \eta_t ^2 \inner {u(\cdot|s)}{[\Delta^t(s,\cdot)]^2} \tag{since $v_\zeta^t(s) = \inner{\pit(\cdot|s)}{q_\zeta^t} + \zeta \cH(\pit(\cdot|s))$} \\
    & \geq \eta_t \left(\inner{u(\cdot|s)}{q_\zeta^t(s,\cdot)} -\inner{\pit(\cdot|s)}{q_\zeta^t(s,\cdot)} - \zeta \cH(\pit(\cdot|s))\right) - \eta_t ^2 H_\tau^2 \tag{since $\norminf{\Delta^t(s,\cdot))} \leq H_\tau$}
\end{align*}
Combining the above inequalities with~\cref{eq:pit-ln-upperbound}, 
\begin{align*}
\frac{\inner{u(\cdot|s) - \pit(\cdot|s)}{\ln(1+\eta_t \, \Delta^t(s,\cdot))}}{\etat} & \geq \inner{u(\cdot|s)}{q_\zeta^t(s,\cdot)} -\inner{\pit}{q_\zeta^t(s,\cdot)} - \etat  \, H_\tau^2 \\
& = \langle u(\cdot|s) - \pit(\cdot|s), q_\zeta^t(s,\cdot)\rangle - \etat H_\tau^2
\end{align*}
Using the above relation with the regret expression, 
\begin{align*}
\sum_{t = 0}^{K-1} & \left[\langle u(\cdot|s) - \pit(\cdot|s), q_\zeta^t(s,\cdot)\rangle + \tau \, [\cH(u(\cdot|s)) - \cH(\pit(\cdot|s))] \right] \\
& \leq \sum_{t = 0}^{K-1} \left[\frac{KL (u(\cdot|s)||\pit(\cdot|s))}{\eta_t} - \frac{KL(u(\cdot|s)||\pitt(\cdot|s))}{\eta_t} -  \tau KL(u(\cdot|s)||\pitt(\cdot|s)) \right]\\
&\quad + 3 \, H_\tau^2 \, \sum_{t = 0}^{K-1} \etat
\end{align*}
Note that we require $\etat \leq \frac{1}{2 H_\tau}$ and $\etat \leq \frac{1}{2\zeta \ln(A)}$ simultaneously for all $t$. Hence, it is sufficient to ensure that $\etat \leq \frac{1}{2 \, \max\{H_\tau, \zeta \ln(A)\}}$ for all $t$, and hence require that $c \geq  2 \, \max\{H_\tau, \zeta \ln(A)\}$. 

Now we consider two cases corresponding to $\SPMA$ and its soft variant. 

\textbf{Soft $\SPMA$}: Using that $\tau \neq 0$, setting $u = \pi^*_\tau$ and bounding the RHS in the above inequality, 
\begin{align*}
\sum_{t = 0}^{K-1} & \left[\inner{\pi^*_\tau(\cdot|s) - \pit(\cdot|s)}{q_\zeta^t(s,\cdot)} + \tau \, [\cH(\pi^*_\tau(\cdot|s)) - \cH(\pit(\cdot|s))] \right] \\
& \leq \sum_{t = 1}^{K-1} \text{KL} (\pi^*_\tau(\cdot|s)||\pit(\cdot|s))
\left[ \frac{1}{\etat} - \frac{1}{\eta_{t-1}} - \tau \right] + \frac{1}{\eta_0} \, \text{KL} (\pi^*_\tau(\cdot|s)||\pi_0(\cdot|s))\\
&\quad + 3 \, H_\tau^2 \sum_{t = 0}^{K-1} \etat \\
& = 3 \, H_\tau^2 \, \sum_{t = 0}^{K-1} \frac{1}{c + \tau \, (t+1)} + (c + \tau) \,  \text{KL} (\pi^*_\tau(\cdot|s)||\pi_0(\cdot|s)) \tag{Since $\etat = \frac{1}{c + \tau \, (t+1)}$} \\
& \le 3 \, H_\tau^2 \, \sum_{t = 0}^{K-1} \frac{1}{c + \tau \, (t+1)} + (c + \tau) \,  \ln(A) \tag{Since $\pi_0(\cdot|s)$ is a uniform distribution for all $s$} \\
& \leq \frac{3 \, H_\tau^2}{\tau} \, \sum_{t = 0}^{K-1} \frac{1}{t+1} + (c + \tau) \,  \ln(A) \\
& \leq \frac{3 \, H_\tau^2}{\tau} \, [1 + \ln(K)] + (c + \tau) \, \ln(A) \tag{Since $\sum_{t = 1}^{K} 1/t \leq 1 + \ln(K)$} \\
\end{align*}
Since the above bound holds for all $s$, 
\begin{align*}
    \max_{s} \; & \left \vert 
    \sum_{t = 0}^{K-1} 
    \big[ \inner{\pi^*_\tau(\cdot|s) - \pit(\cdot|s)}{q_\zeta^t(s,\cdot)} 
    + \tau \, [\cH(\pi^*_\tau(\cdot|s)) - \cH(\pit(\cdot|s))] 
    \big] \right \vert \\[0.3em]
    & \leq \frac{3 \, H_\tau^2}{\tau} \, [1 + \ln(K)] + (c + \tau) \, \ln(A)
    \end{align*}


\vspace{1ex}
\textbf{$\SPMA$}: Using $u = \pi^*$, $\tau = \zeta = 0$, and a constant step-size i.e. $\etat = \eta$ for all $t$, in which case the regret bound can be simplified as:
\begin{align*}
\sum_{t = 0}^{K-1} \left[\inner{\pi^*(\cdot|s) - \pit(\cdot|s)}{q_\zeta^t(s,\cdot)} \right] & \leq \frac{1}{\eta} \text{KL} (\pi^*(\cdot|s)||\pi_0(\cdot|s)) + \frac{3\, \eta K}{(1 - \gamma)^2} \\
& \le \frac{\ln(A)}{\eta} + \frac{3 \, \eta K}{(1 - \gamma)^2} \tag{Since $\pi_0(\cdot|s)$ is a uniform distribution for all $s$} \\
\intertext{Recall that in the presence of entropy ensuring $\ln(1+\etat \Delta^t(s,\cdot))$ is well-defined required us to choose $\etat \le \frac{1}{2 \, H_\tau}$. When $\tau=0$ and $\etat = \eta$, the condition simplifies to $\eta \le \frac{1-\gamma}{2}$ . Setting $\eta = \min\left\{\frac{\sqrt{2} \, (1-\gamma) \, \sqrt{\ln(A)}}{\sqrt{K}}, \frac{1-\gamma}{2} \right\}$ and using the fact that $\frac{1}{\min \{a, b\}} =\max\{1/a, 1/b\}$}
& \leq \ln(A) \, \max \left\{\frac{\sqrt{K}}{\sqrt{2} \, (1-\gamma) \, \sqrt{\ln(A)}}, \frac{2}{1-\gamma}\right\}\\
&\quad+ \frac{3 \, K}{(1 - \gamma)^2} \, \min\left\{\frac{\sqrt{2} \, (1-\gamma) \, \sqrt{\ln(A)}}{\sqrt{K}}, \frac{1-\gamma}{2} \right\} \\
& \leq \frac{7\, \sqrt{\ln(A)} \, \sqrt{K}}{\sqrt{2} \, (1-\gamma)} + \frac{2}{1-\gamma} \, \ln(A) \tag{Since $\max\{a,b\} \leq a + b$, $\min\{a,b\} \leq a$}
\end{align*}
Since the above bound holds for all $s$, 
\[
\max_{s} \left|\sum_{t = 0}^{K-1} \left[\inner{\pi^*(\cdot|s) - \pit(\cdot|s)}{q_\zeta^t(s,\cdot)} \right]\right| \leq \frac{7\, \sqrt{\ln(A)} \, \sqrt{K}}{\sqrt{2} \, (1-\gamma)} + \frac{2}{1-\gamma} \, \ln(A).
\]
\end{proof}

\subsection{Proof of~\texorpdfstring{\cref{thm:soft-npg-pe-entropy-main}}{}}
\label{app:thm-main-result-proof}
Finally, we put everything together, and in the following two subsections, we prove theorems that quantify the performance of (soft) $\NPG$ and (soft) $\SPMA$ when using the hard or soft Bellman operator (i.e., $\zeta=\tau$ or $\zeta=0$) as well as the case of no entropy regularization (i.e, $\tau=0$ and $\zeta=0$).
\subsubsection{\texorpdfstring{Putting everything together for soft $\NPG$}{Putting everything together for soft NPG}}
\label{app:soft-npg-main-thm-proof}
\begin{theorem}[Sub-optimality of soft $\NPG$]
Let $\pi^*_\tau$ denote the optimal entropy-regularized policy with value function $v^*_\tau$. Consider the soft $\NPG$ update with step size $\etat = \frac{1}{c + \tau \, (t+1)}$, $c \ge \max \left\{\frac{8 \, (1 + \tau \, \ln(A))}{(1-\gamma)}, 32 \, \tau \, \ln(A)\right\}$ and, $\delta(\tau, \zeta) := \frac{|\tau - \zeta|\, \ln(A)}{1-\gamma}$. Let $\pi_0(\cdot|s)$ be the uniform policy over actions for all $s \in \cS$ and assume the policy evaluation step in~\cref{eq:pe-general}. Then the resulting mixture policy $\bar{\pi}_K$ satisfies the following sub-optimality bound,   
\begin{align*}
    \norminf{ v_\tau^{\bar{\pi}_{K}} - v_\tau^*} & \leq  \frac{1}{K \, (1-\gamma)} \, \left[ \frac{(1 + \tau \, \ln(A))^2}{2 \, \tau \, (1 - \gamma)^2} \, [1 + \ln(K)] + (c + \tau) \, \ln(A) \right] \\
    &\quad + \frac{16 (1+\tau \ln(A)) \gamma^m}{(1 - \gamma)^4 \, K} \, \bigg[ \left(1 + \tau \ln\left(A \, K^4 \right) \right)\\ &\hspace{3.9cm}\bigg(\big(\ln(A \, K^4) + \frac{1+\tau \ln(A)}{\tau(1-\gamma)}\big)  \, \big(1 + \ln(K) + \frac{1}{1-\sqrt{\gamma}}\big)\\
    &\hspace{4.2cm}
    + \sqrt{\tau}\bigg)\\
    &\hspace{3.9cm} + \tau \, \left(\ln(A) + 1 \right) \, \frac{1}{K} \bigg] + \frac{2 \, \delta(\tau, \zeta)}{(1-\gamma)^2}
\end{align*}
\label{thm:soft-npg-pe-entropy}
\end{theorem}
\begin{proof}
Plugging the regret bound in~\cref{cor:npg-regret-guarantee} for soft $\NPG$ into the regret part of~\cref{thm:meta-entropy} immediately gives the first part of the upper-bound:
\begin{align*}
    \frac{\norminf{\text{Regret}(K)}}{K \, (1 - \gamma)} \le \frac{1}{K \, (1-\gamma)} \, \left[ \frac{(1 + \tau \, \ln(A))^2}{2 \, \tau \, (1 - \gamma)^2} \, [1 + \ln(K)] + (c + \tau) \, \ln(A) \right]
\end{align*}
Using the result from~\cref{thm:error-analysis-general} and~\cref{cor:npg-pe-entropy} to upper-bound the error part $E_K := \frac{2 \, \sum_{k \in [K]} \norminf{\bepsilon_k}}{K \, (1 - \gamma)}$in~\cref{thm:meta-entropy} we obtain:
\begin{align*}
    E_K & \le \frac{16 (1+\tau \ln(A)) \gamma^m}{(1 - \gamma)^4 \, K} \, \bigg[ \left(1 + \tau \ln\left(A \, K^4 \right) \right)\\ &\hspace{3.5cm}\left(\big(\ln(A \, K^4) + \frac{1+\tau \ln(A)}{\tau(1-\gamma)}\big)  \, \big(1 + \ln(K) + \frac{1}{1-\sqrt{\gamma}}\big)
    + \sqrt{\tau}\right)\\
    &\hspace{3.5cm} + \tau \, \left(\ln(A) + 1 \right) \, \frac{1}{K} \bigg] + \frac{2 \, \delta(\tau, \zeta)}{(1-\gamma)^2}\\
    \intertext{Where the above inequality can be obtained using the fact $\sum_{t=1}^K \frac{1}{t} \le 1 + \ln(K)$ (using integration) and $\sum_{t=1}^K (\gamma^m)^{t/2}  \le \frac{1}{1 - \gamma^{m/2}} \le \frac{1}{1-\sqrt{\gamma}}$}
\end{align*}
\end{proof}

\begin{corollary}[Sub-optimality of soft $\NPG$ with critic entropy]
Let $\pi^*_\tau$ denote the optimal entropy-regularized policy with value function $v^*_\tau$. Consider the soft $\NPG$ update with step size $\etat = \frac{1}{c + \tau \, (t+1)}$ and $c \ge \max \left\{\frac{8 \, (1 + \tau \, \ln(A))}{(1-\gamma)}, 32 \, \tau \, \ln(A)\right\}$. Let $\pi_0(\cdot|s)$ be the uniform policy over actions for all $s \in \cS$ and assume the policy evaluation step in~\cref{eq:pe-general} with $\zeta=\tau$. Then the resulting mixture policy $\bar{\pi}_K$ satisfies the following sub-optimality bound,   
\begin{align*}
    \norminf{ v_\tau^{\bar{\pi}_{K}} - v_\tau^*} & \leq  \frac{1}{K \, (1-\gamma)} \, \left[ \frac{(1 + \tau \, \ln(A))^2}{2 \, \tau \, (1 - \gamma)^2} \, [1 + \ln(K)] + (c + \tau) \, \ln(A) \right] \\
    &\quad + \frac{16 (1+\tau \ln(A)) \gamma^m}{(1 - \gamma)^4 \, K} \, \bigg[ \left(1 + \tau \ln\left(A \, K^4 \right) \right)\\ &\hspace{3.9cm}\bigg(\big(\ln(A \, K^4) + \frac{1+\tau \ln(A)}{\tau(1-\gamma)}\big)  \, \big(1 + \ln(K) + \frac{1}{1-\sqrt{\gamma}}\big)\\
    &\hspace{4.2cm}
    + \sqrt{\tau}\bigg)\\
    &\hspace{3.9cm} + \tau \, \left(\ln(A) + 1 \right) \, \frac{1}{K} \bigg]
\end{align*}
\label{corr:soft-npg-pe-entropy}
\end{corollary}
\begin{proof}
Using~\cref{thm:meta-entropy} with~\cref{cor:npg-regret-guarantee} for soft $\NPG$,~\cref{thm:error-analysis-general}, ~\cref{cor:npg-pe-entropy} and the result from~\cref{thm:soft-npg-pe-entropy}. 
\end{proof}

\begin{corollary}[Sub-optimality of soft $\NPG$ without critic entropy]
Let $\pi^*_\tau$ denote the optimal entropy-regularized policy with value function $v^*_\tau$. Consider the soft $\NPG$ update with step size $\etat = \frac{1}{c + \tau \, (t+1)}$ and $c \ge \max \left\{\frac{8 \, (1 + \tau \, \ln(A))}{(1-\gamma)}, 32 \, \tau \, \ln(A)\right\}$. Let $\pi_0(\cdot|s)$ be the uniform policy over actions for all $s \in \cS$ and assume the policy evaluation step in~\cref{eq:pe-general} with $\zeta=0$. Then the resulting mixture policy $\bar{\pi}_K$ satisfies the following sub-optimality bound,   
\begin{align*}
    \norminf{ v_\tau^{\bar{\pi}_{K}} - v_\tau^*} & \leq  \frac{1}{K \, (1-\gamma)} \, \left[ \frac{(1 + \tau \, \ln(A))^2}{2 \, \tau \, (1 - \gamma)^2} \, [1 + \ln(K)] + (c + \tau) \, \ln(A) \right] \\
    &\quad + \frac{16 (1+\tau \ln(A)) \gamma^m}{(1 - \gamma)^4 \, K} \, \bigg[ \left(1 + \tau \ln\left(A \, K^4 \right) \right)\\ &\hspace{3.9cm}\bigg(\big(\ln(A \, K^4) + \frac{1+\tau \ln(A)}{\tau(1-\gamma)}\big)  \, \big(1 + \ln(K) + \frac{1}{1-\sqrt{\gamma}}\big)\\
    &\hspace{4.2cm}
    + \sqrt{\tau}\bigg)\\
    &\hspace{3.9cm} + \tau \, \left(\ln(A) + 1 \right) \, \frac{1}{K} \bigg] + \frac{2 \, \tau \, \ln(A)}{(1-\gamma)^3}
\end{align*}
\label{corr:soft-npg-pe-no-entropy}
\end{corollary}
\begin{proof}
Using~\cref{thm:meta-entropy} with~\cref{cor:npg-regret-guarantee} for soft $\NPG$,~\cref{thm:error-analysis-general}, ~\cref{cor:npg-pe-no-entropy} and the result from~\cref{thm:soft-npg-pe-entropy}. 
\end{proof}

\begin{theorem}[$\NPG$ + policy evaluation without entropy regularization]
If $\pi^*$ is the optimal policy whose value function is equal to $v^*$, the $\NPG$ update with $\etat = \eta = \frac{\sqrt{2} \, (1-\gamma) \, \sqrt{\ln(A)}}{\sqrt{K}}$, $\pi_0(\cdot|s)$ as the uniform initial policy for each $s \in \cS$ with the policy evaluation procedure in~\cref{eq:pe-general} with $\zeta=0$ satisfies the following sub-optimality bound for the mixture policy $\bar{\pi}_K$,   
\begin{align*}
\norminf{ v^{\bar{\pi}_{K}} - v^*} & \leq \frac{\sqrt{2 \ln(A)}}{\sqrt{K} \, (1-\gamma)^2} + \frac{4\, \sqrt{\ln(A)} \, \gamma^m}{\sqrt{K} \, (1-\gamma)^4}
\end{align*}
\label{thm:npg-pe-no-entropy}
\end{theorem}
\begin{proof}
Using~\cref{cor:meta-no-entropy} with~\cref{cor:npg-regret-guarantee} for $\NPG$ and~\cref{cor:npg-pe-no-entropy}.  
\end{proof}

\subsubsection{\texorpdfstring{Putting everything together for soft $\SPMA$}{Putting everything together for soft SPMA}}
\label{app:soft-spma-main-thm-proof}

\begin{theorem}[Sub-optimality of soft $\SPMA$]
Let $\pi^*_\tau$ denote the optimal entropy-regularized policy with value function $v^*_\tau$. Consider the soft $\SPMA$ update with step size $\etat = \frac{1}{c + \tau \, (t+1)}$, $c \ge \max \left\{\frac{4 \, (1 + \tau \, \ln(A))}{(1-\gamma)}, 32 \, \tau \, \ln(A) \right\}$, and, $\delta(\tau, \zeta) := \frac{|\tau - \zeta|\, \ln(A)}{1-\gamma}$. Let $\pi_0(\cdot|s)$ be the uniform policy over actions for all $s \in \cS$ and assume the policy evaluation step in~\cref{eq:pe-general}. Then the resulting mixture policy $\bar{\pi}_K$ satisfies the following sub-optimality bound,   
\begin{align*}
\norminf{ v_\tau^{\bar{\pi}_{K}} - v_\tau^*} & \leq \frac{1}{K \, (1-\gamma)} \, \left[\frac{3 \, (1 + \tau \, \ln(A))^2}{2 \, \tau \, (1-\gamma)^2} \, [1 + \ln(K)] + (c + \tau) \, \ln(A) \right] \\
&\quad + \frac{16 (1+\tau \ln(A)) \gamma^m}{(1 - \gamma)^4 \, K} \, \bigg[ \left(1 + \tau \ln\left(A \, K^4 \right) \right)\\ &\hspace{3.9cm}\bigg(\big(\ln(A \, K^4) + \frac{1+\tau \ln(A)}{\tau(1-\gamma)}\big)  \, \big(1 + \ln(K) + \frac{1}{1-\sqrt{\gamma}}\big)\\
&\hspace{4.2cm} + \sqrt{\tau}\bigg)\\
&\hspace{3.9cm} + \tau \, \left(\ln(A) + 1 \right) \, \frac{1}{K} \bigg] + \frac{2 \, \delta(\tau, \zeta)}{(1-\gamma)^2}
\end{align*}
\label{thm:soft-spma-pe-entropy}
\end{theorem}
\begin{proof}
Using~\cref{thm:meta-entropy} with~\cref{cor:spma-regret-guarantee} for soft $\SPMA$,~\cref{thm:error-analysis-general},~\cref{cor:spma-pe-entropy} and the facts from the proof of~\cref{thm:soft-npg-pe-entropy}. 
\end{proof}

\begin{corollary}[Sub-optimality of soft $\SPMA$ with critic entropy]
Let $\pi^*_\tau$ denote the optimal entropy-regularized policy with value function $v^*_\tau$. Consider the soft $\SPMA$ update with step size $\etat = \frac{1}{c + \tau \, (t+1)}$, $c \ge \max \left\{\frac{4 \, (1 + \tau \, \ln(A))}{(1-\gamma)}, 32 \, \tau \, \ln(A)\right\}$. Let $\pi_0(\cdot|s)$ be the uniform policy over actions for all $s \in \cS$ and assume the policy evaluation step in~\cref{eq:pe-general} with $\zeta=\tau$. Then the resulting mixture policy $\bar{\pi}_K$ satisfies the following sub-optimality bound,   
\begin{align*}
\norminf{ v_\tau^{\bar{\pi}_{K}} - v_\tau^*} & \leq \frac{1}{K \, (1-\gamma)} \, \left[\frac{3 \, (1 + \tau \, \ln(A))^2}{2 \, \tau \, (1-\gamma)^2} \, [1 + \ln(K)] + (c + \tau) \, \ln(A) \right] \\
&\quad + \frac{16 (1+\tau \ln(A)) \gamma^m}{(1 - \gamma)^4 \, K} \, \bigg[ \left(1 + \tau \ln\left(A \, K^4 \right) \right)\\ &\hspace{3.9cm}\bigg(\big(\ln(A \, K^4) + \frac{1+\tau \ln(A)}{\tau(1-\gamma)}\big)  \, \big(1 + \ln(K) + \frac{1}{1-\sqrt{\gamma}}\big)\\
&\hspace{4.2cm} + \sqrt{\tau}\bigg)\\
&\hspace{3.9cm} + \tau \, \left(\ln(A) + 1 \right) \, \frac{1}{K} \bigg]
\end{align*}
\label{corr:soft-spma-pe-entropy}
\end{corollary}
\begin{proof}
Using~\cref{thm:meta-entropy} with~\cref{cor:spma-regret-guarantee} for soft $\SPMA$,~\cref{thm:error-analysis-general},~\cref{cor:spma-pe-entropy},~\cref{thm:soft-spma-pe-entropy} and the facts from the proof of~\cref{thm:soft-npg-pe-entropy}. 
\end{proof}

\begin{corollary}[Sub-optimality of soft $\SPMA$ without critic entropy]
Let $\pi^*_\tau$ denote the optimal entropy-regularized policy with value function $v^*_\tau$. Consider the soft $\SPMA$ update with step size $\etat = \frac{1}{c + \tau \, (t+1)}$, $c \ge \max \left\{\frac{4 \, (1 + \tau \, \ln(A))}{(1-\gamma)}, 32 \, \tau \, \ln(A)\right\}$. Let $\pi_0(\cdot|s)$ be the uniform policy over actions for all $s \in \cS$ and assume the policy evaluation step in~\cref{eq:pe-general} with $\zeta=0$. Then the resulting mixture policy $\bar{\pi}_K$ satisfies the following sub-optimality bound,   
\begin{align*}
\norminf{ v_\tau^{\bar{\pi}_{K}} - v_\tau^*} & \leq \frac{1}{K \, (1-\gamma)} \, \left[\frac{3 \, (1 + \tau \, \ln(A))^2}{2 \, \tau \, (1-\gamma)^2} \, [1 + \ln(K)] + (c + \tau) \, \ln(A) \right] \\
&\quad + \frac{16 (1+\tau \ln(A)) \gamma^m}{(1 - \gamma)^4 \, K} \, \bigg[ \left(1 + \tau \ln\left(A \, K^4 \right) \right)\\ &\hspace{3.9cm}\bigg(\big(\ln(A \, K^4) + \frac{1+\tau \ln(A)}{\tau(1-\gamma)}\big)  \, \big(1 + \ln(K) + \frac{1}{1-\sqrt{\gamma}}\big)\\
&\hspace{4.2cm} + \sqrt{\tau}\bigg)\\
&\hspace{3.9cm} + \tau \, \left(\ln(A) + 1 \right) \, \frac{1}{K} \bigg] + \frac{2 \, \tau \, \ln(A)}{(1-\gamma)^3}
\end{align*}
\label{corr:soft-spma-pe-no-entropy}
\end{corollary}
\begin{proof}
Using~\cref{thm:meta-entropy} with~\cref{cor:spma-regret-guarantee} for soft $\SPMA$,~\cref{thm:error-analysis-general},~\cref{cor:spma-pe-no-entropy},~\cref{thm:soft-spma-pe-entropy} and the facts from the proof of~\cref{thm:soft-npg-pe-entropy}. 
\end{proof}

\begin{theorem}[$\SPMA$ + policy evaluation without entropy regularization]
If $\pi^*$ is the optimal policy whose value function is equal to $v^*$, the $\SPMA$ update in~\cref{eq:spma} with $\etat = \eta = \min\left\{\frac{1-\gamma}{2}, \frac{\sqrt{2} \, (1-\gamma) \, \sqrt{\ln(A)}}{\sqrt{K}} \right\}$, $\pi_0(\cdot|s)$ as the uniform initial policy for each $s \in \cS$ with the policy evaluation procedure in~\cref{eq:pe-general} with $\zeta=0$ satisfies the following sub-optimality bound for the mixture policy $\bar{\pi}_K$,   
\begin{align*}
\norminf{ v^{\bar{\pi}_{K}} - v^*} & \leq \frac{1}{K \, (1-\gamma)} \left[\frac{7\, \sqrt{\ln(A)} \, \sqrt{K}}{\sqrt{2} \, (1-\gamma)} + \frac{2}{1-\gamma} \, \ln(A) \right] + \frac{2  \, \sqrt{\ln(A)} \, \gamma^m}{\sqrt{K} \, (1-\gamma)^4} 
\end{align*}
\label{thm:spma-convergence}
\end{theorem}
\begin{proof}
Using~\cref{cor:meta-no-entropy} with~\cref{cor:spma-regret-guarantee} for $\SPMA$ and~\cref{cor:spma-pe-no-entropy}. 
\end{proof}

\section{Helper Lemmas}
\label{app:helper-lemmas}
\begin{lemma}
For any constant $C \in (0,1/2)$, if $P$ and $Q$ are discrete distributions with support $A$, if $\norm{P - Q}_1 \leq \frac{1}{2}$,then,  
\begin{align*}
\vert \cH(Q) - \cH(P) \vert &\leq \norm{Q - P}_1 \,   \ln\left(\frac{A}{C} \right) +  \left(\frac{\ln(A-1)}{2} + \sqrt{2} \right) \, \sqrt{C}     
\end{align*}
\label{lemma:entropy-difference}
\end{lemma}
\begin{proof}
By~\citet[Theorem 17.3.3]{cover1999elements}, if $\norm{Q - P}_1 \leq \frac{1}{2}$, then, 
\begin{align}
\vert \cH(Q) - \cH(P) \vert \le & \norm{Q - P}_1 \, \ln\left( \frac{A}{\norm{Q - P}_1} \right)
\label{eq:entropy-cover}
\end{align}
Furthermore, using~\citet[Theorem 3]{sason2013entropy}, we also know that, 
\begin{align}
\vert \cH(Q) - \cH(P) \vert & \leq \frac{\ln(A-1)}{2} \, \norm{Q - P}_{1} +  h\left(\frac{1}{2} \, \norm{Q - P}_{1} \right) \,, \nonumber
\intertext{where $h$ is a binary entropy, i.e. for $0<x<1$, $h(x) = -x\ln(x) - (1-x)\ln(1-x)$. Using the fact that $h(x) \leq 2\sqrt{x(1-x)}$,}
& \leq \frac{\ln(A-1)}{2} \, \norm{Q - P}_{1} +  \sqrt{2} \, \sqrt{\norm{Q - P}_{1} \, \left(1 - \frac{\norm{Q - P}_{1}}{2} \right)} \nonumber \\
& \leq \frac{\ln(A-1)}{2} \, \norm{Q - P}_{1} +  \sqrt{2} \, \sqrt{\norm{Q - P}_{1} } \nonumber \\
\intertext{Assuming $\norm{Q - P}_{1} \leq 2$} 
\vert \cH(Q) - \cH(P) \vert & \leq \left(\frac{\ln(A-1)}{2} + \sqrt{2} \right) \, \sqrt{\norm{Q - P}_{1}} \label{eq:entropy-sason}
\end{align}
We will use each of these inequalities for different cases of $\norm{Q - P}_{1}$, where $\norm{Q - P}_{1} \leq \frac{1}{2}$.

\textbf{Case (1)}: For any constant $C \in (0,1/2)$, if $\norm{Q - P}_{1} \geq C$, then, using~\cref{eq:entropy-cover},
\begin{align*}
\vert \cH(Q) - \cH(P) \vert & \leq \norm{Q - P}_1 \,   \ln\left(\frac{A}{C} \right)
\end{align*}

\textbf{Case (2)}: On the other hand, if $\norm{Q - P}_{1} \leq C$, then we use~\cref{eq:entropy-sason} to get that, 
\begin{align*}
\vert \cH(Q) - \cH(P) \vert & \leq \left(\frac{\ln(A-1)}{2} + \sqrt{2} \right) \, \sqrt{C}     
\end{align*}

Combining both cases, if $\norm{Q - P}_{1} \leq \frac{1}{2}$, then, for a constant $C \in (0,1/2)$, 
\begin{align*}
\vert \cH(Q) - \cH(P) \vert &\leq \norm{Q - P}_1 \,   \ln\left(\frac{A}{C} \right) +  \left(\frac{\ln(A-1)}{2} + \sqrt{2} \right) \, \sqrt{C}        
\end{align*}
\end{proof}

\begin{lemma}
For any constant $C \in (0, 1/2)$, and $\alpha = \eta \, \tau$ for $0 < \eta \le 1$ to be determined, if $P$ and $Q$ are discrete distribution with support $A$ such that $\|Q - P\|_1 \le \frac{1}{2}$, then,
\begin{align*}
    \frac{1}{2} \|Q - P\|_1^2 \le \alpha (\cH(Q) - \cH(P)) \implies \|Q - P\|_1 \le 4 \, \eta \, \tau \ln(\frac{A}{C}) + 2 \, \sqrt{\tau} C^{\frac{1}{4}}
\end{align*}
\begin{proof}
If the condition on the left-hand side of the lemma is true we have:
\begin{align*}
    \|Q - P\|_1^2 &\le 2\, \eta \, \tau (\cH(Q) - \cH(P))\\
    &\le 2\, \eta \, \tau \left(2 \, \|Q - P\|_1 \, \ln(\frac{A}{C}) + 2 \, \sqrt{C}\right) \tag{Using the result from~\cref{lemma:entropy-difference}}\\
    &\le 4 \, \eta \, \tau \, \|Q - P\|_1 \, \ln(\frac{A}{C}) + 4 \, \tau \, \sqrt{C} \tag{since $\eta \le 1$}
\end{align*}
After completing the square w.r.t $\|Q - P\|_1$ we obtain:
\begin{align*}
    \|Q - P\|_1 \le 4 \, \eta \, \tau \, \ln(\frac{A}{C}) + 2 \, \sqrt{\tau} \, C^{\frac{1}{4}}
\end{align*}
\end{proof}

\label{lemma:prob-difference}
\end{lemma}

\begin{lemma}
For all $k \in [K]$,  
\begin{align*}
\sum_{i = 1}^k  \frac{\gamma^{k-i}}{(i+1)} \leq \frac{2}{k \, (1 - \gamma)} + \frac{\gamma^{k/2}}{1 - \gamma}   
\end{align*}
\label{lemma:sequence-sum}
\end{lemma}
\begin{proof}
Define $j = k - i$, in which case, we need to bound, 
\begin{align*}
\sum_{j = 0}^{k-1} \frac{\gamma^{j}}{k - j + 1} & = \underbrace{\sum_{j = 0}^{\floor{k/2}} \frac{\gamma^{j}}{k - j + 1}}_{\text{Term (i)}} + \underbrace{\sum_{j = \floor{k/2}+1}^{k-1} \frac{\gamma^{j}}{k - j + 1}}_{\text{Term (ii)}}
\end{align*}  
For bounding Term (i), note that $j \leq \frac{k}{2} + 1$, meaning that $k - j +1 \geq \frac{k}{2}$. Hence, 
\begin{align*}
\text{Term (i)} &= \sum_{j = 0}^{\floor{k/2}} \frac{\gamma^{j}}{k - j + 1}  \leq \frac{2}{k} \, \sum_{j = 0}^{\floor{k/2}} \gamma^{j} \leq \frac{2}{k} \, \sum_{j = 0}^{\infty} \gamma^{j} \leq \frac{2}{k} \, \frac{1}{1 - \gamma}
\end{align*}
For bounding Term (ii), note that since $j \leq k$, $k - j + 1 \geq 1$. Hence, 
\begin{align*}
\text{Term (ii)} &= \sum_{j = \floor{k/2}+1}^{k-1} \frac{\gamma^{j}}{k - j + 1} \leq \sum_{j = \floor{k/2}+1}^{k-1} \gamma^j \leq \sum_{j = \floor{k/2}+1}^{\infty} \gamma^j \leq \frac{\gamma^{k/2}}{1 - \gamma}
\end{align*}
Combining both terms, 
\begin{align*}
\sum_{j = 0}^{k-1} \frac{\gamma^{j}}{k - j + 1} & \leq \frac{2}{k} \, \frac{1}{1 - \gamma} + \frac{\gamma^{k/2}}{1 - \gamma} 
\end{align*}
\end{proof}

\begin{lemma}
For any state-action value function $q$, for any $m \geq 1$,  
\begin{align*}
\vert (T_\tau^{\pi_{1}} q)^m(s,a) - (T_\zeta^{\pi_{1}} q)^m(s,a) \vert & \leq \frac{|\tau- \zeta|}{1-\gamma}\ln{(A)} 
\end{align*}
\label{lemma:bellman-difference}
\end{lemma}
\begin{proof}
\begin{align}
    (T_\tau^{\pi_{1}} q)^m(s,a)
    &= \E_{\substack{
    s_{t+1} \sim \cP(\cdot|s_{t}, a_{t}) \\
    a_{t} \sim \pi_1(\cdot|s_t)}} \Bigg[
    \sum_{t=0}^{m-1} \gamma^t \big(
        r(s_t,a_t) - \tau \ln \pi_1(a_t|s_t)
    \big) \\
    &\hspace{2.7cm} + \gamma^m q(s_m,a_m) 
    \,\Big|\, s_0 = s,\, a_0 = a
    \Bigg]
    \tag{By definition of $T_\tau^{\pi_{1}}$} \\
     & = \E \left[\sum_{t=0}^{m-1} \gamma^t(r(s_t,a_t)) + \gamma^m q(s_m,a_m) \right]
     - \tau \, \E \left[\sum_{t=0}^{m-1} \gamma^t \ln{(\pi_1(a_t|s_t))} \right]\\
     &\quad + \zeta \, \E \left[\sum_{t=0}^{m-1} \gamma^t \ln{(\pi_1(a_t|s_t))} \right] - \zeta \, \E \left[\sum_{t=0}^{m-1} \gamma^t \ln{(\pi_1(a_t|s_t))} \right] \tag{Add/Subtract $\zeta \, \E \left[\sum_{t=0}^{m-1} \gamma^t \ln{(\pi_1(a_t|s_t))} \right]$} \\
    & = (T_\zeta^{\pi_{1}} q)^m(s,a) - (\tau - \zeta) \, \E \left[\sum_{t=0}^{m-1} \gamma^t \, \ln{(\pi_1(a_t|s_t))} \right] \tag{By definition of $T_\zeta^{\pi_{1}}$} \\
    & = (T_\zeta^{\pi_{1}} q)^m(s,a) + (\tau - \zeta) \, \E \left[\sum_{t=0}^{m-1} \gamma^t \, \cH{(\pi_1(.|s_t))}\right] \tag{By definition of $ \cH{(\pi_1(.|s_t))}$}\\
    & \leq (T_\zeta^{\pi_{1}} q)^m(s,a)+ \frac{|\tau- \zeta|}{1-\gamma}\ln{(A)} \tag{Since $\E[\cH{(\pi_1(.|s_t))}] \leq \ln(A)$}
\label{eq:diff-soft}
\end{align}    
\end{proof}

\clearpage
\section{Experimental Details and Additional Results}
\label{app:experimental-details-and-additional-results}

\subsection{Details for Stable Baselines Experiments}
\label{app:stable-baselines-details}
Following~\citet{tomar2020mirror, asad2024fast} we use the default hyperparameters from \texttt{stable-baselines3}~\citep{stable-baselines3} for each method. This choice is motivated by prior work, which focuses on evaluating the effectiveness of different surrogate losses rather than performing exhaustive hyperparameter searches. Such searches are particularly impractical for CNN-based actor and critic networks, where tuning multiple hyperparameters (e.g., framestack, buffer size, minibatch size, discount factor) is computationally expensive. The full list of hyperparameters for the Atari experiments is provided in~\cref{tab:atari-hyperparams}. 

For our forward and reverse KL objectives, we set $\etat$ to a constant chosen via a small grid search over $[0.01, 0.1, 1.0]$. Following prior work~\citep{haarnoja2018soft,christodoulou2019soft,mnih2013playing} for policy evaluation, $q_\zeta^{t-1}$ in~\cref{eq:pe-fa} is parameterized using a separate target network, whose parameters are updated via an exponential moving average of those in the critic model. For on-policy $\PPO$, we adopt the optimal hyperparameters reported in RL Baselines3 Zoo~\citep{rl-zoo3}. For all experiments, we use $n=1$ actor gradient updates unless otherwise stated. 

Regarding computational resources, each Atari run was executed on Digital Research Alliance of Canada's Narval cluster using one GPU with 12 CPU cores and 64GB of memory; runs with $n=1$ actor-gradient update took approximately 12 hours, while runs with $n=10$ actor-gradient updates took approximately 32 hours. 

\begin{table}[ht]
\centering
\resizebox{0.9\textwidth}{!}{
\begin{tabular}{|l|c|c|c|c|}
\hline
\textbf{Hyperparameter} & \textbf{FKL Objectives} & \textbf{RKL Objectives} & \textbf{$\DSAC$} & \textbf{$\DQN$} \\ 
\hline
Reward normalization & \xmark & \xmark & \xmark & \xmark \\
Observation normalization & \xmark & \xmark & \xmark & \xmark \\
Orthogonal weight initialization & \xmark & \xmark & \xmark & \xmark \\
Value function clipping & \xmark & \xmark & \xmark & \xmark \\
Gradient clipping & \xmark & \xmark & \xmark & \xmark  \\
Probability ratio clipping & \xmark & \xmark & \xmark & \xmark \\
Entropy coefficient & auto & auto & auto & $\epsilon$-greedy\\
\hline
Adam step-size & \multicolumn{4}{c|}{$3\times10^{-4}$} \\
Buffer size & \multicolumn{4}{c|}{$10^6$} \\
Minibatch size & \multicolumn{4}{c|}{256} \\
Framestack & \multicolumn{4}{c|}{4} \\
Number of environment copies & \multicolumn{4}{c|}{8} \\
Discount factor & \multicolumn{4}{c|}{0.99} \\
Total number of timesteps & \multicolumn{4}{c|}{$10^7$} \\
Number of runs for plot averages & \multicolumn{4}{c|}{5} \\
Confidence interval for plot runs & \multicolumn{4}{c|}{$\sim 95\%$} \\
\hline
\end{tabular}
}
\caption{\centering Hyper-parameters for Atari and Retro experiments.}
\label{tab:atari-hyperparams}
\end{table}



\clearpage
\subsection{\texorpdfstring{Comparison to $\GreedyAC$}{Comparison to GreedyAC}}
\label{app:greedy-ac-comparison}
We examine the empirical finding in our $\DSAC$ ablation in~\cref{sec:experiments} (~\cref{fig:dsac-ablation-m-1}) on discrete classic control tasks. The results in Figures~\ref{fig:dsac-greedyac-32} and~\ref{fig:dsac-greedyac-256} support our observation: $\DSAC$ with a fixed, per-environment tuned entropy coefficient, with or without critic entropy, attains strong performance on par with the prior discrete actor-critic framework $\GreedyAC$~\citep{neumann2018greedy}. 

\begin{figure*}[!ht]
\centering
\includegraphics[width=\textwidth]{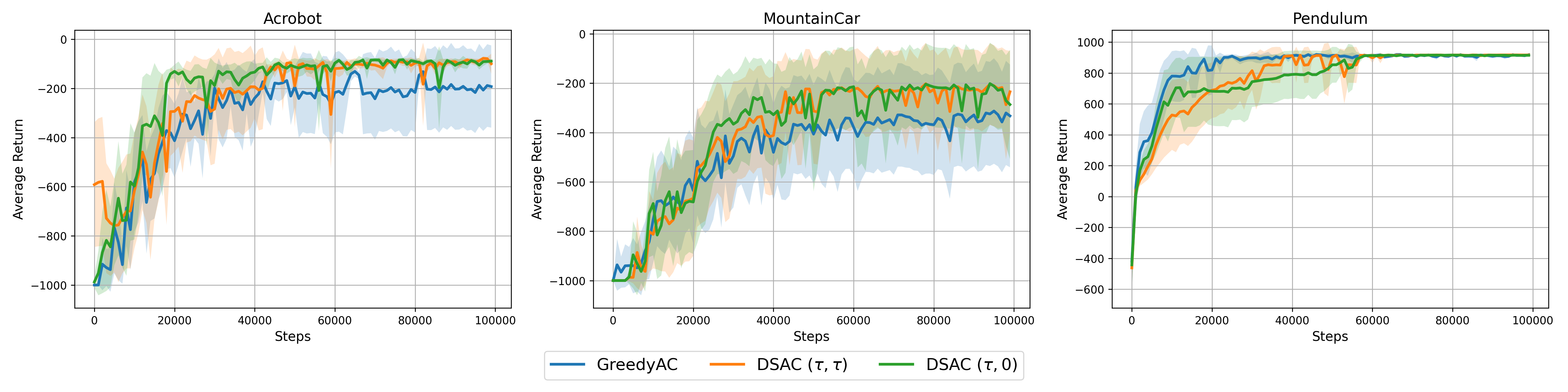}
\caption{Comparing $\GreedyAC$ with the default $\DSAC$ (with and without critic entropy and a fixed entropy coefficient) using batch size 32.}
\label{fig:dsac-greedyac-32}
\end{figure*}

\begin{figure*}[!ht]
\centering
\includegraphics[width=\textwidth]{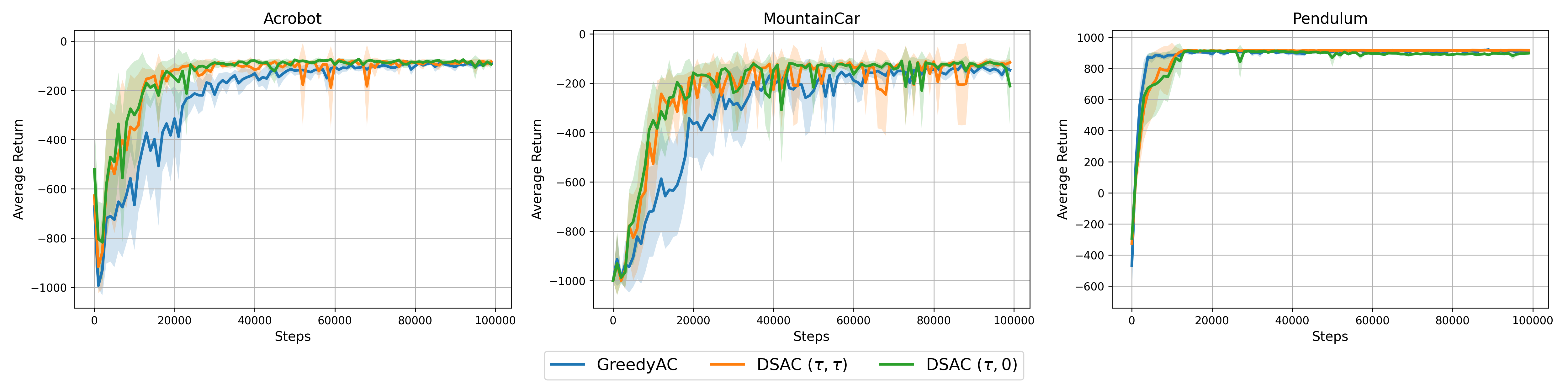}
\caption{Comparing $\GreedyAC$ with the default $\DSAC$ (with and without critic entropy and a fixed entropy coefficient) using batch size 256.}
\label{fig:dsac-greedyac-256}
\end{figure*}

\clearpage
\subsection{\texorpdfstring{Does Double $Q$ Learning Hurt or Improve Performance?}{Does Double Q Learning Hurt or Improve Performance?}}
\label{app:double-q-results}
\citet{zhou2022revisiting} hypothesize that $\DSAC$ suffers from two main failure modes. First, the coupled learning of the actor and critic can induce fluctuations in the $Q$-function distribution and policy entropy, leading to a sharp policy and poor performance. Second, taking the minimum of two $Q$ functions in double $Q$-learning can induce $Q$-value underestimation and pessimistic exploration. To address these issues, they propose \texttt{SD-SAC}, which adds an entropy-penalty regularizer to prevent the policy entropy from deviating too much from its value at the previous iterate, together with clipped double-average $Q$-learning to stabilize the critic learning. Crucially, the reported \texttt{SD-SAC} experiments remove the automatic entropy tuning of the original $\DSAC$ and require tuning a fixed entropy coefficient. On the other hand, our work identifies critic entropy as the main bottleneck limiting $\DSAC$'s efficacy and proposes $\DSAC(\tau,0)$, which substantially outperforms the default $\DSAC$ while retaining the automatic entropy tuning of the original method.

The results in~\cref{tab:normalized_score_double_q_sd_sac} show that, unlike the findings of~\citet{zhou2022revisiting} for default $\DSAC$, taking the minimum of two $Q$ functions does not harm $\DSAC(\tau,0)$, even without $Q$ clipping. Furthermore, both the single-$Q$ and double-$Q$ versions of $\DSAC(\tau,0)$ outperform \texttt{SD-SAC}.

\begin{table}[!htbp]
\centering
\caption{IQM summary of final human-normalized scores with 95\% stratified bootstrap confidence intervals: prior work hypothesizes that taking the minimum of two $Q$-networks harms $\DSAC$ and proposes \texttt{SD-SAC}, which clips and averages the two $Q$-values and uses an entropy penalty to prevent the collapse of policy entropy. However, the results show that the performance of our $\DSAC(\tau, 0)$ with adaptive actor entropy does not degrade when taking the minimum of two $Q$-networks, and that both its single-$Q$ and double-$Q$ variants substantially outperform \texttt{SD-SAC} without clipping.}
\label{tab:normalized_score_double_q_sd_sac}
\begin{tabular}{lccc}
\toprule
Method & IQM & CI Low & CI High \\
\midrule
DSAC Double-Q $(\tau,0)$ & \textbf{0.9487} & 0.8184 & 1.0789 \\
DSAC $(\tau,0)$ & 0.9226 & 0.8228 & 1.0288 \\
SD-SAC (Baseline) & 0.7605          & 0.6728 & 0.8545 \\
\bottomrule
\end{tabular}
\end{table}

\begin{figure*}[!ht]
\centering
\includegraphics[width=0.99\textwidth]{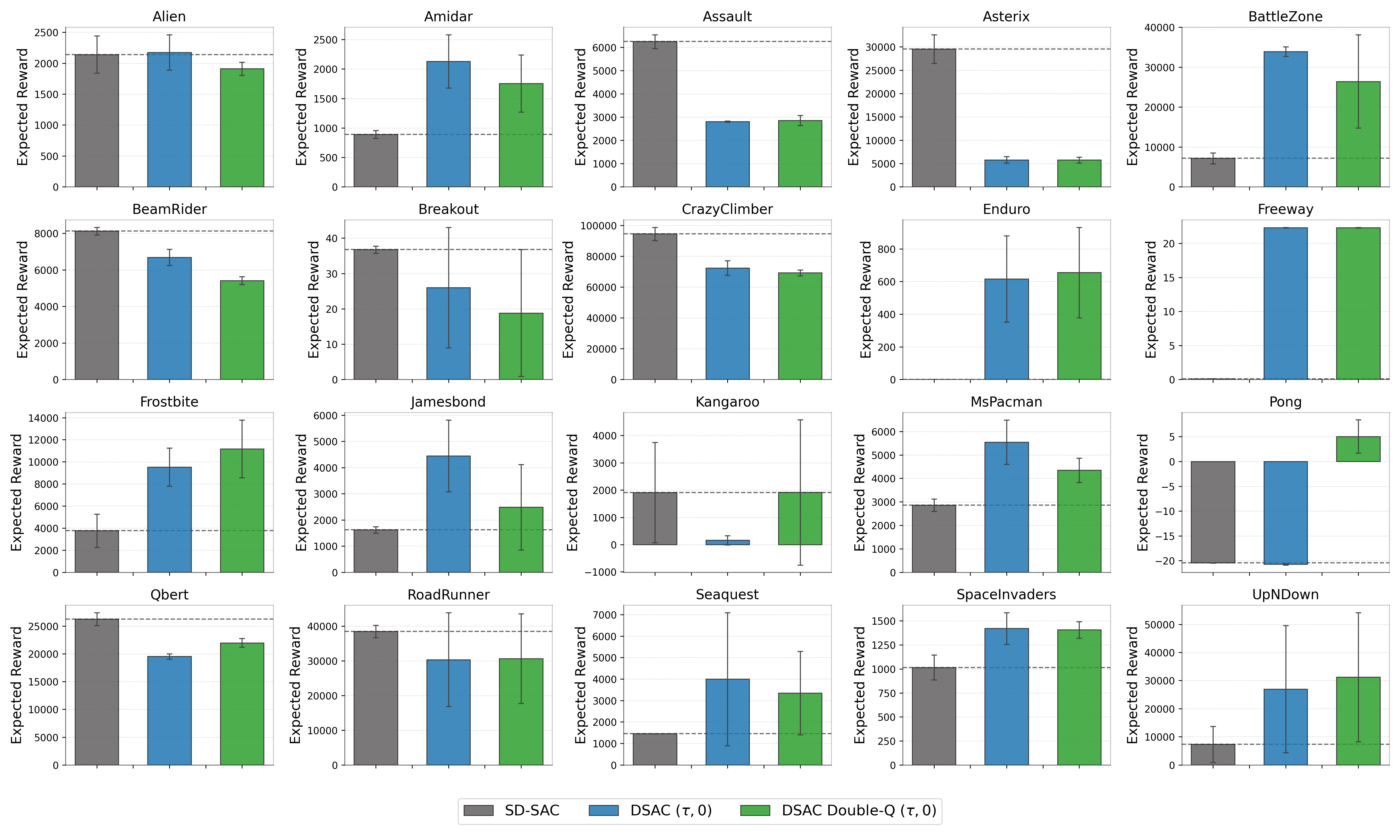}
\caption{Expected reward after 10M steps on each of the 20 Atari games, comparing our $\DSAC(\tau, 0)$ and $\DSAC$ Double-$Q$ $(\tau, 0)$ against $\texttt{SD-SAC}$. While the table above summarizes aggregate performance across games using IQM, this figure shows the per-game final rewards.}
\label{fig:dsac-double-q-ablation-m-1}
\end{figure*}

\clearpage
\subsection{\texorpdfstring{Is Fixed Target Entropy the Main Bottleneck in $\DSAC$?}{Is Fixed Target Entropy the Main Bottleneck in DSAC?}}
\label{app:adaptive-target-ent-results}

\begin{table}[!htbp]
\centering
\caption{IQM summary of final human-normalized scores with 95\% stratified bootstrap confidence intervals: prior work~\citep{xu2021target} hypothesizes that using a fixed target entropy in the adaptive entropy-coefficient scheme of~\citet{christodoulou2019soft} harms $\DSAC$, and proposes \texttt{TES-SAC}, which instead uses an adaptive target entropy. The results show that our $\DSAC(\tau, 0)$, despite using a fixed target entropy, substantially outperforms \texttt{TES-SAC}.}
\label{tab:normalized_score_adaptive_target_ent}
\begin{tabular}{lccc}
\toprule
Method & IQM & CI Low & CI High \\
\midrule
DSAC $(\tau,0)$ & \textbf{0.9226} & 0.8228 & 1.0288 \\
TES-SAC (Baseline) & 0.7210 & 0.6390 & 0.7965 \\
\bottomrule
\end{tabular}
\end{table}

\begin{figure*}[!ht]
\centering
\includegraphics[width=\textwidth]{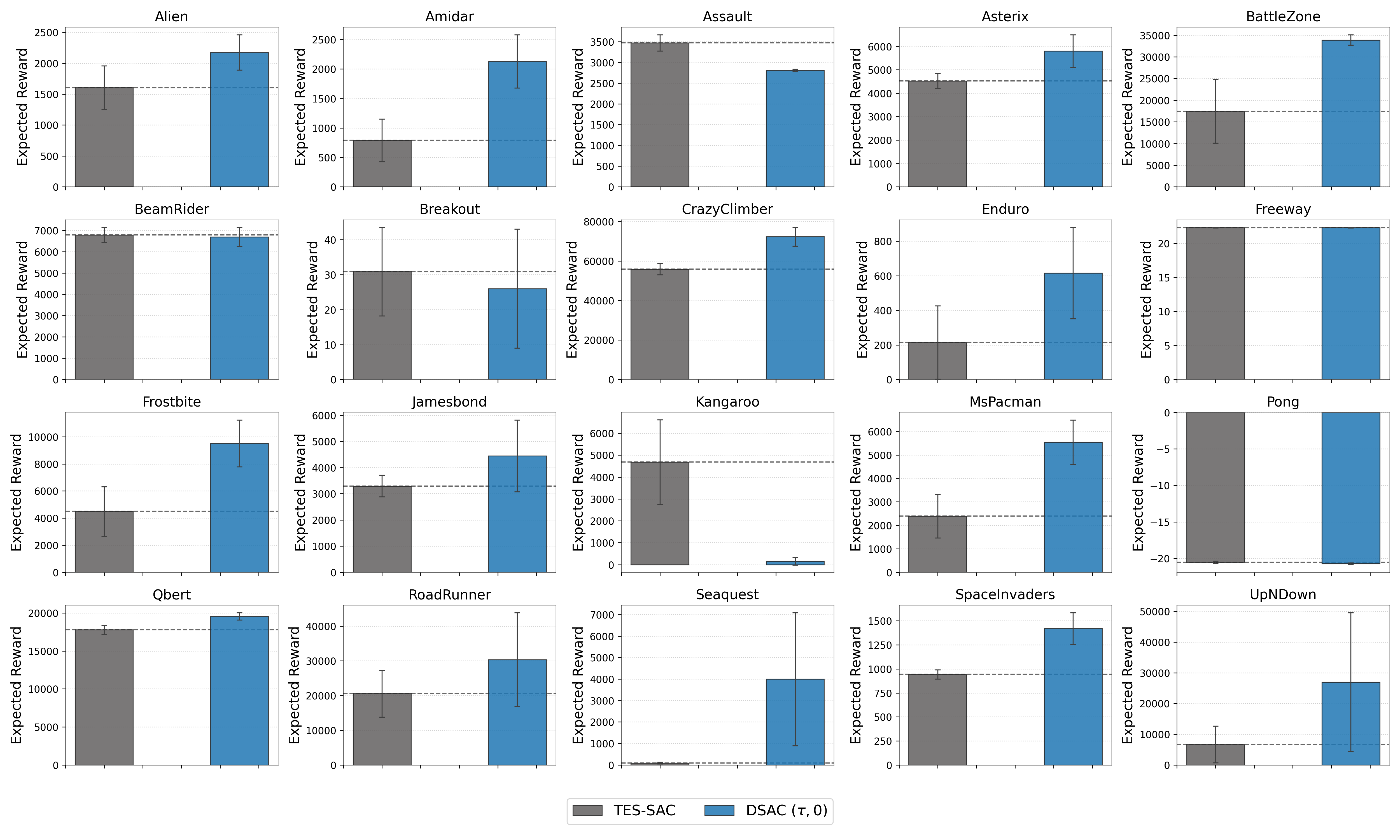}
\caption{Expected reward after 10M steps on each of the 20 Atari games, comparing our $\DSAC(\tau, 0)$ against $\texttt{TES-SAC}$. While the table above summarizes aggregate performance across games using IQM, this figure shows the per-game final rewards.}
\label{fig:dsac-adaptive-target-ent-ablation-m-1}
\end{figure*}


\clearpage
\subsection{\texorpdfstring{Evaluating our $\DSAC$ Variant on Hard Exploration Games}{Evaluating our DSAC Variant on Hard Exploration Games}}
\label{app-hard-exploration}
Table 1 in~\citep{taiga2021bonus} provides a taxonomy of Atari games and identifies hard exploration games, further dividing them into dense- and sparse-reward settings. We compare our $\DSAC(\tau, 0)$ with adaptive entropy against the baselines on 10 hard exploration games, comprising 5 dense-reward games (Alien, Amidar, Frostbite, MsPacman, Qbert) and 5 sparse-reward games (Freeway, Gravitar, Pitfall, Solaris, Venture). The results in~\cref{tab:normalized_score_hard_exploration_summary} show that, when evaluation is restricted to these hard exploration games, $\DSAC(\tau, 0)$ substantially outperforms the baselines. The per-game final expected rewards are provided in~\cref{fig:hard-exploration}.

\begin{table}[!htbp]
\centering
\caption{IQM summary of final human-normalized scores with 95\% stratified bootstrap confidence intervals: comparing $\DSAC(\tau,0)$, which uses adaptive actor entropy and disables critic entropy, against baselines on hard-exploration games with dense and sparse reward structures. Our $\DSAC$ variant achieves significantly higher overall performance than the baselines.}
\label{tab:normalized_score_hard_exploration_summary}
\begin{tabular}{lccc}
\toprule
Method & IQM & CI Low & CI High \\
\midrule
DSAC $(\tau,0)$            & \textbf{0.6951} & 0.6497 & 0.7413 \\
Adam LMCDQN (Baseline)     & 0.5862          & 0.5470 & 0.6299 \\
DQN (Baseline)             & 0.4742          & 0.4558 & 0.4958 \\
PPO (Baseline)             & 0.3225          & 0.2664 & 0.3825 \\
TES-SAC (Baseline)         & 0.3022          & 0.2497 & 0.3485 \\
SD-SAC (Baseline)          & 0.2053          & 0.1928 & 0.2188 \\
\bottomrule
\end{tabular}
\end{table}

\begin{figure*}[!ht]
\centering
\includegraphics[width=\textwidth]{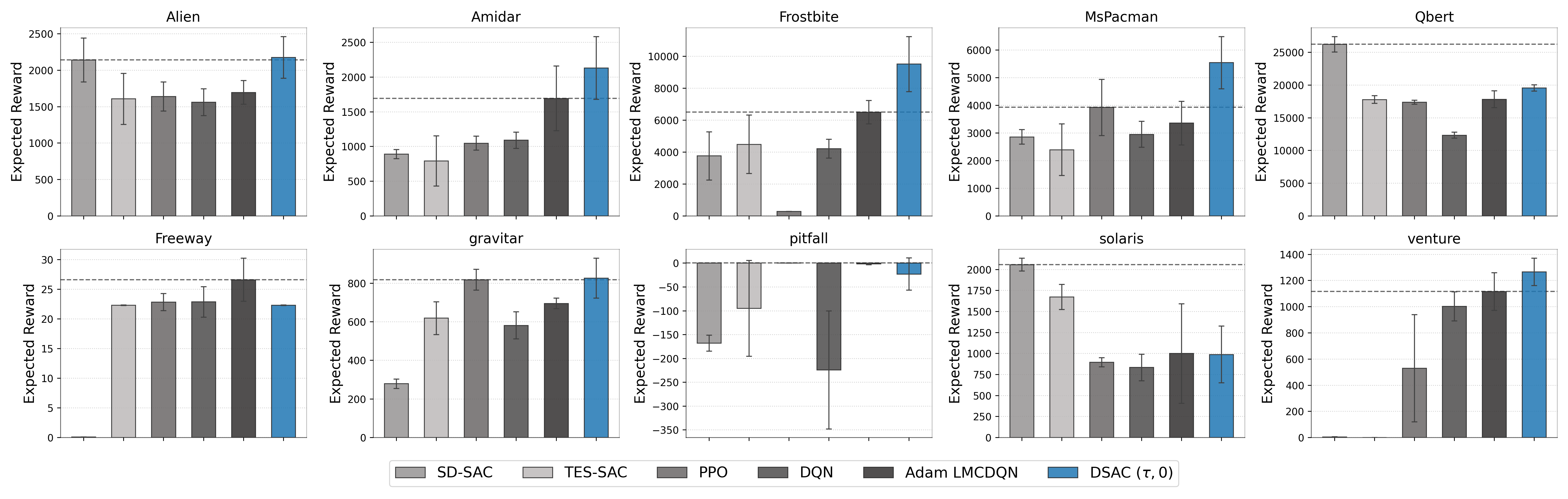}
\caption{Expected reward after 10M steps on each of the 10 hard exploration games, comparing our $\DSAC$ variant with adaptive actor entropy and no critic entropy against baselines (grey). The table above summarizes aggregate performance using IQM, while this figure shows per-game final performance.}
\label{fig:hard-exploration}
\end{figure*}

\clearpage
\subsection{\texorpdfstring{Evaluating our $\DSAC$ Variant on \texttt{stable-retro} games}{Evaluating our DSAC Variant on stable-retro games}}
\label{app:dsac-ablation-stable-retro}
In addition to the standard Atari benchmark, we test our hypothesis on the role of critic entropy in $\DSAC$ using \texttt{stable-retro} games~\citep{stable-retro} with diverse reward structures and exploration levels. Importantly, \texttt{stable-retro} has a substantially larger action space than Atari, ranging from 18--468 actions compared to 4--18. The results in~\cref{fig:dsac-stable-retro-ablations} show that our $\DSAC$ variant, which disables critic entropy while retaining automatic entropy tuning for the actor, yields substantial performance gains.

\begin{figure*}[!ht]
\centering
\includegraphics[width=0.8\textwidth]{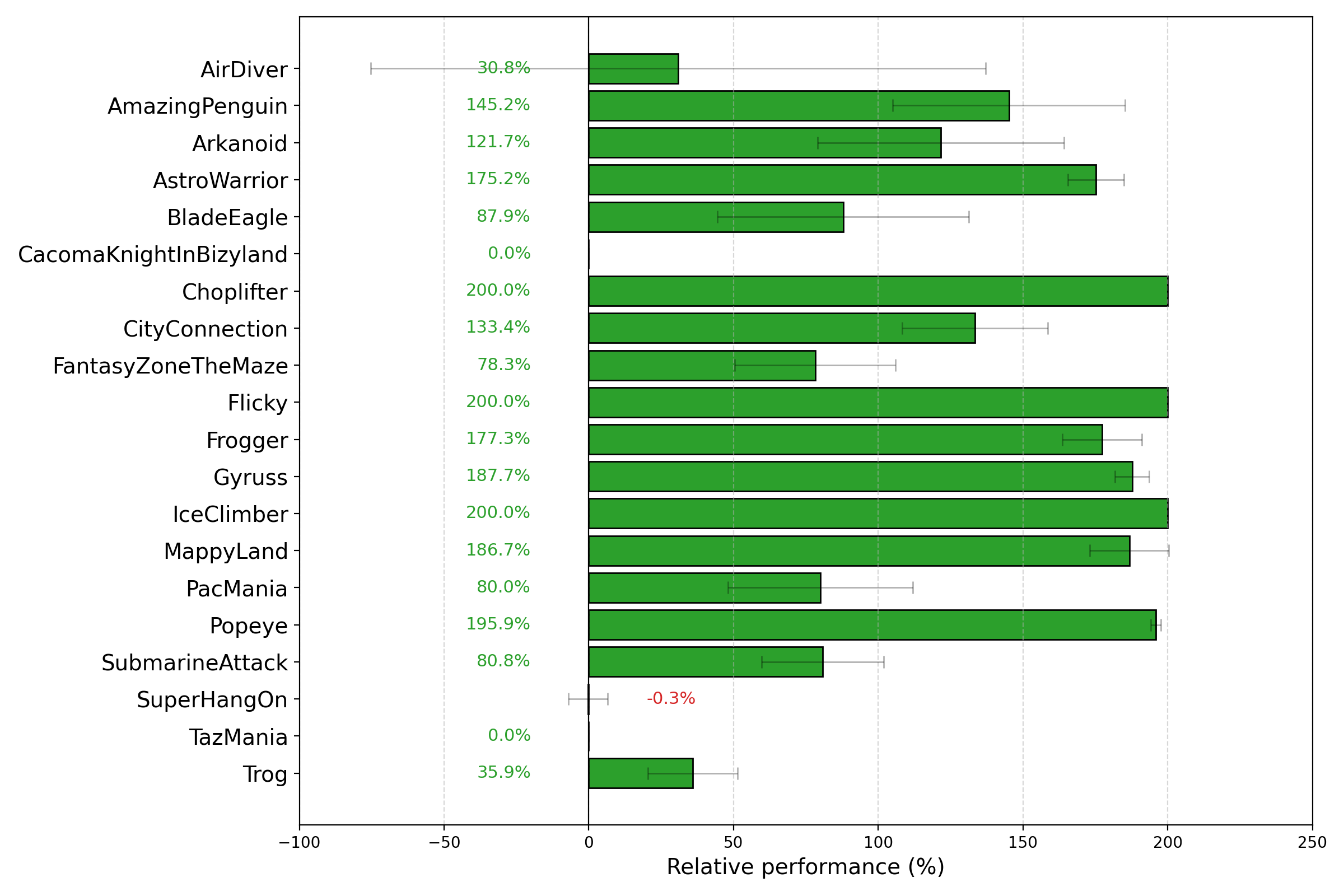}
\caption{Relative performance of $\DSAC(\tau,0)$, which uses adaptive actor entropy and disables critic entropy (our contribution), compared to the default $\DSAC$ across 20 \texttt{stable-retro} games. Results are averaged over 5 seeds with 95\% confidence intervals. Consistent with the Atari results, removing critic entropy yields substantial performance gains.}
\label{fig:dsac-stable-retro-ablations}
\end{figure*}

\clearpage
\subsection{Our Objectives Benefit Similarly from the Hard Bellman Operator}
\label{app:soft-vs-hard-bellman-results}
When using the forward and reverse KL-based objectives in our framework, we consistently observe the same trend as with $\DSAC$: As observed in~\cref{fig:all-objectives-with-without-critic-ent-20-games-m-1}, employing the hard Bellman operator ($\zeta=0$) with an adaptive entropy coefficient leads to substantial performance improvement compared to using the soft Bellman operator ($\zeta=\tau$). 

\begin{figure*}[!ht]
\centering
\includegraphics[width=0.85\textwidth]{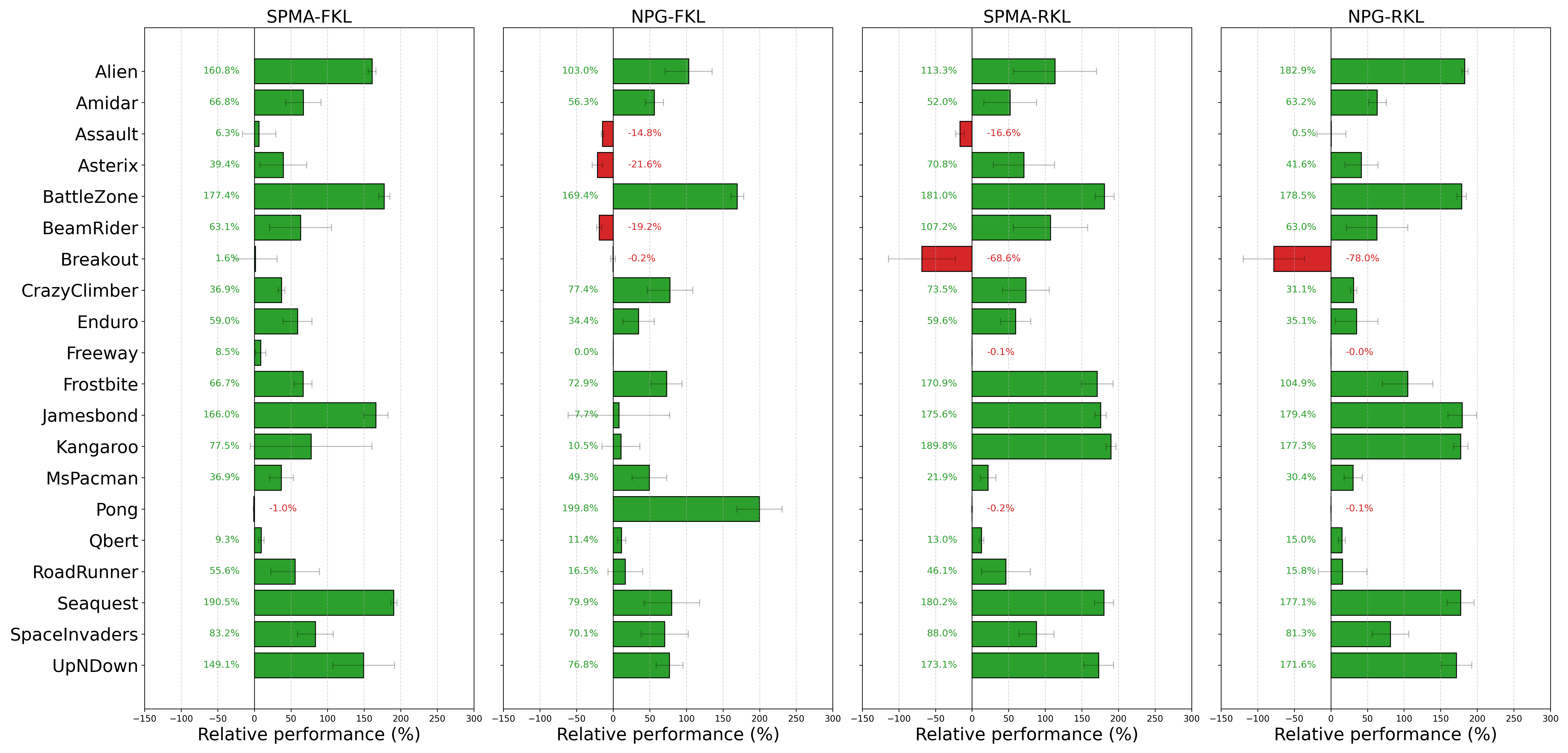}
\caption{Relative performance of forward and reverse KL-based objectives when critic entropy is disabled, measured against the same objectives with critic entropy enabled. Disabling critic entropy while retaining the adaptive entropy coefficient loss, as in $\DSAC$, leads to performance improvements for all objectives.}
\label{fig:all-objectives-with-without-critic-ent-20-games-m-1}
\end{figure*}

While choosing $\zeta=0$ is practical and avoids per-game hyperparameter search, we also consider an intermediate variant of our $\DSAC$ in which $\zeta$ is a nonzero constant selected from \{0.001, 0.01, 0.1, 0.5\}. The results in~\cref{tab:normalized_score_const_zeta_ablation} indicate that tuning a constant nonzero $\zeta$ can improve performance, but only at the cost of per-game hyperparameter search. Examining the distribution of the best $\zeta$ values across games, we find that the optimal choice varies by game, with the highest concentration at 0.01 and 0.001 (~\cref{fig:zeta-dist}). We therefore recommend using $\zeta=0$, which avoids per-game hyperparameter search.

\begin{table}[!htbp]
\centering
\caption{IQM summary of final human-normalized scores with 95\% stratified bootstrap confidence intervals: comparing $\DSAC(\tau, 0)$ (adaptive actor entropy and critic entropy disabled) vs. $\DSAC(\tau, \text{const-}\zeta)$ (adaptive actor entropy and constant non-zero critic entropy).}
\label{tab:normalized_score_const_zeta_ablation}
\begin{tabular}{lccc}
\toprule
Method & IQM & CI Low & CI High \\
\midrule
DSAC $(\tau, \text{const-}\zeta)$ & \textbf{0.9495} & 0.8318 & 1.0707 \\
DSAC $(\tau,0)$ & 0.9226 & 0.8228 & 1.0288 \\
\bottomrule
\end{tabular}
\end{table}

\begin{figure*}[!ht]
\centering
\includegraphics[width=0.4\textwidth]{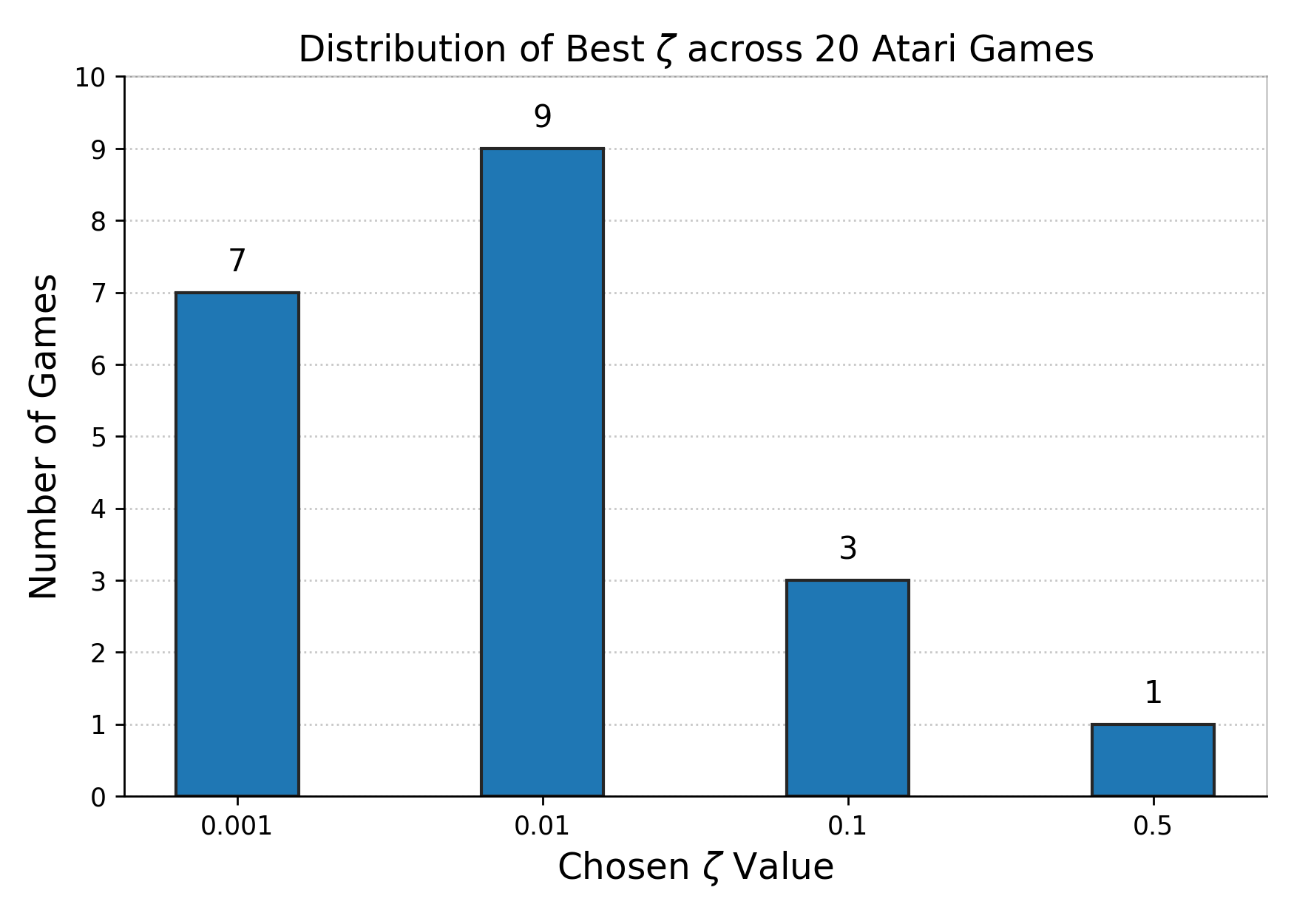}
\caption{Distribution of the best nonzero constant $\zeta$ across games for $\DSAC(\tau,\text{const-}\zeta)$.}
\label{fig:zeta-dist}
\end{figure*}

\clearpage
\subsection{Additional Results: Is Entropy Regularization Necessary?}
\label{app:ent-reg-kl-ablation}

\begin{figure*}[!ht]
\centering
\includegraphics[width=\textwidth]{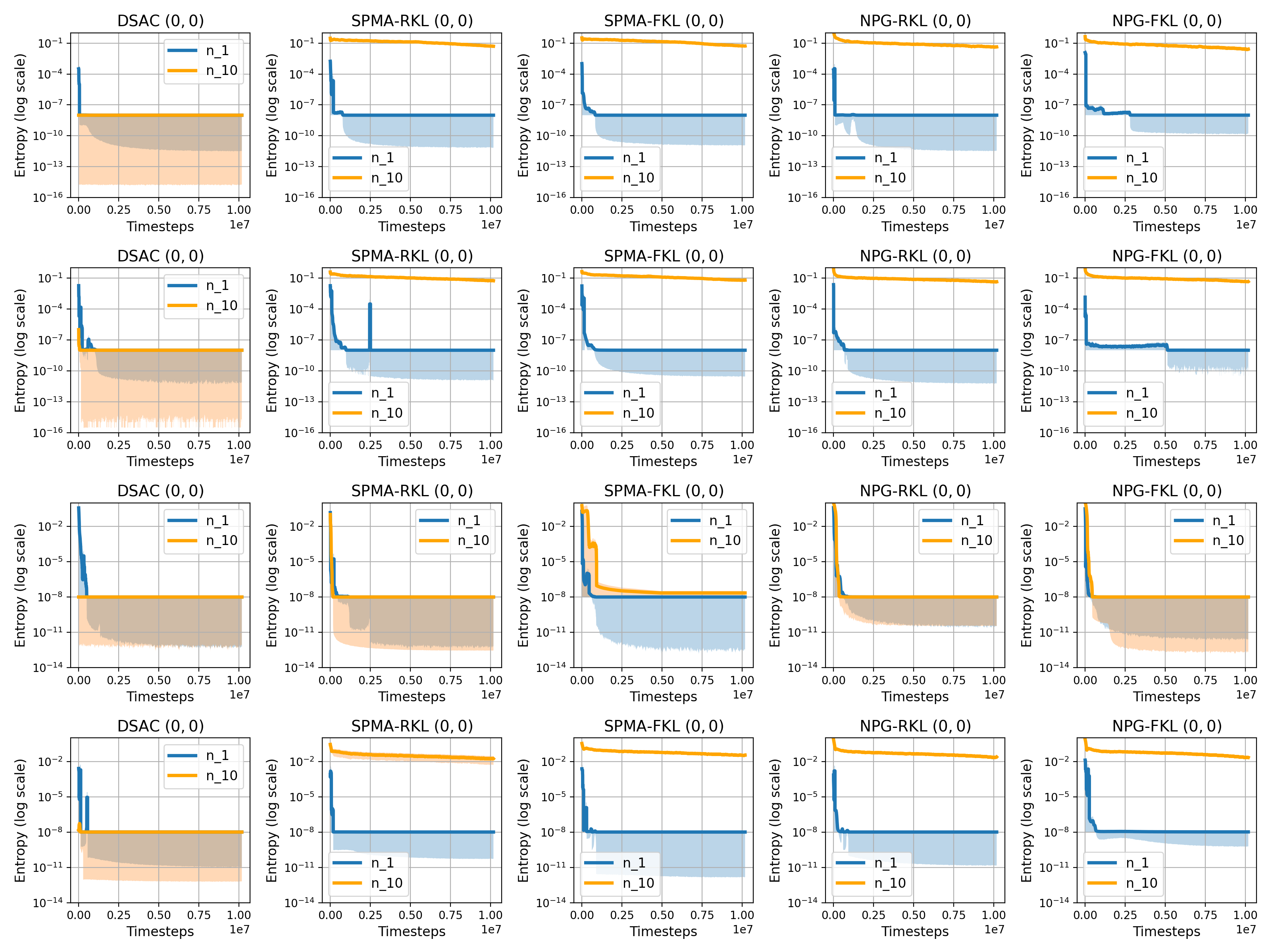}
\caption{\textbf{Policy Shannon entropy during training.} Without entropy regularization, increasing $n$ from 1 to 10 substantially increases and stabilizes the Shannon entropy for our forward and reverse KL-based methods, compared to $\DSAC$, across four Atari games. The rows correspond to MsPacman-v4, Seaquest-v4, Freeway-v4, and Breakout-v4.}
\label{fig:all-objectives-no-ent-4-games-ent-metric-m-1-10}
\end{figure*}

\begin{figure*}[!ht]
\centering
\includegraphics[width=\textwidth]{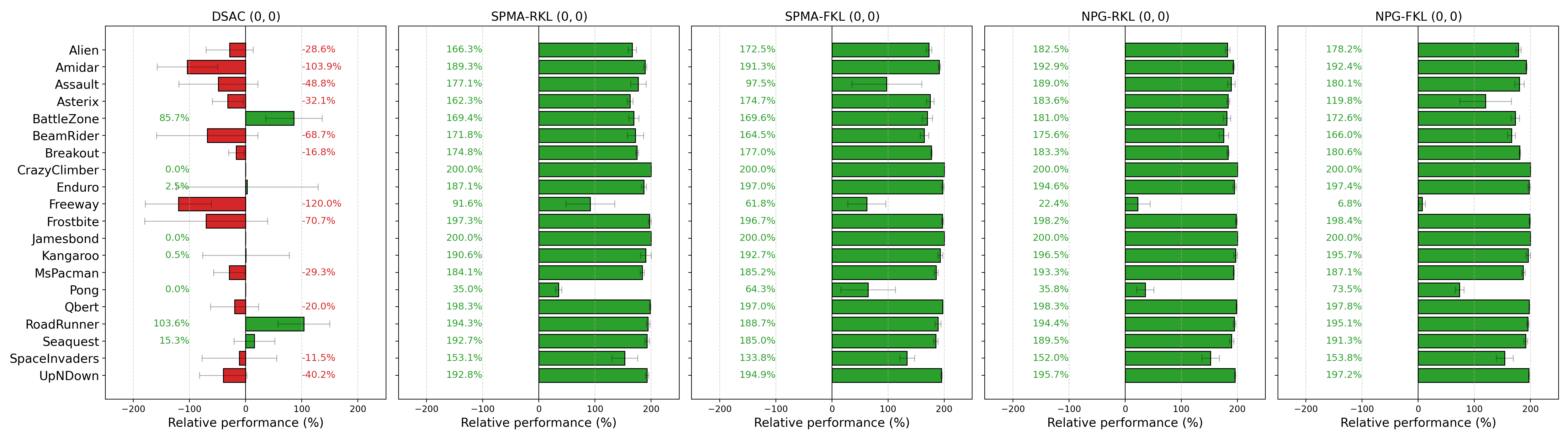}
\caption{Effect of the number of actor optimization steps $n$ without entropy regularization. Unlike $\DSAC$, the performance of all objectives derived from our framework significantly boosts when $n$ is increased from 1 to 10.}
\label{fig:all-objectives-no-ent-20-games-m-1}
\end{figure*}

\begin{figure*}[!ht]
\centering
\includegraphics[width=\textwidth]{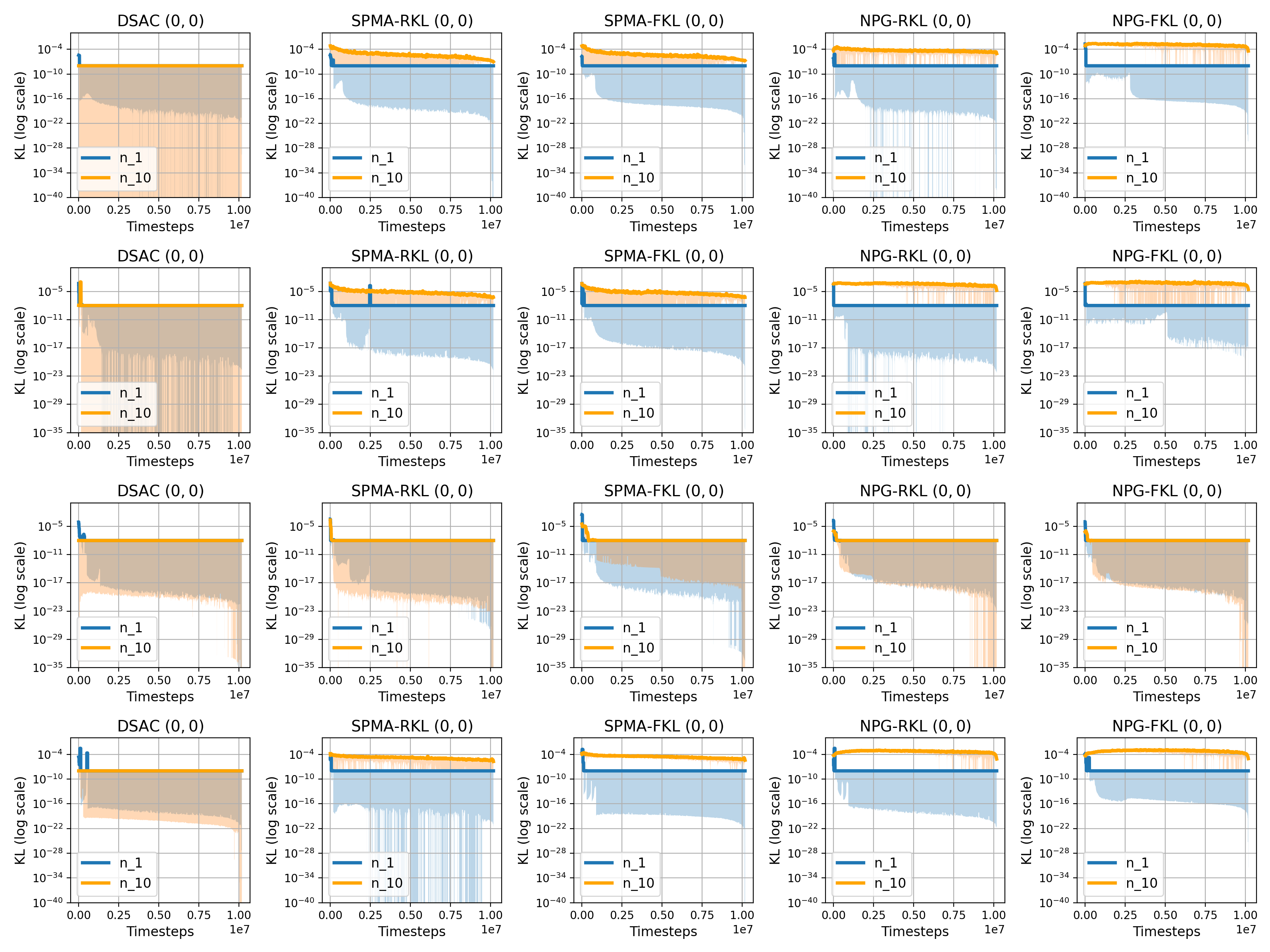}
\caption{\textbf{KL divergence during training.} Consistent with the Shannon entropy results, in the absence of entropy regularization, increasing $n$ to 10 allows our forward- and reverse-KL-based actors to achieve higher KL divergence than $\DSAC$, which also stabilizes over the course of training. The rows correspond to MsPacman-v4, Seaquest-v4, Freeway-v4, and Breakout-v4.}
\label{fig:all-objectives-no-ent-4-games-kl-metric-m-1-10}
\end{figure*}

\begin{figure*}[!ht]
\centering
\includegraphics[width=\textwidth]{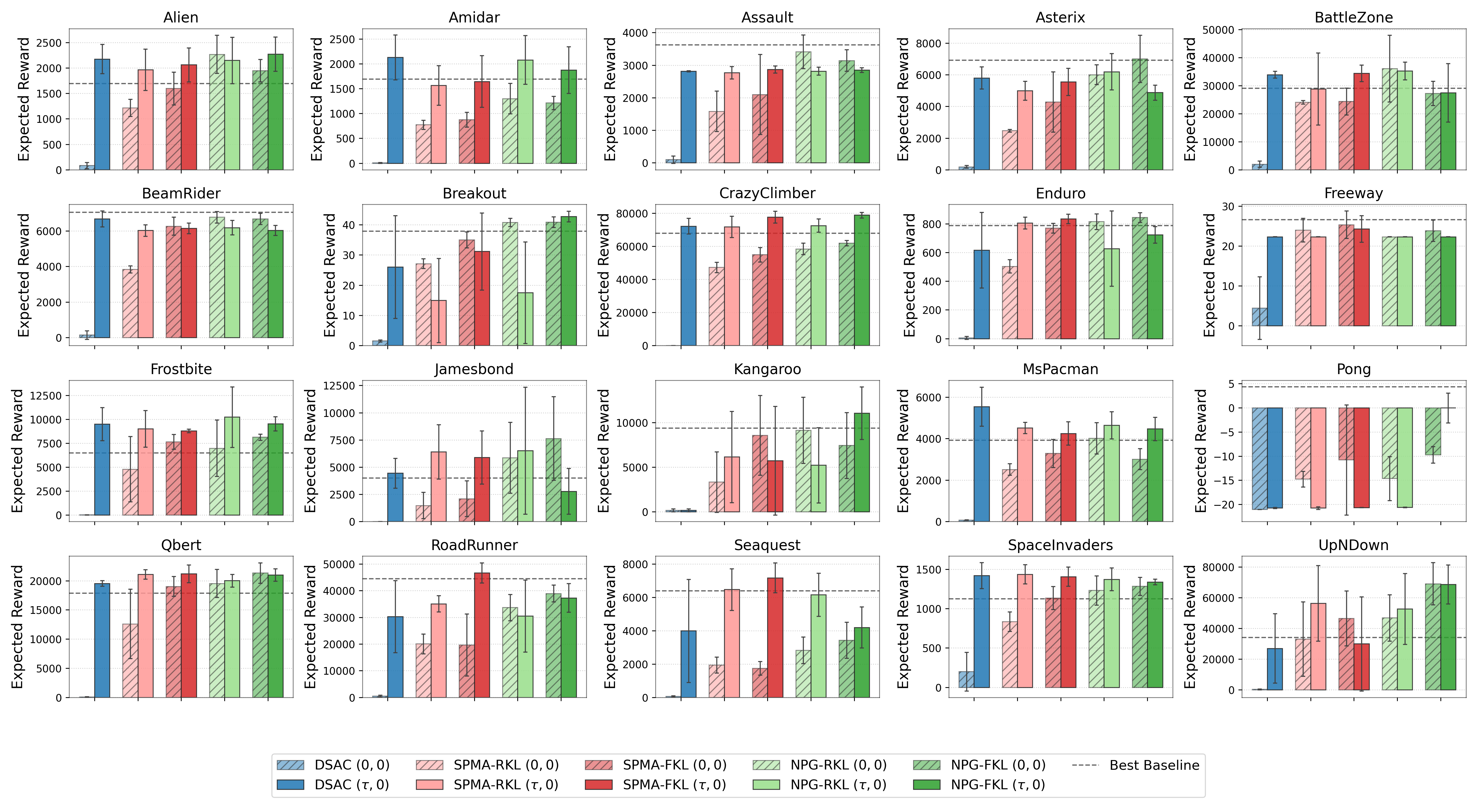}
\caption{\textbf{Entropy regularization ablation: disabling entropy.} Expected reward after 10M steps. Unlike $\DSAC$, objectives from our framework, including the novel $\SPMARKL$ and $\NPGFKL$, remain robust with both actor and critic entropy disabled (dashed bars); solid bars denote the actor-entropy-enabled setting $(\tau \ne 0)$. The dashed line denotes the best baseline among $\PPO$, $\DQN$, and $\LMC$.}
\label{fig:spma-vs-npg-vs-baselines-no-ent-m-10-12-games}
\end{figure*}





\clearpage
\subsection{Expected reward after 10M steps}
\label{app:20-games-final-performance}

\begin{figure*}[!ht]
\centering
\includegraphics[width=\textwidth]{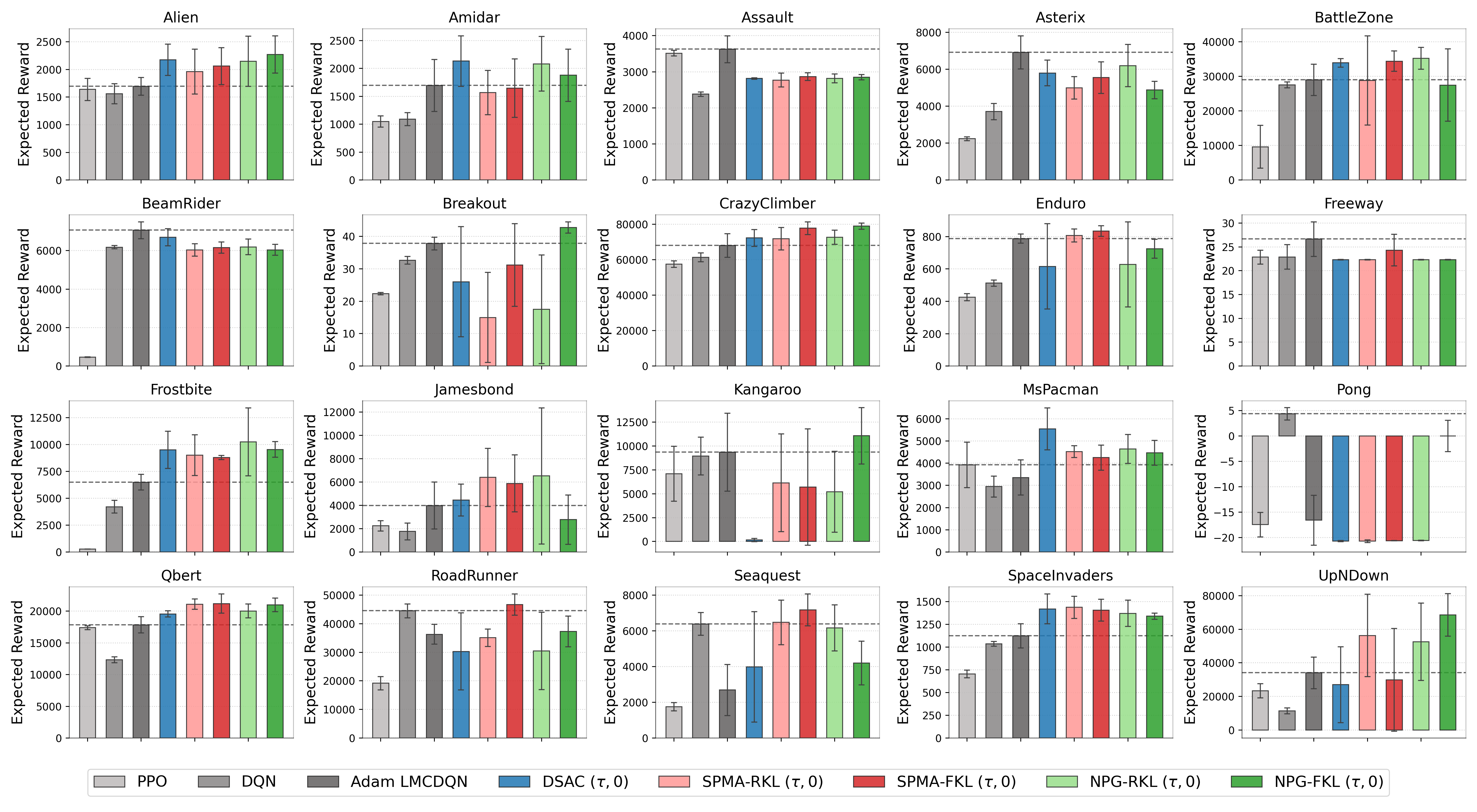}
\caption{\textbf{Expected reward after 10M steps}. While~\cref{tab:normalized_score_summary} reports aggregate performance using IQM, this figure shows the corresponding per-game final rewards. Objectives from our proposed framework (colored) use adaptive actor entropy ($\tau \neq 0$) and no critic entropy ($\zeta = 0$); baselines are shown in grey.}
\label{fig:lmc-short-12}
\end{figure*}

\clearpage
\subsection{\texorpdfstring{Comparison to Baselines: Full Horizon Training Curves Across 20 Games}{Comparison to Baselines: Full Horizon Training Curves Across 20 Games}}
\label{app:full-horizon-curves}

\begin{figure*}[!ht]
\centering
\includegraphics[width=\textwidth]{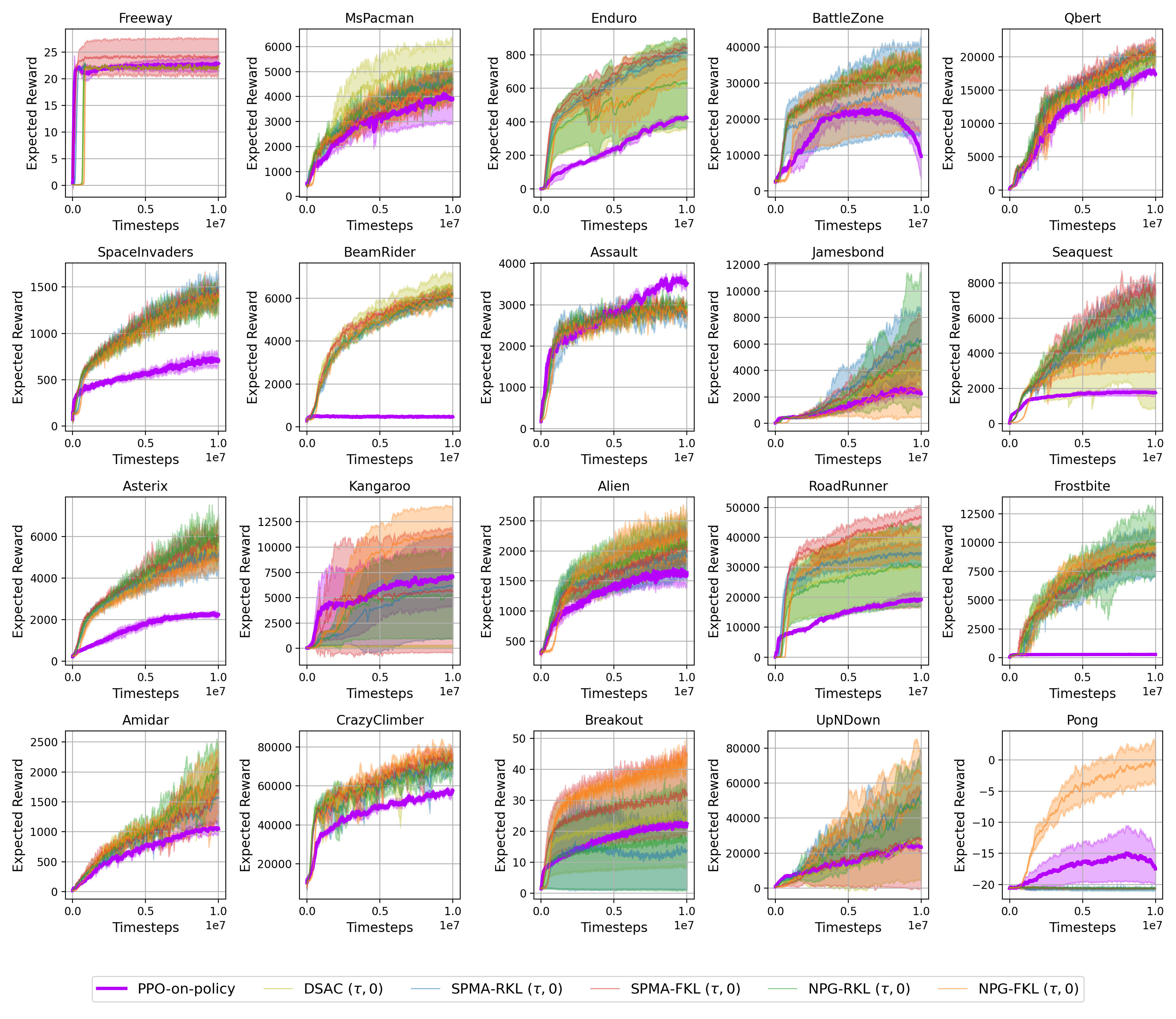}
\caption{On-policy $\PPO$ compared to objectives presented in this paper over the full training horizon. Overall, $\PPO$ underperforms.}
\label{fig:ppo-long}
\end{figure*}

\clearpage
\begin{figure*}[!ht]
\centering
\includegraphics[width=\textwidth]{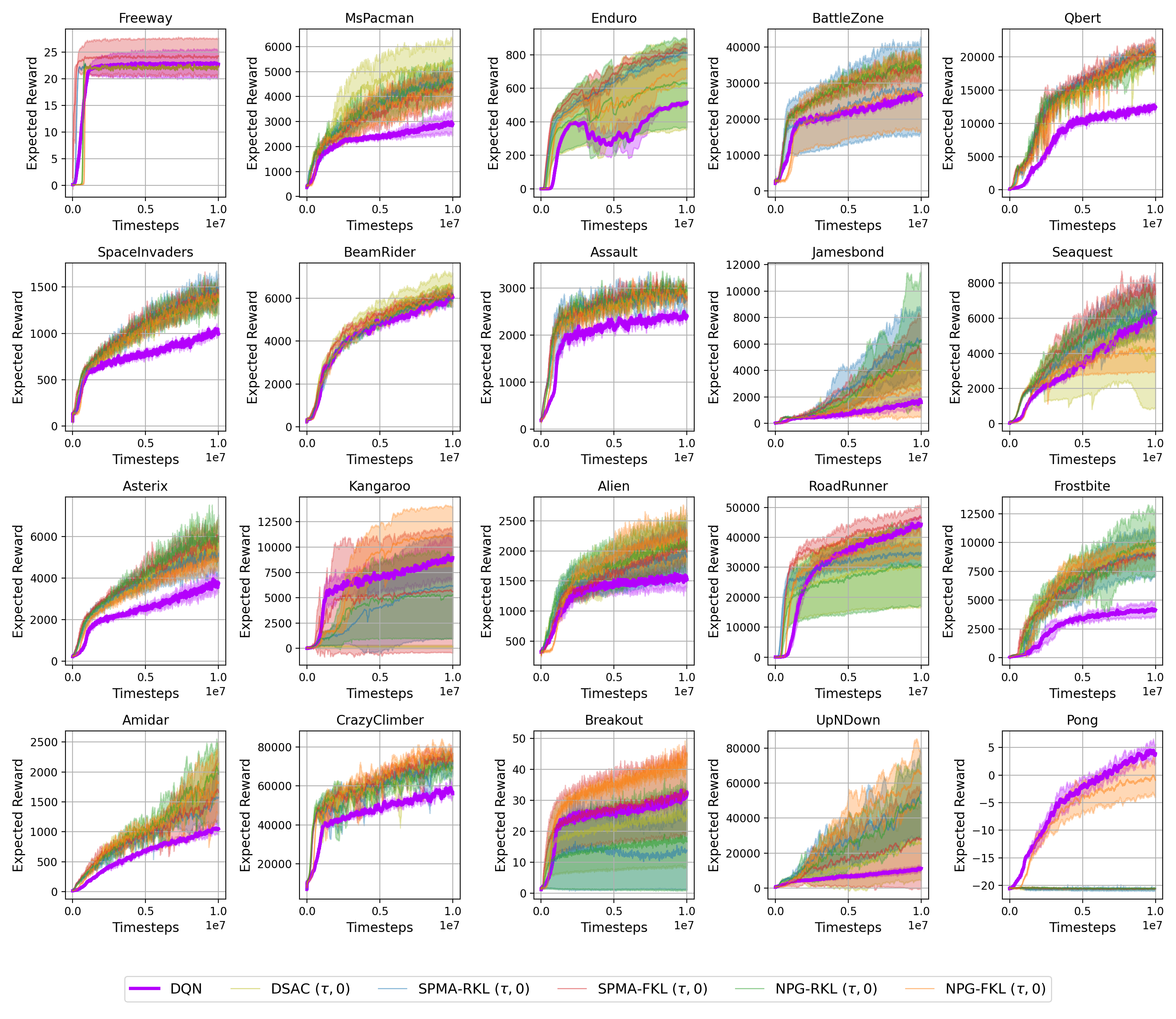}
\caption{$\DQN$ compared to objectives presented in this paper over the full training horizon. Overall, $\DQN$ underperforms on most games.}
\label{fig:dqn-long}
\end{figure*}

\clearpage
\begin{figure*}[!ht]
\centering
\includegraphics[width=\textwidth]{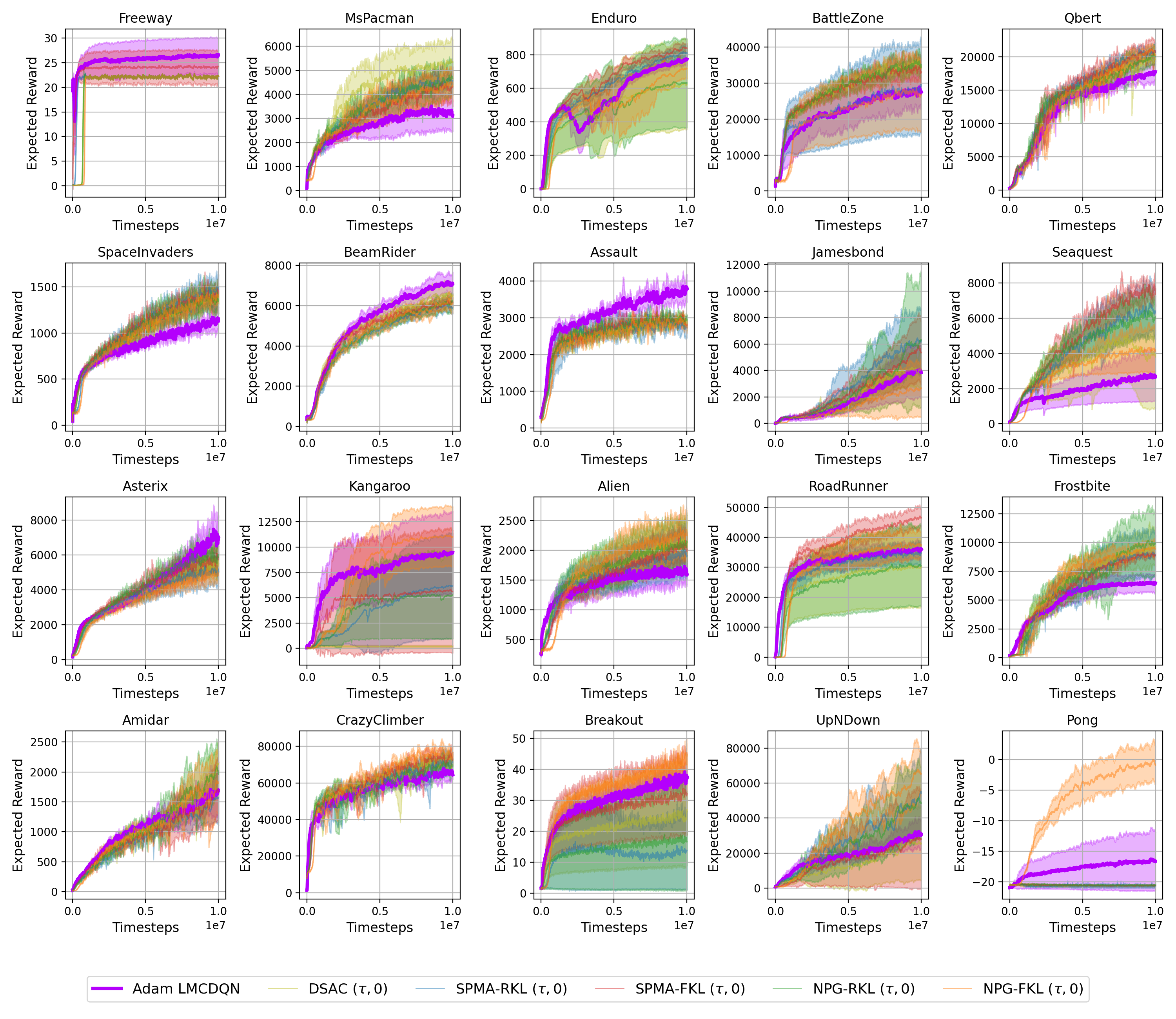}
\caption{Comparing $\LMC$ with all off-policy objectives presented in this paper over the full training horizon. Overall, the objectives from our proposed framework exhibit performance comparable to $\LMC$.}
\label{fig:lmc-long}
\end{figure*}

\clearpage
\subsection{\texorpdfstring{Entropy Collapse of Default $\DSAC$}{Entropy Collapse of Default DSAC}}
When ablating the default $\DSAC$ on four Atari games in~\cref{sec:experiments} (\cref{fig:dsac-ablation-m-1}), we find that it performs well on some games but fails on others (Alien and BeamRider). The results in~\cref{fig:dsac-ablation-entropy-collapse} show that poor performance correlates with a collapse in policy entropy (columns 1 and 3, orange curve). Prior work~\citep{xu2021target} attributes this collapse to the fixed target entropy in $\DSAC$, and our observations support this: reducing the learning rate for the entropy coefficient loss improves performance (blue curve). Nevertheless, removing critic entropy eliminates this issue, without requiring additional tuning of the entropy coefficient loss and results in strong performance.     
\begin{figure*}[!ht]
\centering
\includegraphics[width=\textwidth]{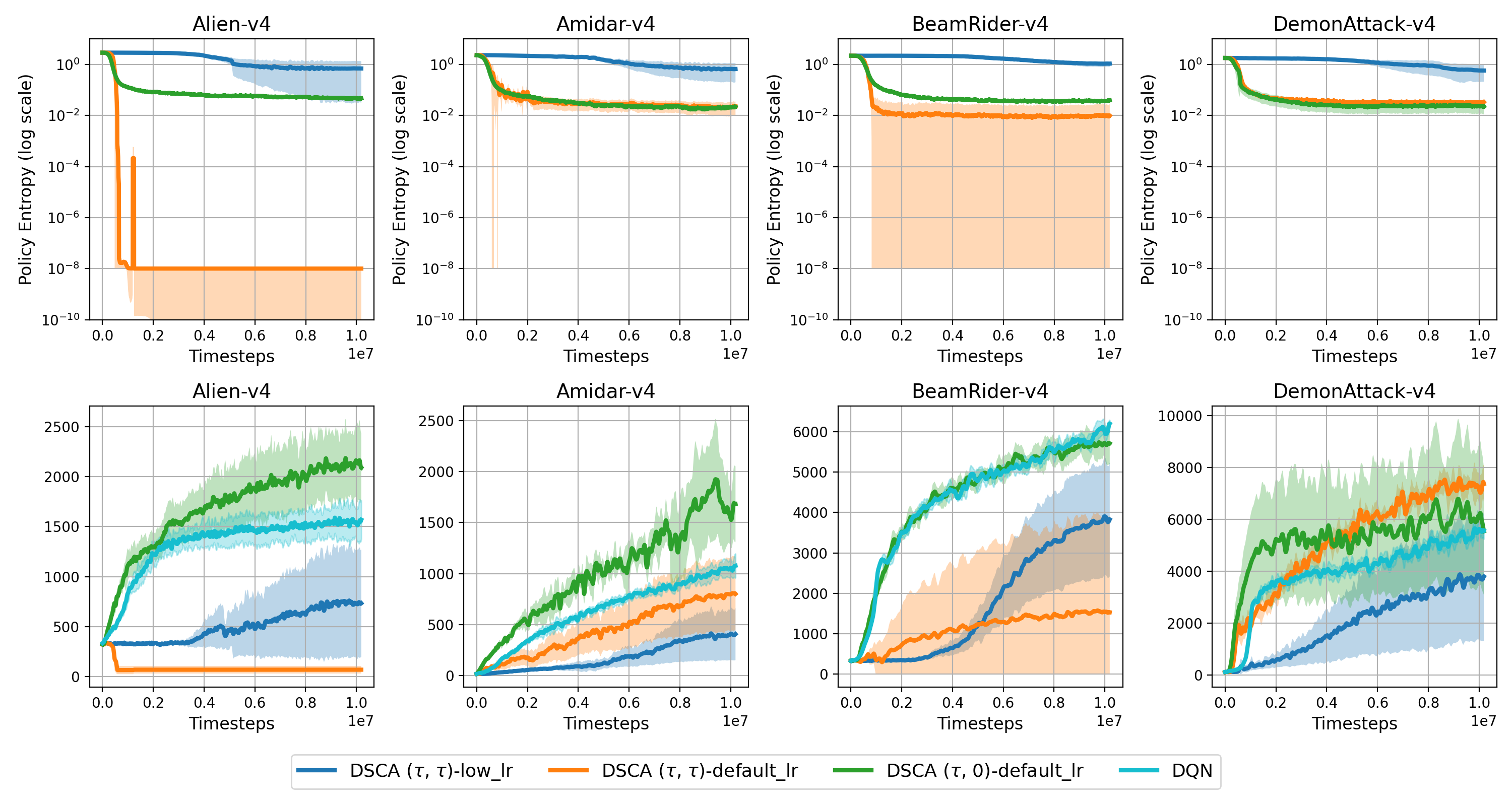}
\caption{The poor performance of default $\DSAC$ (orange curve) correlates with a collapse in policy entropy (columns 1 and 3). Although lowering the learning rate for the entropy coefficient loss improves performance, removing the critic entropy term resolves the issue without additional hyperparameter tuning.}
\label{fig:dsac-ablation-entropy-collapse}
\end{figure*}

\end{document}